\documentclass{article}
\usepackage{arxiv}
\usepackage{microtype}
\usepackage{hyphenat}

\makeatother
\usepackage{lmodern}

\usepackage{microtype}
\usepackage{hyperref}
\usepackage{url}
\usepackage{booktabs}
\usepackage{float} 
\usepackage{microtype}
\usepackage{caption}
\usepackage{subcaption}
\usepackage{hyperref}
\usepackage{url}
\usepackage{amsthm}
\usepackage{bm}
\usepackage{bbm}
\usepackage{booktabs}
\usepackage{textcomp}
\usepackage{amsmath,amssymb}
\let\LaTeXunderbrace\underbrace
\usepackage{mathabx}
\usepackage{mathabx}
\let\underbrace\LaTeXunderbrace
\DeclareMathSymbol{\braceld}{\mathord}{largesymbols}{"7A}
\DeclareMathSymbol{\bracerd}{\mathord}{largesymbols}{"7B}
\DeclareMathSymbol{\bracelu}{\mathord}{largesymbols}{"7C}
\DeclareMathSymbol{\braceru}{\mathord}{largesymbols}{"7D}
\usepackage{bbm}

\definecolor{darkblue}{rgb}{0, 0, 0.5}
\hypersetup{colorlinks=true, citecolor=darkblue, linkcolor=darkblue, urlcolor=darkblue}

\RequirePackage[T1]{fontenc}
\RequirePackage[tt=false, type1=true]{libertine}
\RequirePackage[varqu]{zi4}

\usepackage{etoolbox}
\usepackage{appendix}
\usepackage{etoc}
\usepackage{xspace}

\usepackage{amsthm}
\usepackage[ruled,vlined,linesnumbered]{algorithm2e}
\usepackage{hyperref}
\usepackage{booktabs}
\usepackage{enumitem}

\usepackage{tikz}
\usetikzlibrary{shapes.geometric, arrows.meta, positioning, calc, fit, backgrounds, patterns,patterns.meta,matrix, decorations.pathreplacing, decorations.pathmorphing,calc, 
                  shadows.blur, positioning}

\usepackage[normalem]{ulem}

\usepackage[capitalize,noabbrev]{cleveref}
\crefname{algocf}{Algorithm}{Algorithms}
\Crefname{algocf}{Algorithm}{Algorithms}
\usepackage{listings}
\lstdefinestyle{prompt}{
  basicstyle=\ttfamily\scriptsize,
  breaklines=true,
  breakatwhitespace=false,
  frame=single,
  backgroundcolor=\color{gray!8},
  rulecolor=\color{gray!40},
  xleftmargin=4pt,
  xrightmargin=4pt,
  aboveskip=4pt,
  belowskip=4pt,
  literate={—}{{---}}1,
}
\usepackage{amsthm}
\theoremstyle{plain}                                   \newtheorem{theorem}{Theorem}[section]
         \newtheorem{lemma}[theorem]{Lemma}
\newtheorem{corollary}[theorem]{Corollary}             \theoremstyle{definition}                              \newtheorem{definition}[theorem]{Definition}
\newtheorem{assumption}[theorem]{Assumption}           \theoremstyle{remark}                                                     
\usepackage{mathtools}
\usepackage{multirow}                                  \usepackage{adjustbox}
\usepackage{pgfplots}
\usepackage{pgfplotstable}
\usepackage{placeins}
\usepackage{wrapfig}                                   \usepackage{xcolor}
\pgfplotsset{compat=1.18}
  
\definecolor{border}{HTML}{1E293B}
\definecolor{borderm}{HTML}{334155}
\definecolor{vis}{HTML}{3B82F6}
\definecolor{visDark}{HTML}{1D4ED8}
\definecolor{visLight}{HTML}{DBEAFE}
\definecolor{visBg}{HTML}{EFF6FF}
\definecolor{evictA}{HTML}{EF4444}
\definecolor{evictLight}{HTML}{FEE2E2}
\definecolor{accent}{HTML}{10B981}
\definecolor{accentDark}{HTML}{047857}
\definecolor{accentLight}{HTML}{D1FAE5}
\definecolor{traincolor}{HTML}{7C3AED}
\definecolor{trainLight}{HTML}{EDE9FE}
\definecolor{labelc}{HTML}{475569}
\definecolor{gridkept}{HTML}{93C5FD}
\definecolor{promptc}{HTML}{64748B}
\definecolor{tokbg}{HTML}{F8FAFC}

\usepackage{amsmath,amsfonts,bm}

\def\eqref#1{equation~\ref{#1}}

\def\1{\bm{1}}

\DeclareMathAlphabet{\mathsfit}{\encodingdefault}{\sfdefault}{m}{sl}
\SetMathAlphabet{\mathsfit}{bold}{\encodingdefault}{\sfdefault}{bx}{n}

\definecolor{darkblue}{rgb}{0, 0, 0.5}
\hypersetup{colorlinks=true, citecolor=darkblue, linkcolor=darkblue, urlcolor=darkblue}
\usepackage{tikz}
\usetikzlibrary{arrows.meta, positioning, decorations.pathreplacing, calc, patterns}

\definecolor{border}{HTML}{3B4252}
\definecolor{borderm}{HTML}{4C566A}
\definecolor{labelc}{HTML}{6B7394}
\definecolor{tokbg}{HTML}{F0F1F5}
\definecolor{vis}{HTML}{5E81AC}
\definecolor{visDark}{HTML}{3B6EA5}
\definecolor{visLight}{HTML}{D8E6F3}
\definecolor{visBg}{HTML}{E8F0F8}
\definecolor{evictA}{HTML}{BF616A}
\definecolor{evictLight}{HTML}{F2DEDE}
\definecolor{gridkept}{HTML}{A3BE8C}
\definecolor{promptc}{HTML}{8B7EC8}
\definecolor{accent}{HTML}{4CAF50}
\definecolor{accentDark}{HTML}{2E7D32}
\definecolor{accentLight}{HTML}{E8F5E9}
\definecolor{traincolor}{HTML}{D08C3E}
\definecolor{trainLight}{HTML}{FDF0E0}
\definecolor{budgetcolor}{HTML}{C0392B}
\definecolor{budgetBg}{HTML}{FDF2F0}
\definecolor{border}{HTML}{4A4A4A}
\definecolor{gradcol}{HTML}{C0392B}
\definecolor{autogradcol}{HTML}{E67E22}

\definecolor{sumcolor}{HTML}{E67E22}
\definecolor{sumLight}{HTML}{FDF2E9}
\definecolor{sumDark}{HTML}{BF6516}
\definecolor{sumKept}{HTML}{F0C27A}

\definecolor{vis}{HTML}{4A90D9}       
\definecolor{evictA}{HTML}{E05A4F}    
\definecolor{evictB}{HTML}{E8A735}    
\definecolor{fut}{HTML}{F0F0F0}       
\definecolor{border}{HTML}{4A4A6A}
\definecolor{labelc}{HTML}{3A3A5C}

\definecolor{metaBg}{HTML}{FFFBF5}
\definecolor{metaBorder}{HTML}{C05621}
\definecolor{metaText}{HTML}{9C4221}
\definecolor{dumpColor}{HTML}{7B2D8E}
\definecolor{dumpLight}{HTML}{F5E6FF}

\newlength{\fscale}
\title{BALAR : A Bayesian Agentic Loop for Active Reasoning}
\author{\normalsize
Aymen Echarghaoui\thanks{Correspondence to \texttt{aymen20@stanford.edu}. Code: \url{https://github.com/AymenEcharghaoui/BALAR}}\\
\normalsize
Department of Statistics, Stanford University
\And
\normalsize
Dongxia Wu\\
\normalsize
Department of Computer Science, Stanford University
\And
\normalsize
Emily B.~Fox\\
\normalsize
Department of Statistics, Stanford University\\
Department of Computer Science, Stanford University
}

\begin{document}
\raggedbottom

\maketitle

\begin{abstract}
Large language models increasingly operate in interactive settings where solving a task requires multiple rounds of information exchange with a user. However, most current systems treat dialogue reactively and lack a principled mechanism to reason about what information is missing.
We propose \textbf{BALAR} (\textbf{B}ayesian \textbf{A}gentic \textbf{L}oop for \textbf{A}ctive \textbf{R}easoning), a task-agnostic outer-loop algorithm that requires no fine-tuning and enables multi-turn interaction between an LLM agent and a user. 
BALAR maintains a structured belief over latent states, selects clarifying questions by maximizing expected mutual information, and dynamically expands its state representation when the current one proves insufficient. 
We evaluate BALAR on three diverse benchmarks: AR-Bench-DC (detective cases), AR-Bench-SP (thinking puzzles), and iCraft-MD (clinical diagnosis). 
BALAR outperforms all baselines across the three benchmarks, with $\mathbf{14.6\%}$ higher accuracy on AR-Bench-DC, $\mathbf{38.5\%}$ on AR-Bench-SP, and $\mathbf{30.5\%}$ on iCraft-MD.
We further study whether BALAR can serve as a teacher for a questioning policy through supervised fine-tuning (SFT), direct preference optimization (DPO), and dense-reward reinforcement learning (RL). 
Across 18 iCraft-MD replications, distilling BALAR into a Llama-8B yields relative gains in frozen-Qwen final-answer accuracy of $\mathbf{8.1\%}$ with SFT, $\mathbf{10.9\%}$ with DPO, and $\mathbf{12.1\%}$ with RL over the untuned policy.
\end{abstract}

\etocdepthtag.toc{mtchapter}

\section{Introduction}
\label{sec:intro}

Modern deployments of large language models (LLMs) span domains where user intent is rarely unambiguous: a patient asking a medical AI for advice might omit critical symptoms, a customer service agent must resolve which product a user is asking about from a vague description. In all these cases, the bottleneck is not the reasoning capacity of the LLM, but the absence of a mechanism that (1) detects when a prompt is ambiguous, (2) formulates targeted clarifying questions, (3) integrates responses, and (4) decides when enough information has been gathered to commit to an answer.

Existing approaches address this challenge in different ways. 
Tree-of-Thoughts (ToT) \citep{tot} focuses on internal search over intermediate reasoning states using BFS/DFS to maintain multiple partial solutions, whereas Uncertainty-of-Thoughts (UoT) \citep{uot} relies on forward simulation of interaction trees to select follow-up questions. These procedures can be computationally costly, and neither maintains an explicit belief over latent user intent.
Other methods such as \texttt{CollabLLM} \citep{collabLLM} learn to select questions by fine-tuning on simulated trajectories, introducing costly training dependencies. The \texttt{MediQ Expert} system \citep{mediq} operates through a fixed pipeline of LLM calls and is task specific. Benchmarks such as AR-Bench \citep{arbench} and MediQ \citep{mediq} expose the gap: state-of-the-art LLMs at the time of their publication fall far below an oracle that has access to all private information.

We address this gap with \textbf{BALAR} (\textbf{B}ayesian \textbf{A}gentic \textbf{L}oop for \textbf{A}ctive \textbf{R}easoning), a \emph{task-agnostic}, \emph{training-free} Bayesian outer loop. The key insight is to model user intent as a latent discrete variable over a structured space of disambiguating dimensions where each dimension captures a facet of potential ambiguity (e.g.\ \emph{severity level}). BALAR initializes this belief at ``sleep time'' (before interaction begins), then iteratively selects the unasked question with highest mutual information with the belief, updates the posterior via Bayes' rule upon receiving each response, and dynamically expands the state space by proposing new dimensions when the current representation is insufficient. We illustrate the full pipeline in \cref{fig:balar-example-red} using a medical example.

Beyond using BALAR as an inference-time controller, we use it as a teacher to amortize reasoning into a questioning policy. 
BALAR provides belief states, selected questions, and information-theoretic supervision: SFT imitates BALAR’s state updates and questions, DPO learns to prefer questions with higher mutual information, and RL rewards questions that reduce uncertainty over the final answer. The student generates questions directly at test time without executing the Bayesian loop.

\paragraph{Contributions.} We make the following contributions:
\begin{itemize}[topsep=2pt,itemsep=1pt,leftmargin=*]
  \item We propose BALAR, a task-agnostic, training-free Bayesian outer loop that enables LLM agents to engage in structured multi-turn interaction, actively selecting informative questions and updating a belief over latent task states during dialogue (\cref{sec:method}).
  \item We study BALAR as a teacher for a questioning policy through supervised distillation of compact belief states, preference optimization over question information gain, and dense-reward reinforcement learning from answer-entropy reduction. SFT, DPO, and RL yield relative gains in final-answer accuracy of $8.1\%$, $10.9\%$, and $12.1\%$, respectively, over the untuned policy (\cref{subsec:distillation-results}).
  \item We evaluate BALAR on three benchmarks across open-weight LLMs. BALAR outperforms baselines by $14.6\%$ on AR-Bench-DC, $38.5\%$ on AR-Bench-SP, and $30.5\%$ on iCraft-MD (\cref{subsec:main}).
\end{itemize}

\begin{figure}
    \centering
    \includegraphics[width=\linewidth]{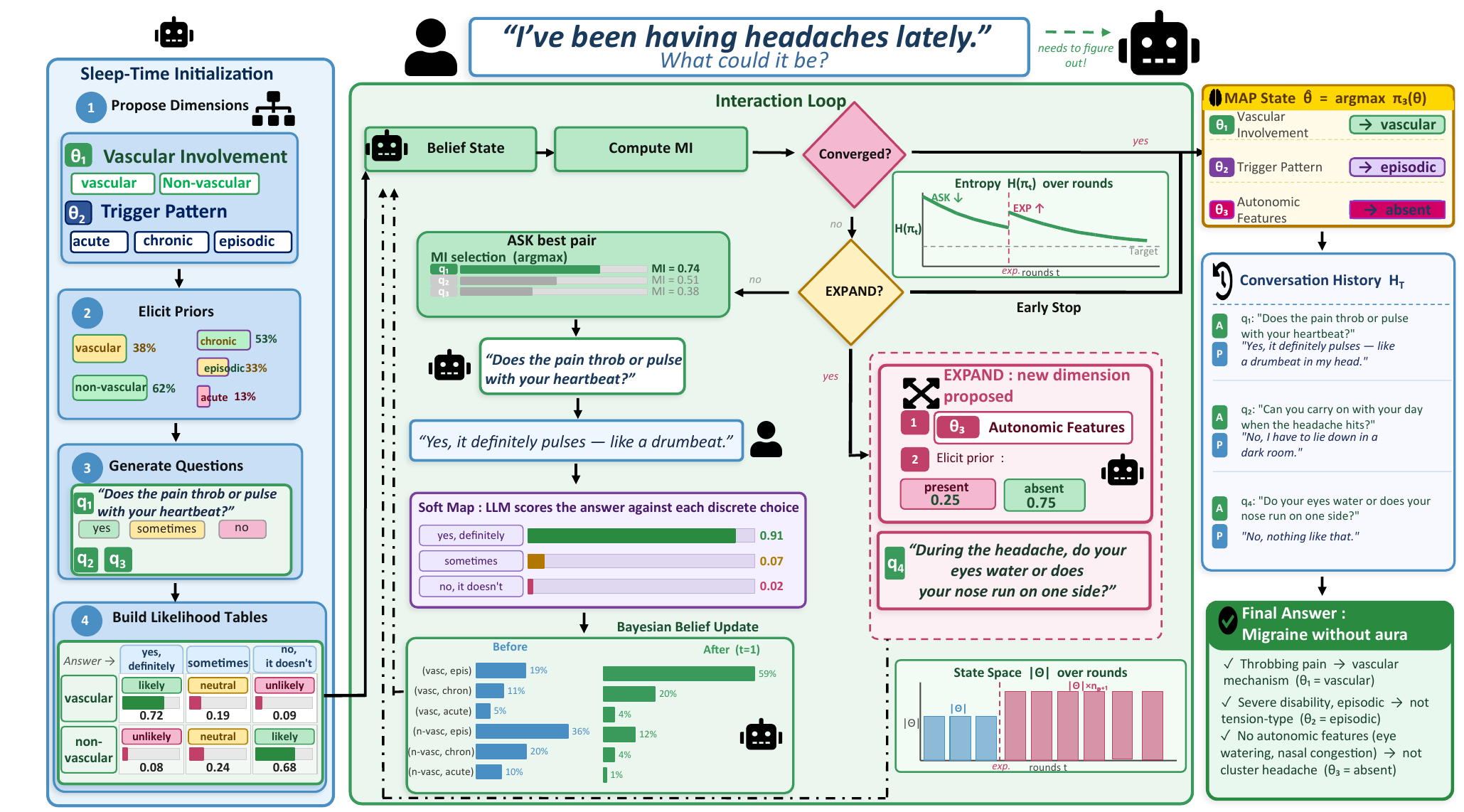}
    \caption{
\textbf{BALAR overview.}
Given an ambiguous query, BALAR performs multi-turn reasoning.
\emph{Sleep-time initialization} (left): the agent constructs a latent state representation by proposing disambiguating dimensions $\{\theta_j\}$, eliciting priors $\pi^{(j)}$, generating candidate questions $\mathcal{Q}$, and estimating likelihood tables $L^j(y \mid \theta_j;q)$.
\emph{Interaction loop} (center): the agent maintains a belief $\pi_t(\theta)$ and iteratively selects the unasked question $q$ maximizing mutual information $I_t(\theta; Y \mid \mathcal{H}_t;q)$.
User responses are incorporated via a Bayesian update.
When the entropy gap cannot be closed within the remaining rounds, BALAR triggers \textsc{Expand}, introducing new dimensions and targeted questions.
\emph{Final answer} (right): once the belief concentrates, an LLM produces the final answer using the MAP state $\hat{\theta}=\arg\max_\theta \pi_T(\theta)$ and history $\mathcal{H}_T$.}
\label{fig:balar-example-red}
\end{figure}

\section{Related Work}
\label{sec:related}

\paragraph{Active reasoning benchmarks.}
In \emph{active reasoning}, a model must iteratively acquire missing information through interaction rather than solving a problem from a fully specified prompt. The \textsc{AR-Bench} benchmark \citep{arbench} evaluates this capability by placing language models in multi-turn environments where they must ask informative questions to uncover hidden facts before producing a final answer. Empirical results reveal a substantial gap between passive and active reasoning performance: even state-of-the-art models achieve relatively low accuracy, and models frequently ask vague or redundant questions while struggling to accumulate useful information across turns. 
More broadly, related benchmarks evaluate multi-turn interaction and the handling of underspecified requests \citep{lmrl,mteval,abg_coqa,clamber,clarify_survey}.
These findings highlight the need for architectures that explicitly reason about uncertainty and guide question selection in a principled manner. BALAR addresses this by maintaining a structured posterior over latent user intent and selecting clarifying questions by maximizing expected information gain, providing a principled mechanism for strategic information gathering across interaction rounds.

\paragraph{Interactive medical dialogue.}
The \textsc{MEDIQ} framework \citep{mediq} introduces a benchmark for evaluating the ability of LLMs to proactively seek missing information in clinical decision-making tasks. MEDIQ converts existing datasets into interactive tasks by revealing only limited initial patient information and requiring the model to iteratively gather missing evidence before making a decision. Experiments show that prompting state-of-the-art LLMs to ask questions often \emph{degrades} performance relative to answering directly with partial information, highlighting the difficulty of proactive information-seeking for current models. In contrast to task-specific pipelines such as MediQ Expert, BALAR provides a general-purpose Bayesian outer-loop that operates across domains without relying on specialized heuristics.

\paragraph{Search-based and uncertainty-aware reasoning.}
\textsc{Tree-of-Thoughts (ToT)} \citep{tot} extends chain-of-thought reasoning by organizing intermediate reasoning steps as a search tree, combining LM generation with BFS/DFS to maintain multiple partial solutions and backtrack when necessary. \textsc{Uncertainty of Thoughts (UoT)} \citep{uot} similarly relies on forward simulation of interaction trees, propagating information-theoretic rewards across hypothetical dialogue trajectories to select the question with highest expected uncertainty reduction. While both approaches improve reasoning through exploration, they incur substantial computational cost and do not maintain an explicit probabilistic model of the problem state. The idea of selecting questions by maximizing expected entropy reduction over a belief state has classical precedent in expert 
systems. \citet{pathfinder} introduced this strategy in the PATHFINDER system for lymph-node pathology diagnosis, where a probability distribution over diseases is updated after each observation. 
Related approaches also use model uncertainty or question informativeness to decide when to clarify \citep{ask_more_informative_questions,clarify_when_necessary,tree_of_clarifications}.
BALAR maintains and updates a posterior belief over a latent space, selecting questions through mutual information. Although it also requires multiple LLM calls, these are decomposed into independent computations that can be executed in parallel, avoiding the sequential overhead of trajectory-based methods.

\paragraph{Proactive dialogue through learning and prompting.}
Proactive agents, which take initiative by asking clarifying questions or steering the dialogue, have been studied extensively \citep{proactive_survey,proactive_agents,mixed_initiative_dialogue}.
\textsc{CollabLLM} \citep{collabLLM} uses reinforcement learning to fine-tune LLMs to optimize collaboration outcomes, estimating a multiturn-aware reward by simulating future conversations with LLM-based user simulators \citep{abs-2311-05584,HuFLHL23}.
\textsc{Proactive Chain-of-Thought (ProCoT)} \citep{procot} instead induces proactive behaviors through prompting, augmenting standard prompting with intermediate reasoning steps that describe the dialogue state and plan the next action.  
Similar prompting approaches elicit clarification questions for code generation, retrieval, and planning \citep{clarifygpt,clarinet,ask_before_plan}.
\textsc{STaR-GATE} \citep{stargate} takes a self-improvement approach: starting from a pretrained model, it iteratively fine-tunes on questions that increase the likelihood of high-quality task responses, bootstrapping better clarification-seeking behavior without requiring human-labeled trajectories. Both approaches demonstrate that models can be made to ask clarifying questions, but rely on either additional training or heuristic prompt engineering. BALAR requires neither: its Bayesian formulation directly identifies informative queries from a maintained posterior, enabling principled proactive dialogue without fine-tuning or prompt heuristics.
Other learning-based approaches train clarification policies through preference optimization or reinforcement learning \citep{learn_to_clarify,circle}.

\paragraph{Dialogue state tracking and POMDP dialogue management.}
Dialogue state tracking (DST) maintains a belief over latent user goals and
other dialogue variables as evidence accumulates over a conversation.  In the
classical POMDP formulation, this belief is the information state used by a
dialogue policy to select system actions under uncertainty.  The Bayesian
Update of Dialogue State framework (BUDS) of \citet{thomsonyoung2010} represents the
dialogue state with a factored dynamic Bayesian network and performs tractable
approximate updates using loopy belief propagation.  It couples this tracker
with a component-based policy learned using episodic Natural Actor Critic.
\citet{thomson2013} gives a broader treatment of the resulting architecture,
inference algorithms, parameter estimation, and policy optimization. BALAR inherits the central DST idea of maintaining a structured Bayesian
belief over latent user state, but addresses a different representation and
control regime.  Classical BUDS assumes a domain-specific state structure and
ontology fixed at design time, and learns a policy from a reward function and
training dialogues.  BALAR instead uses LLM calls to construct the dimensions,
their value sets, candidate questions, and likelihood tables. It then selects
questions through closed-form mutual information rather than a learned
policy, and can add new dimensions and questions during interaction.  Thus,
classical BUDS tracks uncertainty within a predefined dialogue state, whereas BALAR also
attempts to refine the state representation when its current question bank
offers insufficient projected information gain.

\paragraph{Post-training and policy distillation.}
Supervised fine-tuning (SFT) on demonstrations is a standard way to teach language models to follow instructions and adopt task-specific behavior \citep{flan,rlhf}.
Reasoning distillation extends this approach by training a smaller model on explanations or intermediate states produced by a larger teacher \citep{distillstep}.
Reinforcement learning from human feedback instead optimizes a reward, commonly with Proximal Policy Optimization (PPO) \citep{christiano2017,ppo,rlhf}.
Direct Preference Optimization (DPO) avoids an explicit reward-modeling and online RL stage by directly increasing the likelihood of preferred outputs relative to rejected ones \citep{dpo}.
Group-relative policy optimization further replaces a learned critic with advantages computed by comparing rewards among outputs sampled for the same prompt \citep{deepseekmath}.
Multi-turn extensions of these methods optimize trajectory-level rewards over an entire interaction \citep{multiturn_rlhf,archer}.
These methods usually optimize instruction following, response preference, or reasoning correctness.
BALAR distillation instead targets active information gathering: SFT copies a compact Bayesian MAP state and BALAR-selected question, DPO prefers high over low mutual-information questions, and RL rewards the entropy reduction caused by each answer.

\section{Problem Setup}
\label{sec:setup}

We consider an LLM agent interacting with a single user to resolve an ambiguous prompt. Generalization to the multi-user setting is in Appendix~\ref{app:multi-user}, with results on AR-Bench-DC in \cref{sec:results}.

\begin{definition}
An \textbf{interaction instance} is a tuple $(\mathbf{p}, \mathbf{c})$ with $\mathbf{p}$ an \emph{ambiguous prompt} and $\mathbf{c}$ optional \emph{meta-context} (publicly known background). A user holds \emph{private information} not visible to the agent.

\end{definition}

Let $y^\star$ denote the correct final answer to $\mathbf{p}$. The agent may ask questions $q \in \mathcal{Q}$ to the user, receiving a \emph{natural language response} $r$. We treat this as a \emph{Bayesian active information gathering} problem: the agent maintains a belief over a dynamic latent state $\theta$ representing user intent, selects actions to reduce uncertainty, and commits to an answer when sufficiently confident.

\section{BALAR: Bayesian Agentic Loop for Active Reasoning}
\label{sec:method}

BALAR operates in two phases: a \emph{sleep-time initialization} \citep{sleep} that constructs a structured belief and a question bank from LLM calls, and an \emph{interaction loop} that adaptively selects questions, updates the belief, and expands the state space as needed.

\noindent\textbf{Running example.} We trace a concrete instance: a patient submits the prompt \emph{``I've been having headaches lately. What could it be ?''} This complaint could
indicate migraine, tension headache, or hypertensive crisis, and the patient cannot be expected to use clinical terminology such as \emph{vascular involvement} or \emph{trigger pattern}. BALAR's goal is to infer these latent medical concepts from patient-friendly conversational questions, without the patient ever encountering medical jargon.

\subsection{Sleep-Time Initialization}
\label{subsec:init}
Given $(\mathbf{p},\mathbf{c})$, the agent performs four initialization steps
before interacting with the user.
\paragraph{Disambiguating dimensions.}
A single LLM call proposes $p$ \emph{disambiguating dimensions}
$(\theta_1,\ldots,\theta_p)$, with each $\theta_j$ having a finite value set
$\Theta_j = \{v_{j,1},\ldots,v_{j,n_j}\}$. In the example of \cref{fig:ex-step1}, $\theta_2$ is ``trigger pattern'', with values
``acute'', ``chronic'', and ``episodic''.
The prompt queries for such disambiguating dimensions and their values by defining the dimensions as distinct unresolved
factors whose value would narrow the answer.
No task-specific knowledge base is pre-selected. We
retain the same prompt structure across tasks, adapting only
the domain-specific wording.

\paragraph{Prior elicitation.}
Our Bayesian approach relies on priors $\pi^{(j)}(v_{j,k})$. Given a finite
label set $\mathcal{L}$, we enable an LLM to define these priors by prompting
it to assign a prior label $\ell_{j,k}^{\mathrm p}\in\mathcal{L}$ to each value
$v_{j,k}$ (see \cref{fig:ex-step2}). A fixed map $\phi$ maps
each label to a positive, unnormalized weight for its corresponding value,
so that $\pi^{(j)}(v_{j,k})=\phi(\ell_{j,k}^{\mathrm p})/
\sum_{k'=1}^{n_j}\phi(\ell_{j,k'}^{\mathrm p})$. We choose
$\mathcal{L}=\{\texttt{likely},\texttt{neutral},\texttt{unlikely}\}$ as a
general-purpose three-level scale. The role-specific prompts are in
Appendix~\ref{app:prompts}, and the label definitions are in
Appendix~\ref{subsec:label-def}.

\paragraph{Question generation.}
A single LLM call generates $|\mathcal{Q}|$ candidate clarifying questions,
each with a discrete answer set $\mathcal{Y}_q$ (see \cref{fig:ex-step3}). Questions are designed to
be informative about the dimensions from Step~1. Note that the user provides free-form responses to these questions and is not shown the discrete answer set $\mathcal{Y}_q$. Instead, $\mathcal{Y}_q$ is used internally for tractable belief updates (see \cref{subsec:update}).

\paragraph{Likelihood table construction.}
A separate LLM call defines the \emph{dimension-level likelihood matrix}:
\begin{equation}
L^j(y \mid \theta_j=v_{j,k};q)
  \;=\; \frac{\phi(\ell_{q,j,k,y}^{\mathrm l})}
             {\sum_{y' \in \mathcal{Y}_q} \phi(\ell_{q,j,k,y'}^{\mathrm l})},
  \quad \forall (v_{j,k}, y) \in \Theta_j \times \mathcal{Y}_q.
\label{eq:dimension-likelihood}
\end{equation}
Each call handles exactly one question-dimension pair $(q,j) \in \mathcal{Q} \times [p]$ (see \cref{fig:ex-step4}) and jointly returns a likelihood label $\ell_{q,j,k,y}^{\mathrm l} \in \mathcal{L}$ for every $(v_{j,k}, y) \in \Theta_j \times \mathcal{Y}_q$. Thus, labels are elicited jointly within each dimension-level matrix and separately across pairs.

\subsection{Structured Belief State}
\label{subsec:belief}

Under an a priori \emph{independence assumption} across dimensions, the prior over the joint state $\theta = (\theta_1,\ldots,\theta_p) \in \Theta := \prod_{j=1}^p \Theta_j$ factorizes as $\pi_0(\theta)=\prod_{j=1}^p\pi^{(j)}(\theta_j)$.
The \emph{state-level likelihood} for question $q$ is obtained by combining dimension-level likelihoods. For $\theta \in \Theta$ and $y \in \mathcal{Y}_q$:
\begin{equation}
K(y\mid\theta;q) = \frac{\prod_{j=1}^p L^j(y \mid \theta_j;q)}
                    {\sum_{y' \in \mathcal{Y}_q} \prod_{j=1}^p L^j(y' \mid \theta_j;q)}.
\label{eq:state-likelihood}
\end{equation}
This is a modeling choice that enables efficient computation. BALAR's deterministic inference layer maintains the global joint belief,
and performs the Bayesian updates and information-theoretic calculations. The LLM supplies the qualitative labels used to construct priors,
likelihoods, and soft observations, but it does not perform the probabilistic
inference.

\subsection{Information-Theoretic Question Selection}
\label{subsec:mi}

At round $t$, let $\pi_t(\theta) = \mathbb{P}(\theta \mid
\mathcal{H}_t)$ denote the posterior over states given the interaction
history $\mathcal{H}_t$. Its \emph{belief entropy} is
$\mathbb{H}(\pi_t)=-\sum_{\theta\in\Theta}\pi_t(\theta)\log\pi_t(\theta)$.
The agent selects the unasked question with the highest \emph{mutual information} (MI):
\begin{equation}
    q_t^* \in \arg\max_{q\notin\mathcal{H}_t} I_t(\theta;\, Y \mid \mathcal{H}_t;\, q)
    = \arg\max_{q\notin\mathcal{H}_t}
    \left[\mathbb{H}(\pi_t)-\mathbb{H}(\theta\mid Y,\mathcal{H}_t;q)\right],
\label{eq:greedy-query}
\end{equation}
where $\mathbb{H}(\theta\mid Y,\mathcal{H}_t;q)$ denotes the \emph{conditional entropy}:
\begin{equation}
  \mathbb{H}(\theta\mid Y,\mathcal{H}_t;q)
  = \sum_{y\in\mathcal{Y}_q}
    \underbrace{\mathbb{P}(Y=y\mid\mathcal{H}_t;q)}_{\textstyle p_t^q(y)}
    \mathbb{H}\bigl(\,
    \underbrace{\mathbb{P}(\theta\mid Y=y,\mathcal{H}_t;q)}_{\textstyle \pi_t^{q,y}(\theta)}\,
    \bigr).
\label{eq:conditional-entropy}
\end{equation}
The conditional entropy calculation relies on the \emph{answer predictive
distribution} :
\begin{equation}
  p_t^q(y)
  \;=\; \mathbb{P}(Y=y\mid\mathcal{H}_t;\,q)
  \;=\; \sum_{\theta\in\Theta}\pi_t(\theta)K(y\mid\theta;q),
\label{eq:predictive-response}
\end{equation}
and the posterior probability of $\theta$ assuming answer $y$ is given to
question $q$ :
\begin{equation}
  \pi_t^{q,y}(\theta)
  \;=\; \mathbb{P}(\theta\mid Y=y,\mathcal{H}_t;\,q)
  \;=\; \frac{K(y\mid\theta;q)\,\pi_t(\theta)}
             {\sum_{\theta'\in\Theta}K(y\mid\theta';q)\,\pi_t(\theta')}.
\label{eq:hypothetical-posterior}
\end{equation}

\paragraph{Theoretical Guarantee.}
Assuming answers to distinct questions are conditionally independent given $\theta$, BALAR's MI-maximizing selection within the \textsc{Ask} sub-loop satisfies
$G_k\geq(1-1/e)G_k^*$,
where $G_k$ is the expected information gain of the greedy policy after $k$ rounds and $G_k^*$ is the information gain of the best fixed set of at most $k$ candidate questions. The proof is in Appendix~\ref{subsec:theory}.

\subsection{Soft Bayesian Belief Update}
\label{subsec:update}
The user provides a free-form natural language answer $r_t$, which
need not coincide with any element of $\mathcal{Y}_{q_t^*}$ (see \cref{fig:ex-round1}).
An LLM maps $r_t$ to an answer label $\ell_y^{\mathrm a} \in \mathcal{L}$ for each choice
$y \in \mathcal{Y}_{q_t^*}$, inducing a probability vector : 
\begin{equation}
\hat{\omega}_y(r_t,q_t^*)
  \;=\; \frac{\phi(\ell_y^{\mathrm a})}
             {\sum_{y' \in \mathcal{Y}_{q_t^*}} \phi(\ell_{y'}^{\mathrm a})}, \qquad \forall y \in \mathcal{Y}_{q_t^*}.
\label{eq:soft-map}
\end{equation}
The \emph{effective per-state likelihood} under this soft observation is 
\begin{equation}
  \widehat{K}(r_t\mid\theta;q_t^*)
  \;=\; \sum_{y \in \mathcal{Y}_{q_t^*}}
    \hat{\omega}_y(r_t,q_t^*)\, K(y\mid\theta;q_t^*).
\label{eq:effective-likelihood}
\end{equation}
The posterior then updates as:
\begin{equation}
\pi_{t+1}(\theta)
  \;=\; \frac{\widehat{K}(r_t\mid\theta;q_t^*)\,\pi_t(\theta)}
             {\sum_{\theta' \in \Theta} \widehat{K}(r_t\mid\theta';q_t^*)\,\pi_t(\theta')}.
\label{eq:belief-update}
\end{equation}

\subsection{Dynamic State Expansion}
\label{subsec:expand}

When the existing question bank is exhausted or its best available MI indicates insufficient progress within the remaining rounds, BALAR triggers an \textbf{EXPAND} step.

\paragraph{Entropy gap criterion.}
Let $\alpha \in (0,1-1/|\Theta|]$. 
Define the target entropy as the entropy of a distribution that assigns probability $1-\alpha$ to a single state and distributes the remaining mass $\alpha$ uniformly over the other $|\Theta|-1$ states : 
\begin{equation}
  \mathbb{H}_\alpha = -(1-\alpha)\log(1-\alpha) - \alpha \log\!\left(\tfrac{\alpha}{|\Theta|-1}\right).
\label{eq:target-entropy}
\end{equation}
Equivalently, among all distributions satisfying
$\max_{\theta \in \Theta} \pi(\theta) \ge 1-\alpha$,
this choice achieves the largest possible entropy. The gap between BALAR's current belief entropy and this target entropy is $\Delta_t = \max(0, \mathbb{H}(\pi_t) - \mathbb{H}_\alpha)$. 
This gap drives when we choose to expand the state space. In particular, if
the gap exceeds the estimated information gain available from the current
question bank over the remaining $T-t$ rounds, scaled by $\lambda$, BALAR
triggers an \textsc{Expand} call, where $T$ is the maximum number of rounds.
Formally, the condition is
\begin{equation}
  \Delta_t > \lambda I_t^*(T-t),
\label{eq:expand-trigger}
\end{equation}
where $I_t^* = \max_{q\,\text{unasked}} I_t(\theta; Y \mid \mathcal{H}_t; q)$. This condition has an interpretation in terms of the number of rounds
required to close the entropy gap. To better understand this, suppose that, while the representation and question bank remain fixed, the
conditional expected information gain of subsequent queries is no greater than
the current best gain $I_t^*$. Then the expected
cumulative reduction in entropy over $k$ rounds satisfies
\begin{equation}
  \mathbb{E}\!\left[\mathbb{H}(\pi_t)-\mathbb{H}(\pi_{t+k})\mid\mathcal H_t\right]
  = \sum_{i=0}^{k-1}\mathbb{E}\!\left[I_{t+i}(\theta;\,Y\mid\mathcal H_{t+i};q_i)\mid\mathcal H_t\right]
  \;\leq\; k\, I_t^*,
\label{eq:expand-information-bound}
\end{equation}
for any continuation using the current question bank. To reach the target entropy $\mathbb{H}_\alpha$ from the
current entropy $\mathbb{H}(\pi_t)$ in expectation, the agent must therefore close a gap of
$\Delta_t = \max(0,\,\mathbb{H}(\pi_t) - \mathbb{H}_\alpha)$, which requires at least
$k^*=\lceil\Delta_t/I_t^*\rceil$
rounds under this assumption, with $k^*=\infty$ when $I_t^*=0<\Delta_t$. The EXPAND condition
$\Delta_t > \lambda\, I_t^*\,(T - t)$ implies that $k^* > \lambda\,(T - t)$, i.e the projected number of rounds exceeds a $\lambda$-multiple of the remaining budget. With $\lambda=1$, the current bank cannot close the expected gap within the horizon. Smaller $\lambda$ makes expansion more aggressive while larger $\lambda$ favors continuing to ask from the current bank.

\paragraph{EXPAND procedure.} When triggered, BALAR:
\begin{enumerate}[topsep=2pt,itemsep=0pt]
  \item Generates a new dimension $\theta_{p+1}$ via an LLM call conditioned on the current conversation history $\mathcal{H}_t$, and all existing dimensions $(\theta_j)_{j\in[p]}$.
  \item Elicits a prior $\pi^{(p+1)}$ for the new dimension, conditioned on $\mathcal{H}_t$.
  \item Expands the belief state under independence: $\pi_t^{\text{new}}(\theta, \theta_{p+1}) = \pi_t(\theta) \cdot \pi^{(p+1)}(\theta_{p+1})$.
  \item Recomputes likelihood tables for all existing questions over the new dimension.
  \item Generates new clarifying questions $\mathcal{Q}_{\mathrm{new}}$ that target the new
  dimension $\theta_{p+1}$ and the
  $m_{\mathrm{top}}$ existing dimensions with the highest marginal entropy,
  where $m_{\mathrm{top}}$ is a hyperparameter.
  \item Computes likelihood tables for the new questions $q\in\mathcal{Q}_{\mathrm{new}}$ over all dimensions $(\theta_j)_{j\in[p+1]}$.
\end{enumerate}

\paragraph{Theoretical Guarantee.}
The \textsc{Expand} action increases state entropy by representing uncertainty along a new coordinate before subsequent \textsc{Ask} rounds reduce it. 
Let $Y^\star$ denote the random variable for the correct final answer, and recall that $\theta_{1:p}=(\theta_1,\ldots,\theta_p)$ is the current representation, and $\theta_{p+1}$ is the proposed new dimension.
The relevant value of $\theta_{p+1}$ is not its added entropy, but its residual task information $I(Y^\star;\theta_{p+1}\mid\theta_{1:p},\mathcal H_t)$.
Appendix~\ref{subsec:expand-theory} proves that this representation value is monotone and, under conditional independence of candidate dimensions given $Y^\star$, $\theta_{1:p}$, and $\mathcal H_t$, submodular.
If the LLM acts as an $\eta$-approximate marginal-gain oracle, where $\eta\in(0,1]$ measures proposal quality, $k$ expansions recover at least a $(1-e^{-\eta})$ fraction of the best $k$-dimension refinement. 

\subsection{Convergence and Final Answer}
\label{subsec:final}

\paragraph{Convergence criterion.} BALAR uses two convergence conditions depending on whether the task exposes a discrete answer set $\mathcal{A}$.

\emph{Answer-probability convergence}. For each dimension $j$, BALAR elicits
a dimension-level answer distribution $P(a\mid\theta_j)$ over possible answers
$a\in\mathcal A$. We take the product of these factors as an unnormalized answer
compatibility score. BALAR averages this score over the joint belief and
then normalizes over answer options, obtaining the predictive answer belief
at round $t$:
\begin{equation}
  \hat{p}_t(a) =
  \frac{\sum_{\theta\in\Theta}\pi_t(\theta)
  \prod_{j=1}^p P(a\mid\theta_j)}
  {\sum_{a'\in\mathcal A}\sum_{\theta\in\Theta}\pi_t(\theta)
  \prod_{j=1}^p P(a'\mid\theta_j)}.
\label{eq:answer-probability}
\end{equation}
The loop terminates when $\max_{a \in \mathcal{A}} \hat{p}_t(a) \ge 1 - \alpha$.

\emph{Marginal-fraction convergence}. When the answer is a free-form explanation with no fixed answer set, BALAR instead declares convergence when a $\beta$-fraction of dimensions are individually concentrated:
\begin{equation}
  \frac{1}{p}\sum_{j=1}^p \mathbf{1}\!\left[\max_{\theta_j \in \Theta_j} \pi_t^{(j)}(\theta_j) \ge 1-\alpha\right] \;\ge\; \beta,
\label{eq:marginal-convergence}
\end{equation}
where $\pi_t^{(j)}$ is the marginal of $\pi_t$ for $\theta_j$, and $\beta \in (0,1]$ controls how many dimensions must converge before termination. $\beta = 1$ requires all dimensions to concentrate marginally. In both cases, the loop terminates when the total round budget $T$ or the ASK-round budget $T_{\text{ask}}$ is exhausted.

\paragraph{Final answer.} The MAP state $\hat{\theta} = \arg\max_{\theta\in \Theta} \pi_T(\theta)$ is extracted. A final LLM call takes $\hat{\theta}$ as a structured disambiguation summary appended to the original prompt and conversation history $\mathcal H_T$ to produce the answer (see \cref{fig:ex-final}).

Appendix~\ref{app:multi-user} provides the full algorithm in \cref{alg:BALAR}. Implementation details including atomic reasoning, parallelism and the verifier are discussed in Appendix~\ref{subsec:practical}. 
Appendix~\ref{subsec:complexity} analyzes both LLM-calls complexity and the cost of the operations used for exact Bayesian inference.

\section{Experimental Setup}
\label{sec:experiments}

\subsection{Datasets}
\label{subsec:datasets}

We evaluate on three structurally distinct benchmarks : \textbf{AR-Bench-DC} (detective reasoning), \textbf{AR-Bench-SP} (situation puzzles), and \textbf{iCraft-MD} (clinical diagnosis).

\paragraph{AR-Bench-DC \citep{arbench}.} 100 detective cases, each with scene metadata, victim description, and up to 5 suspect profiles. Each suspect holds private information (story, task, alibi). The agent must identify the true murderer by interrogating suspects. Scored by exact match.

\paragraph{AR-Bench-SP \citep{arbench}.} 100 situation puzzles, where the agent must reconstruct the hidden explanation of a puzzling scenario by asking yes/no questions to a user who knows the solution. The original AR-Bench evaluation measures the final answer using a character-level F1 similarity with the ground-truth explanation \citep{arbench}. However, this metric is insensitive to semantic structure and word order. For instance, a random permutation of the same characters can still achieve a perfect score. We therefore do not use the character-level F1 metric and instead evaluate solutions using non-strict semantic equivalence \citep{semantic} which we explain in Appendix~\ref{subsec:semantic_equivalence}.

\paragraph{iCraft-MD \citep{mediq}.} 140 patient cases from CRAFT-MD. Each case contains a clinical multiple-choice question, patient demographic information, a chief complaint, and a list of atomic clinical facts private to the user. Scored by exact match on the correct answer option.

For AR-Bench-DC, each user is a suspect with its own private info (story + task). For AR-Bench-SP and iCraft-MD, there is a single user (the puzzle narrator or patient, respectively).

\subsection{Baselines}
\label{subsec:baselines}

We compare against dataset-specific interactive baselines implemented in our experimental pipeline. Our baseline comparison focuses on training-free, inference-time methods. 

\textbf{AR-Bench baselines.} For the detective reasoning and situation puzzle benchmarks (AR-Bench-DC and AR-Bench-SP), we evaluate four baselines: \textbf{Few-Shot (AR-Bench)} which is the best-performing method reported in \citet{arbench}, \textbf{UoT} \citep{uot}, \textbf{ToT} \citep{tot}, and \textbf{Proactive CoT} \citep{procot}. All baselines are run with a maximum of 25 interaction turns, using the best hyperparameter settings reported in \citet{arbench}.

\textbf{Medical dialogue baselines.} For the medical benchmark iCraft-MD, we evaluate two baselines: \textbf{Zero-Shot}, where an LLM iteratively generates clarifying questions from the conversation history and produces a final answer after the interaction budget is exhausted, and \textbf{MediQ Expert} \citep{mediq}, a five-module pipeline comprising initial assessment, abstention, question generation, information integration, and decision making. Both baselines are run with a maximum of 25 interaction rounds. For \textbf{MediQ Expert}, we adopt the best-performing configuration reported in \citet{mediq} (Scale + Rational Generation + Self-Consistency (3)).

We follow the baseline sets used in the original benchmark studies \citep{arbench,mediq} rather than evaluating every method on every dataset. We use the original implementations and retain the authors' prompt templates. Transferring these baselines across benchmarks would require adapting and validating their task-specific prompts and procedures.

\paragraph{Interaction-budget and prompt controls.}
All methods receive the same maximum budget of 25 user-directed questions.
Few-Shot and Zero-Shot always use the full budget, whereas BALAR, UoT, ToT,
Proactive CoT, and MediQ Expert may terminate earlier according to their own
stopping behavior. BALAR therefore has no greater opportunity to query the user
than its baselines. Prompt templates are fixed across all instances within each
dataset. Across datasets, BALAR retains the same prompt structure and functional
instructions while adapting domain-specific wording. The complete templates are provided in
Appendix~\ref{app:prompts}.

\subsection{Models and User Simulator}
\label{subsec:models}

We evaluate five open-weight LLMs spanning a range of parameter scales and training paradigms. Specifically, we use \texttt{Qwen2.5-7B/14B/32B-Instruct} to evaluate scaling with model size, \texttt{QwQ-32B} to assess the effect of reasoning-tuned models, and \texttt{Llama-3.1-8B-Instruct} as an instruction-tuned model from a separate family to evaluate generality beyond the Qwen family. 
All models are served locally via vLLM using bfloat16 precision and bitsandbytes 4-bit quantization. 
For the Qwen-family models (\texttt{Qwen2.5-7B/14B/32B-Instruct}, \texttt{QwQ-32B}), we additionally enable YaRN rope scaling (factor $4\times$, original context length 32768) to support a maximum context length of 131072 tokens for long interaction histories. 
This is not used for \texttt{Llama-3.1-8B-Instruct}. 
In every reported BALAR run, the configured agent model is held fixed across all BALAR LLM calls, including dimension proposal, prior and likelihood elicitation, question generation, soft mapping, expansion, and final-answer generation. 
The agent model is used without access to external tools. 
The user simulator is fixed to \texttt{Qwen2.5-32B-Instruct} across all experiments, is prompted with the user's private information and must respond to the agent's question in natural language.

\subsection{Distilling BALAR into a Questioning Policy}
\label{subsec:balar-distillation}

We additionally test whether a smaller model can learn BALAR's question-selection behavior on iCraft-MD and generate questions without running the Bayesian loop at evaluation time. We use BALAR traces generated with \texttt{Qwen2.5-32B-Instruct} as supervision and train \texttt{Llama-3.1-8B-Instruct} as the student. For each of three patient-split seeds, the 140 cases are partitioned into two disjoint halves of 70 cases. We run both fold directions, training on one half and evaluating on the other, and repeat each fold with three training seeds. This yields 18 matched evaluation replications per method. The same 140 patients recur across splits and seeds.
At turn $t$, every policy receives the same information available to the BALAR agent: 
the ambiguous prompt $\mathbf p$, the public meta-context $\mathbf c$, the observed interaction history $\mathcal H_t$, and the previous compact state. 
The policy never receives the user's private clinical facts or the correct answer $Y^\star$. In the offline SFT and DPO data, the previous compact state comes from the BALAR trace. During evaluation, each policy receives its own previously generated state. A compact state lists BALAR's active dimensions and, for each dimension, the value in the joint MAP state together with a label derived from that value's marginal posterior probability. Appendix~\ref{app:distillation-objectives} gives the exact inputs, targets, rewards, and losses.

\paragraph{Supervised fine-tuning.}
For every \textsc{Ask} turn in each training-patient BALAR trace, the target contains the compact state computed from the current belief $\pi_t$ and the question $q_t^*$ selected by BALAR. The student is trained to map $(\mathbf p,\mathbf c,\mathcal H_t)$ and the previous compact state to this current state and question. This tests whether direct supervision can transfer BALAR's state tracking and question selection without executing BALAR at evaluation time.

\paragraph{Direct preference optimization.}
At a fixed history $\mathcal H_t$, the preferred and rejected outputs contain the same BALAR compact state but different questions. The preferred question has the largest recorded mutual information with the latent state, while the rejected question has the smallest. Starting from the fold's SFT model, DPO increases the relative probability of the higher-information question. Because the input and state output are identical within each pair, the preference signal acts only on question selection.

\paragraph{Dense-reward reinforcement learning.}
Starting from the fold's SFT model, the student samples groups of candidate questions at training histories $\mathcal H_t$. For each valid, non-repeated question, frozen \texttt{Qwen2.5-32B-Instruct} first estimates a distribution over the answer options from $(\mathbf p,\mathbf c,\mathcal H_t)$, then uses the user's private clinical facts to simulate a response, and finally re-estimates the answer distribution after adding that response to the history. The normalized decrease in answer-distribution entropy is the question's reward. Repeated and malformed questions receive fixed penalties. We optimize a one-turn group-relative policy objective with a KL penalty to the SFT model. The reward therefore measures the immediate information obtained from one question rather than success over a complete dialogue trajectory.

\paragraph{Evaluation protocol.}
For each held-out patient, the policy asks 25 questions. A frozen \texttt{Qwen2.5-32B-Instruct}, conditioned on the user's private clinical facts, supplies the responses. We measure final-answer accuracy in two ways. The primary evaluator is a separate frozen \texttt{Qwen2.5-32B-Instruct} call that receives only $(\mathbf p,\mathbf c,\mathcal H_t)$ and selects an answer option. This measures the information elicited by the policy independently of the student's ability to produce the final answer. The secondary evaluator is the student itself, given the same information. 
We also record normalized answer-entropy reduction and the number of distinct questions.

\section{Results}
\label{sec:results}

\subsection{BALAR Benchmark Results}
\label{subsec:main}

\cref{fig:main_results} reports outcome scores for the two main agent models : exact answer accuracy for AR-Bench-DC and iCraft-MD, and non-strict semantic equivalence for AR-Bench-SP \citep{semantic}.
BALAR outperforms all baselines across all three benchmarks and both agent models, except in the case of AR-Bench-SP with \texttt{Llama-3.1-8B-Instruct},
where the smaller model may impact BALAR's structured prompting.
Under \texttt{Qwen2.5-32B-Instruct}, BALAR achieves
relative gains of $\mathbf{14.6\%}$ on AR-Bench-DC, $\mathbf{38.5\%}$ on AR-Bench-SP, and $\mathbf{30.5\%}$ on iCraft-MD over the strongest
respective baselines. To contextualize this gain on iCraft-MD, we compare against an oracle that receives full patient information : BALAR with \texttt{Qwen2.5-32B-Instruct} without access to any private information achieves $73.6\%$, reaching $88\%$ of the oracle upper bound of $83.6\%$ (Appendix~\ref{subsec:oracale}). This suggests that BALAR's Bayesian formulation
provides more robust performance than methods relying on prompt-driven or
search-based heuristics, or task-tailored designs such as MediQ Expert.

\begin{figure}[!t]
    \centering
    \begin{subfigure}{\linewidth}
        \centering
        \includegraphics[width=0.75\linewidth]{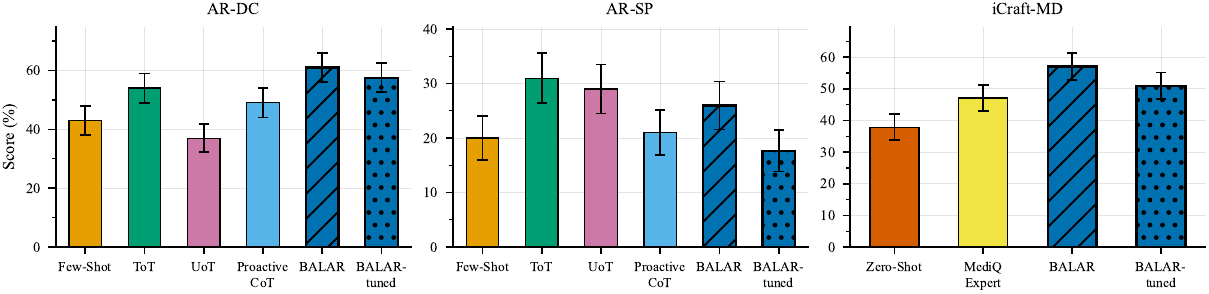}
        \caption{\texttt{Llama-3.1-8B-Instruct}}
        \label{fig:bars_llama}
    \end{subfigure}
    \vspace{0em}
    \begin{subfigure}{\linewidth}
        \centering
        \includegraphics[width=0.75\linewidth]{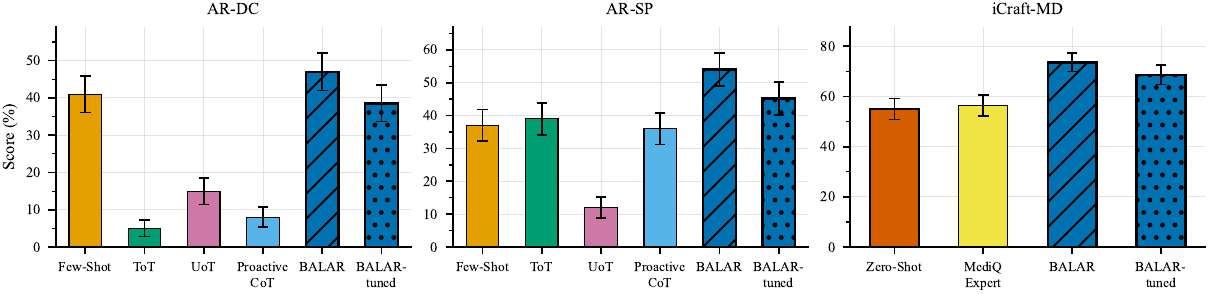}
        \caption{\texttt{Qwen2.5-32B-Instruct}}
        \label{fig:bars_qwen}
    \end{subfigure}
    \vspace{-0.2in}
    \caption{\emph{BALAR-tuned} denotes configurations selected based on average validation error over 500 randomly seeded 50/50 validation-test splits and evaluated by their average accuracy on the corresponding held-out test sets. All other methods use their full-dataset best configurations. Outcome is \emph{exact accuracy} for AR-DC and iCraft-MD, and \emph{non-strict semantic equivalence} for AR-SP. Error bars show binomial standard errors.}
    \label{fig:main_results}
\end{figure}

\paragraph{Held-out configuration selection.}
In \cref{fig:main_results}, results for BALAR are for its best configuration of $(\alpha, \beta, p, |\mathcal{Q}|)$, matching the procedure done in the compared-to published benchmark results.
For each LLM model and dataset, 
we additionally compare to \emph{BALAR-tuned}, where configurations are selected based on average validation error over 500 randomly seeded 50/50 validation-test splits and evaluated by their average accuracy on the corresponding held-out test sets.
Comparing to the baseline results that retain the configuration-selection bias of the original benchmark papers \citep{arbench,mediq}, leading to an unfair comparison, \emph{BALAR-tuned} remains ahead of the strongest baseline in four of six model--dataset comparisons.

\subsection{Distilling BALAR into a Questioning Policy}
\label{subsec:distillation-results}

We next evaluate whether BALAR can teach \texttt{Llama-3.1-8B-Instruct} to ask useful questions without executing BALAR at test time. 
\Cref{fig:distillation-final-accuracy} reports final-answer accuracy after 25 questions across 18 replications. 
With the frozen Qwen readout, mean accuracy increases from $67.5\%$ for the base policy to $72.9\%$ after SFT, $74.8\%$ after DPO, and $75.6\%$ after RL, corresponding to relative gains of $8.1\%$, $10.9\%$, and $12.1\%$, respectively. 
The student's own readout improves from $52.9\%$ for the base policy to $63.5\%$, $62.9\%$, and $63.9\%$, respectively.

The dialogue statistics in \Cref{fig:distillation-information-diversity} show the same broad improvement. 
Mean normalized answer-entropy reduction per question is $0.00076$ for the base policy, $0.00302$ for SFT, $0.00323$ for DPO, and $0.00360$ for RL. 
The mean replication-level median number of distinct questions increases from $2.3$ for the base policy to $5.1$ for SFT, $7.4$ for DPO, and $6.9$ for RL. 
Thus all three trained policies ask a broader set of questions that, on average, causes a larger update in the frozen Qwen answer distribution. 
DPO produces the greatest question diversity, whereas RL produces the largest entropy reduction and highest mean frozen-Qwen accuracy.

\begin{figure}[htbp]
    \centering
    \vspace{-0.1in}
    \begin{subfigure}[t]{0.29\linewidth}
        \centering
        \includegraphics[width=\linewidth]{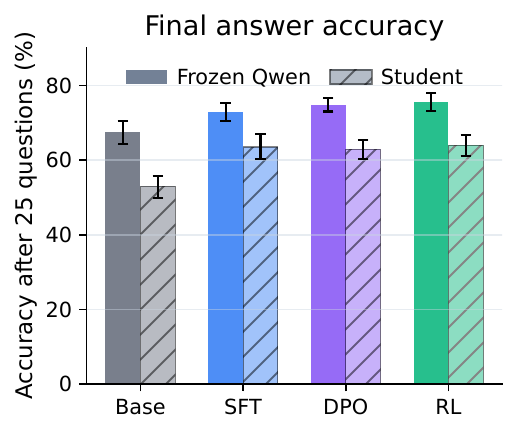}
        \caption{Final-answer accuracy.}
        \label{fig:distillation-final-accuracy}
    \end{subfigure}
    \hspace{0.06\linewidth}
    \begin{subfigure}[t]{0.47\linewidth}
        \centering
        \includegraphics[width=\linewidth]{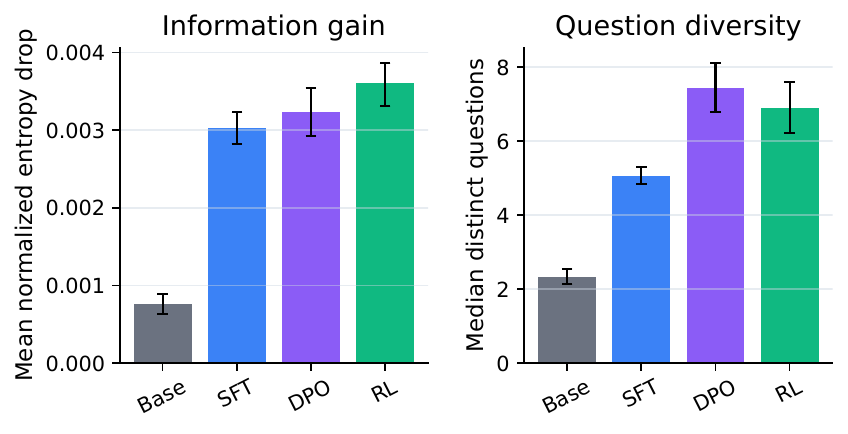}
        \caption{Information gain and question diversity.}
        \label{fig:distillation-information-diversity}
    \end{subfigure}
    \caption{BALAR distillation results on iCraft-MD. \textbf{(a)} Accuracy after 25 questions using either the frozen Qwen readout or the student's own readout. \textbf{(b)} Mean normalized answer-entropy drop per question and the median number of distinct questions over 25 turns. The 18 replications comprise three patient splits, both fold directions, and three training seeds. Error bars are 95\% bootstrap intervals over the replication-level values.}
    \label{fig:distillation-results}
\end{figure}

\subsection{Effect of Reasoning Mode}
\label{subsec:reas}
\cref{tab:reasoning} reports the effect of replacing \texttt{Qwen2.5-32B-Instruct}
with the reasoning-tuned \texttt{QwQ-32B} at $T_{\text{ask}}{=}10$.
Switching models yields relative gains of $10.9\%$ on AR-Bench-DC and $36.5\%$
on AR-Bench-SP, while performance remains unchanged on iCraft-MD suggesting
reasoning capacity matters more for deductive tasks than for clinical decision-making: structured hypothesis elimination benefits most from deeper reasoning, while clinical diagnosis appears less bottlenecked by reasoning depth than by information gathering.

\subsection{Scaling with Model Size}
\label{subsec:scaling}
To isolate the effect of model scale from hyperparameter choice, we fix the best hyperparameter configuration found for \texttt{Qwen2.5-32B-Instruct} and evaluate the \emph{same configuration} on \texttt{Qwen2.5-7B-Instruct} and \texttt{Qwen2.5-14B-Instruct}. \cref{tab:scaling} shows that BALAR scales consistently with model size on iCraft-MD: accuracy improves from $50.7\%$ (7B) to $62.1\%$ (14B) to $73.6\%$ (32B), suggesting BALAR scales with model capacity rather than saturating at smaller scales.

\subsection{Ablation Studies} 
\label{subsec:ablation}
\cref{tab:ablation_grpA} isolates two key components of BALAR on iCraft-MD with \texttt{Qwen2.5-32B-Instruct}, using $T_{\text{ask}}{=}10$, $\alpha{=}0.1$, $p{=}5$, and $|\mathcal{Q}|{=}10$. For question selection, we compare : (i) random selection from the question bank, (ii) LLM-prompted selection (``which question should I ask next given the history?''), and (iii) MI maximization (ours). This isolates the value of the information-theoretic criterion from the rest of the framework. 
Replacing MI maximization with random or LLM-prompted
selection drops accuracy by $23.3\%$ and $9.6\%$ respectively, confirming the
value of the information-theoretic criterion. 
We also consider disabling EXPAND and fixing the state space and question bank to their initial values. 
Disabling \textsc{Expand} reduces
accuracy by $13.7\%$, highlighting that the initial state space can be insufficient and BALAR benefits from dynamic expansion. 
\begin{table}[htbp]
  \centering
  \caption{Effect of reasoning, model scale, question selection, and expansion on BALAR performance.}
  \begin{subtable}[t]{0.34\linewidth}
    \centering
    \caption{Reasoning mode.}
    \label{tab:reasoning}
    \small
    \setlength{\tabcolsep}{2pt}
    \begin{tabular}{@{}lcc@{}}
    \toprule
    \textbf{Dataset} & \textbf{Qwen-32B} & \textbf{QwQ-32B} \\
    \midrule
    AR-DC & 46.0 ± 5.0 & 51.0 ± 5.0 \\
    AR-SP & 52.0 ± 5.0 & 71.0 ± 4.5 \\
    iCraft-MD & 71.4 ± 3.8 & 70.0 ± 3.9 \\
    \bottomrule
    \end{tabular}
  \end{subtable}
  \hfill
  \begin{subtable}[t]{0.27\linewidth}
    \centering
    \caption{Model scaling.}
    \label{tab:scaling}
    \small
    \setlength{\tabcolsep}{2pt}
    \begin{tabular}{@{}lc@{}}
    \toprule
    \textbf{Model} & \textbf{iCraft-MD} \\
    \midrule
    Qwen-7B & 50.7 ± 4.2 \\
    Qwen-14B & 62.1 ± 4.1 \\
    Qwen-32B & 73.6 ± 3.7 \\
    \bottomrule
    \end{tabular}
  \end{subtable}
  \hfill
  \begin{subtable}[t]{0.34\linewidth}
    \centering
    \caption{Question selection and expansion.}
    \label{tab:ablation_grpA}
    \small
    \setlength{\tabcolsep}{2pt}
    \begin{tabular}{@{}lc@{}}
    \toprule
    \textbf{Variant} & \textbf{Accuracy (\%)} \\
    \midrule
    \textbf{MI (ours)} & 67.9 ± 3.9 \\
    Random sel. & 52.1 ± 4.2 \\
    LLM sel. & 61.4 ± 4.1 \\
    No expansion & 58.6 ± 4.2 \\
    \bottomrule
    \end{tabular}
  \end{subtable}
\end{table}
Together, these results confirm that the gains reported in \cref{subsec:main} are driven by the combination of principled question selection and adaptive state refinement, rather than by the Bayesian scaffolding alone. We additionally study sensitivity of the choice of prior mapping in Appendix~\ref{subsec:prior_sens}.
Further analyses of prompt sensitivity, cross-model label consistency,
question quality, dialogue state tracking, predictive-belief calibration, and
answer-belief dynamics are reported in Appendices~\ref{subsec:prompt-sensitivity},
\ref{subsec:label-reliability}, \ref{subsec:question-quality},
\ref{subsec:dst-evaluation}, \ref{subsec:calibration-details}, and
\ref{subsec:answer-belief-dynamics}, respectively.

\subsection{Information Gain from Questions}
\label{subsec:info_gain}

We analyze how much uncertainty is removed by successive \textsc{Ask} actions. \cref{fig:combined} (bottom) plots the cumulative entropy reduction as a function of the number of questions asked under \texttt{Qwen2.5-32B-Instruct}. Let $\pi_{t_k}^{\mathrm{old}}$ denote the belief state immediately before the $k$-th \textsc{Ask} update, and $\pi_{t_k}^{\mathrm{new}}$ the posterior after incorporating the observed response. 
The entropy change at round $k$ is
$
    \Delta_k = \mathbb{H}(\pi_{t_k}^{\mathrm{new}}) - \mathbb{H}(\pi_{t_k}^{\mathrm{old}}).
$ We report the cumulative information gain
$
   -\sum_{i=1}^{k} \Delta_i,
$
indexed by the number of questions asked $k$, which measures the total uncertainty removed by questions alone. \cref{fig:combined} (bottom) shows that it grows
consistently with questions asked across all datasets, confirming that
MI-based selection extracts meaningful uncertainty reduction at every round.
AR-Bench-SP plateaus around $k{=}20$, suggesting diminishing returns once
key dimensions are resolved, while AR-Bench-DC and iCraft-MD remain
approximately linear throughout.

\cref{fig:combined} (top) shows BALAR accuracy as a function of the number of interaction rounds $K$ under \texttt{Qwen2.5-32B-Instruct}, obtained by truncating the dialogue at round $K \in \{5,10,15,20,25\}$ and recomputing the final answer from the resulting partial history. BALAR improves monotonically with more rounds on AR-Bench-SP and iCraft-MD,
demonstrating that additional questions consistently refine the belief toward the
correct answer. On AR-Bench-DC, performance is less monotone, likely due to the
multi-user setting. Further analysis is in Appendix~\ref{subsec:round_distribution}.

\begin{figure*}[t!]
    \centering
    \includegraphics[width=0.74\textwidth]{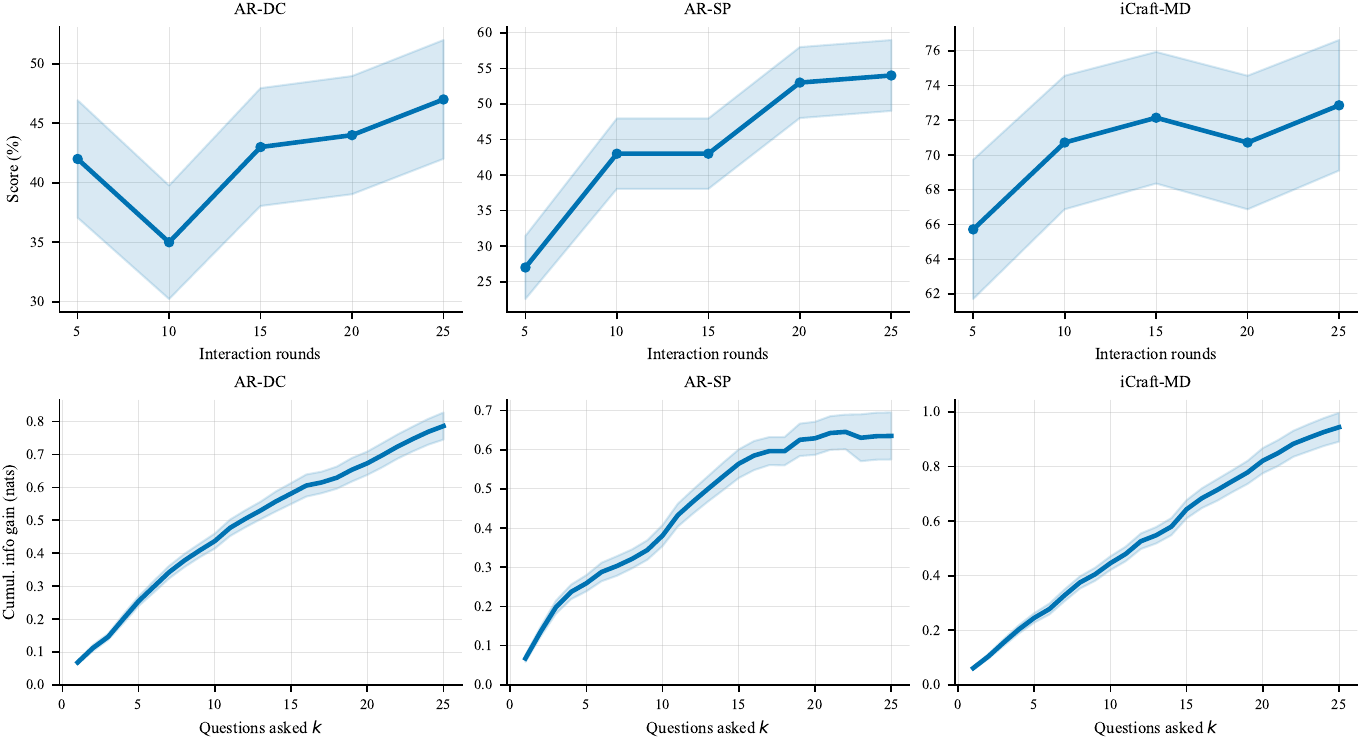}
    \caption{
    \textbf{Top:} BALAR score (\%) as a function of interaction rounds $K \in \{5,10,15,20,25\}$. 
    \textbf{Bottom:} Cumulative entropy reduction induced by successive 
    \textsc{Ask} actions. Solid lines show the mean across runs, shaded 
    regions denote standard errors. All results use
    \texttt{Qwen2.5-32B-Instruct}.}
    \label{fig:combined}
\end{figure*}

\FloatBarrier

\section{Conclusion}
\label{sec:conclusion}

We presented BALAR, a training-free Bayesian outer loop that equips LLM agents with a principled mechanism to detect, track, and resolve prompt ambiguity through targeted multi-turn dialogue. The framework maintains a structured factored belief over a space of disambiguating dimensions, selects questions by mutual information maximization, and dynamically expands the state representation when the existing belief is insufficient. BALAR is task-agnostic, compatible with any instruction-following LLM, and its sleep-time initialization keeps per-turn interaction latency low. Its belief state and information-theoretic signals can train a smaller model to ask better questions without running the Bayesian loop at test time. More broadly, BALAR can serve as both an inference-time belief-state controller and a source of supervision for learned questioning policies. 
Our evaluations demonstrate BALAR's impact across diverse benchmarks.
\paragraph{Limitations.} 
BALAR depends on the agent model's reliability in constructing and using structured representations, which may limit its benefits for less capable models.
The independence assumption across dimensions introduces approximation error, and the functional form $K(\theta, y;q) \propto \prod_j L^j(y \mid \theta_j;q)$ is a modeling choice. Likelihood tables estimated by label-to-weight maps may be miscalibrated.\footnote{BALAR mitigates elicitation error in two ways: LLM calls are kept close to atomic, and an optional verifier checks each response for semantic validity (Appendix~\ref{subsec:practical}). The components most exposed to this error like likelihood construction can also be delegated to a more capable LLM than the one driving the interaction.} The method is most beneficial when the initial context is genuinely ambiguous. 

\paragraph{Future work.}
Promising extensions include: (1) online calibration of likelihood
tables using observed answers, directly targeting the mutual-information
estimates rather than only the predictive answer belief evaluated in
Appendix~\ref{subsec:calibration-details}, and (2) new forms of Bayesian distillation
that transfer BALAR's structured states, posterior probabilities, and
information-gain signals into learned policies.

\section{Acknowledgement}
We especially thank Michael Y.~Li for helpful comments, and the Dynamode group meeting for peer-review feedback. This work was supported in part by ONR Grant N00014-22-1-2110, NSF Grant 2205084, and the Stanford Institute for Human-Centered Artificial Intelligence (HAI). EBF is a Biohub, San Francisco, Investigator.

\newpage
\bibliographystyle{abbrvnat}
\bibliography{references}
\newpage
\appendix
\etocdepthtag.toc{mtappendix}
\etocsettagdepth{mtchapter}{none}
\etocsettagdepth{mtappendix}{section}
\renewcommand{\contentsname}{Appendix Contents}
\tableofcontents
\newpage
\raggedbottom
\section{Multi-user extension}
\label{app:multi-user}

The main text presents BALAR for a single user. This appendix gives the
corresponding multi-user formulation, where the agent may ask a set
$\mathcal{U}$ of users. Candidate actions are question-user pairs $(q,u)$.
Setting $|\mathcal{U}|=1$ recovers the main-text formulation.

Disambiguating-dimension generation, prior elicitation, and question
generation are unchanged from the main text. The multi-user extension first
changes likelihood-table construction, which is performed for each
question-user pair.

\paragraph{Likelihood table construction.}
A separate LLM call defines the \emph{dimension-level likelihood matrix}:
\begin{equation}
L^j(y \mid \theta_j=v_{j,k};q,u)
  = \frac{\phi(\ell_{q,u,j,k,y}^{\mathrm l})}
         {\sum_{y' \in \mathcal{Y}_q}\phi(\ell_{q,u,j,k,y'}^{\mathrm l})},
\quad
\forall (v_{j,k},y) \in \Theta_j \times \mathcal{Y}_q.
\end{equation}
Each call handles one triple $(q,u,j) \in \mathcal{Q} \times
\mathcal{U} \times [p]$ and returns a label
$\ell_{q,u,j,k,y}^{\mathrm l} \in \mathcal{L}$ for every cell $(v_{j,k},y)$.

\paragraph{Structured belief state.}
The factorized prior over dimensions is unchanged from the main text. The
state-level likelihood for pair $(q,u)$ is obtained by combining
dimension-level likelihoods. For $\theta \in \Theta$ and
$y \in \mathcal{Y}_q$:
\begin{equation}
K(y\mid\theta;q,u)
  = \frac{\prod_{j=1}^p L^j(y \mid \theta_j;q,u)}
         {\sum_{y' \in \mathcal{Y}_q}
          \prod_{j=1}^p L^j(y' \mid \theta_j;q,u)}.
\end{equation}

\paragraph{Information-theoretic question selection.}
The posterior $\pi_t$ and belief entropy are as defined in the main text.
For a candidate pair $(q,u) \in \mathcal{Q} \times \mathcal{U}$, the
answer predictive distribution is:
\begin{equation}
p_t^{q,u}(y)
  = \mathbb{P}(Y=y \mid \mathcal{H}_t;q,u)
  = \sum_{\theta \in \Theta}\pi_t(\theta)K(y\mid\theta;q,u).
\end{equation}

Upon observing answer $y$ from pair $(q,u)$, the posterior becomes:
\begin{equation}
\pi_t^{q,u,y}(\theta)
  = \frac{K(y\mid\theta;q,u)\,\pi_t(\theta)}
         {\sum_{\theta' \in \Theta}K(y\mid\theta';q,u)\,\pi_t(\theta')}.
\end{equation}

The conditional entropy after querying pair $(q,u)$ is:
\begin{equation}
\mathbb{H}(\theta \mid Y,\mathcal{H}_t;q,u)
  = \sum_{y \in \mathcal{Y}_q}
    p_t^{q,u}(y)\,\mathbb{H}(\pi_t^{q,u,y}).
\end{equation}

The agent selects the unasked pair $(q_t^*,u_t^*)$ with the highest
mutual information (MI):
\begin{equation}
I_t(\theta;Y \mid \mathcal{H}_t;q,u)
  = \mathbb{H}(\pi_t)
    - \mathbb{H}(\theta \mid Y,\mathcal{H}_t;q,u).
\end{equation}

\paragraph{Soft Bayesian belief update.}
The selected user $u_t^*$ provides a free-form response $r_t$ to
question $q_t^*$. The soft mapping $\hat{\omega}_y(r_t,q_t^*)$ remains as
defined in the main text. The effective likelihood for the selected pair is:
\begin{equation}
\widehat{K}(r_t\mid\theta;q_t^*,u_t^*)
  = \sum_{y \in \mathcal{Y}_{q_t^*}}
    \hat{\omega}_y(r_t,q_t^*)K(y\mid\theta;q_t^*,u_t^*).
\end{equation}
The posterior update retains the main-text form:
\begin{equation}
\pi_{t+1}(\theta)
  = \frac{\widehat{K}(r_t\mid\theta;q_t^*,u_t^*)\,\pi_t(\theta)}
         {\sum_{\theta' \in \Theta}
          \widehat{K}(r_t\mid\theta';q_t^*,u_t^*)\,\pi_t(\theta')}.
\end{equation}

\paragraph{Dynamic state expansion.}
The entropy-gap trigger is unchanged from the main text. In the multi-user
setting, the best available MI is taken over unasked question-user pairs:
\begin{equation}
I_t^* = \max_{(q,u)\,\text{unasked}}
  I_t(\theta;Y \mid \mathcal{H}_t;q,u),
\qquad
\Delta_t > \lambda I_t^*(T-t).
\end{equation}

Suppose that, while the representation and question bank remain fixed, the
conditional expected information gain of subsequent queries is no greater
than the current best gain $I_t^*$. Then the expected cumulative reduction in
entropy over $k$ rounds satisfies
\begin{equation}
  \mathbb{E}\!\left[\mathbb{H}(\pi_t)-\mathbb{H}(\pi_{t+k})\mid\mathcal H_t\right]
  = \sum_{i=0}^{k-1}\mathbb{E}\!\left[I_{t+i}(\theta;\,Y\mid\mathcal H_{t+i};q_i,u_i)\mid\mathcal H_t\right]
  \;\leq\; k\, I_t^*,
\end{equation}
for any continuation using the current question bank.

The interpretation of this bound is unchanged from the main text. To reach
the target entropy $\mathbb{H}_\alpha$ from the current entropy
$\mathbb{H}(\pi_t)$ in expectation, the agent must close a gap of
$\Delta_t = \max(0,\,\mathbb{H}(\pi_t)-\mathbb{H}_\alpha)$. This requires
at least $k^*=\lceil\Delta_t/I_t^*\rceil$ rounds under the assumption above,
with $k^*=\infty$ when $I_t^*=0<\Delta_t$. The \textsc{Expand} condition
$\Delta_t>\lambda I_t^*(T-t)$ implies $k^*>\lambda(T-t)$.

The steps for proposing a new dimension, eliciting its prior, expanding
the state under independence, and generating questions for it are unchanged
from the main text. The likelihood-table steps are performed for every user.
When $\theta_{p+1}$ is added, BALAR recomputes
$L^{p+1}(y\mid\theta_{p+1}=v_{p+1,k};q,u)$ for every existing question $q$
and user $u$. For each new question $q\in\mathcal{Q}_{\mathrm{new}}$, it
computes $L^j(y\mid\theta_j=v_{j,k};q,u)$ for all users $u$ and dimensions
$j\in[p+1]$.

Convergence criteria, round budgets, and final-answer generation remain as
in the main text (Section~\ref{subsec:final}). BALAR uses the same stopping
rules and produces the final answer from the MAP state
$\hat\theta=\arg\max_\theta\pi_T(\theta)$ and interaction history
$\mathcal{H}_T$.

The full multi-user algorithm is given in Algorithm~\ref{alg:BALAR}. The
proofs for greedy MI selection and state expansion are in
Appendices~\ref{subsec:theory} and~\ref{subsec:expand-theory}, respectively.

\begin{algorithm}[t]
\SetAlgoLined
\DontPrintSemicolon
\small
\setlength{\algomargin}{0.6em}

\caption{Bayesian Agentic Loop for Active Reasoning (BALAR)}
\label{alg:BALAR}

\KwIn{$\mathbf{p}$ ambiguous prompt, $\mathbf{c}$ context, users $\mathcal{U}$, config $(p,|\mathcal Q|,\alpha,\beta,T,T_{\text{ask}},\lambda,m_{\mathrm{top}},\phi,\mathcal{L})$}
\KwOut{Final answer $\hat y$}
$\{\theta_j,\Theta_j\}_{j=1}^p \gets \textsc{ProposeDimensions}(\mathbf{p},\mathbf{c})$

$\pi^{(j)} \gets \textsc{ElicitPrior}(\mathbf{p},\mathbf{c},\theta_j)$ for all $j$ (parallel)

$\mathcal Q \gets \textsc{GenerateQuestions}(\mathbf{p},\mathbf{c},\{\theta_j\})$

$L_{q,u,j} \gets \textsc{EstimateLikelihood}(\mathbf{p},\mathbf{c},q,u,\theta_j)$ for all $(q,u,j)$ (parallel)

$\pi_0 \gets \bigotimes_{j=1}^p \pi^{(j)}$,  
$K_{q,u} \gets \textsc{BuildStateLikelihood}(L_{q,u,\cdot})$

$t\gets1$, $n_{\text{asked}}\gets0$, $\mathrm{asked}\gets\emptyset$

\While{$t\le T$ \textbf{and} $n_{\text{asked}}<T_{\text{ask}}$}{

\If{$\mathcal A$ exists \textbf{and} $\max_a \hat p_t(a)\ge1-\alpha$}{break}

\If{$\frac1p\sum_{j=1}^p \mathbf1[\max_v\pi^{(j)}_{t-1}(v)\ge1-\alpha]\ge\beta$}{break}

$I_{q,u}\gets I_{t-1}(\theta;Y|\mathcal{H}_{t-1};q,u)$ for all $(q,u)\notin\mathrm{asked}$

\eIf{no unasked $(q,u)$}{
    \If{state cap reached}{break}
    $I_t^*\gets0$
}{
    $(q^*,u^*)\gets\arg\max I_{q,u}$,
    $I_t^*\gets I_{q^*,u^*}$
}

$\Delta_t\gets\max(0,\mathbb{H}(\pi_{t-1})-\mathbb{H}_\alpha)$

\uIf{no unasked pair \textbf{or} $\Delta_t>\lambda I_t^*(T-t)$}{

\If{state cap reached}{break}

$\theta_{p+1},\Theta_{p+1}\gets\textsc{NewDimension}(\mathbf p,\mathbf c,\mathcal{H}_{t-1})$

$\pi^{(p+1)}\gets\textsc{ElicitPrior}(\mathbf p,\mathbf c,\theta_{p+1},\mathcal{H}_{t-1})$

$\pi_{t-1}\gets\pi_{t-1}\otimes\pi^{(p+1)}$, $p\gets p+1$

$L_{q,u,p+1}\gets\textsc{EstimateLikelihood}(\cdot)$ for $q\in\mathcal Q,u\in\mathcal U$

$\mathcal Q_{\text{new}}\gets\textsc{GenerateExpandedQuestions}(\cdot)$

$L_{q,u,j}\gets\textsc{EstimateLikelihood}(\cdot)$ for $q\in\mathcal Q_{\text{new}},u\in\mathcal U,j\in[p]$

$\mathcal Q\gets\mathcal Q\cup\mathcal Q_{\text{new}}$

$K_{q,u}\gets\textsc{BuildStateLikelihood}(L_{q,u,\cdot})$
}
\Else{

$r_t\gets\textsc{GetUserAnswer}(u^*,q^*)$

$\hat\omega\gets\textsc{SoftMap}(r_t,q^*)$

$\widehat K_{q^*,u^*}(r_t\mid\theta)\gets\sum_{y\in\mathcal Y_{q^*}}\hat\omega_yK_{q^*,u^*}(y\mid\theta)$

$Z_t\gets\sum_{\theta'\in\Theta}\widehat K_{q^*,u^*}(r_t\mid\theta')\pi_{t-1}(\theta')$

$\pi_t(\theta)\gets\widehat K_{q^*,u^*}(r_t\mid\theta)\pi_{t-1}(\theta)/Z_t$

$\mathrm{asked}\gets\mathrm{asked}\cup\{(q^*,u^*)\}$,  
$n_{\text{asked}}\gets n_{\text{asked}}+1$
}

$t\gets t+1$
}

$\hat\theta\gets\arg\max_{\theta}\pi_{t-1}(\theta)$

$\hat y\gets\textsc{FinalAnswer}(\mathbf p,\mathbf c,\mathcal{H}_{t-1},\hat\theta)$

\Return{$\hat y$}

\end{algorithm}
\FloatBarrier

\flushbottom

\section{Practical Considerations}
\label{subsec:practical}

\paragraph{Atomic reasoning.}
Most LLM calls in BALAR are kept as close to atomic as possible: each call typically handles a single dimension, question, user, or value. This keeps prompts small and helps avoid the degradation that occurs when too many likelihood judgments are packed into one context window. In particular, if too many likelihood entries are evaluated jointly, the model tends to return nearly uniform likelihoods, which in turn drives mutual information toward zero throughout the loop.

There is, however, one deliberate exception: likelihood estimation is performed at the level of a \emph{single} $(\text{user}, \text{question}, \text{dimension})$ triple, producing the full table over that dimension's values and the question's answer choices in one call. In principle, one could push this further and score each individual cell atomically, but this would be substantially more expensive in the number of parallel LLM calls. We therefore adopt this as a practical trade-off between granularity and computational cost.

\paragraph{Verifier.}
BALAR uses two layers of output checking. First, every LLM response is required to conform to a structured output schema (a Pydantic model). If the response does not parse into the required schema, the call is retried, up to a fixed maximum number of attempts. This handles formatting and structural failures such as malformed JSON, missing required fields, or invalid types.

An optional verifier LLM provides a second layer of checking. Unlike schema validation, which is purely syntactic and structural, the verifier is used to assess whether the response is logically coherent and semantically valid for the task. For example, it can flag issues such as inconsistent reasoning, or other task-level errors that still satisfy the formal schema. When the verifier rejects an output, it returns corrective feedback, and the model is queried again with that feedback appended to the prompt.

A lightweight version of this verification principle which is active in all current experiments is that 
every prompt explicitly asks the model to generate a short \texttt{reason} before 
committing to its output label or value. This chain-of-thought prior to the structured 
output encourages the model to surface its reasoning explicitly, making logical 
inconsistencies easier to detect and 
reducing labeling errors even without a full verification pass.

\paragraph{Parallelism.}
BALAR is implemented using asynchronous execution. Most LLM calls are written as \texttt{async} functions and are dispatched concurrently using \texttt{asyncio.gather}, with a semaphore controlling the maximum number of simultaneous API requests. This design is used not only during initialization (Steps~1--4), but also throughout the interaction loop for tasks such as likelihood-table construction, prior estimation, and expansion-related computations. As a result, large batches of independent LLM calls can run in parallel, so the dominant cost becomes the latency of the slowest request rather than the sum of all calls. This allows the system to maintain interactive latency even when many likelihood evaluations are required.

\paragraph{State space size.}
The joint state space $|\Theta| = \prod_j n_j$ grows exponentially with the number of dimensions. To keep inference tractable, BALAR enforces a configurable upper bound on the total number of states. Before triggering an \texttt{EXPAND} step, the algorithm checks whether adding another dimension would keep the total state space within this cap. If the expansion would exceed the limit, no further dimensions are introduced.

\section{Near-Optimality of Greedy MI Maximization}
\label{subsec:theory}

We analyze the multi-user version of BALAR, where each candidate action is a question-user pair $(q,u) \in \mathcal{Q} \times \mathcal{U}$. The single-user setting in the main text is recovered when $|\mathcal{U}|=1$. Under conditional independence, BALAR's greedy question-selection recovers at least a $(1-1/e)$ fraction of the information gain of the best fixed set of at most $k$ pairs in $\mathcal{Q} \times \mathcal{U}$.

Let $\theta \in \Theta$ be the latent state with prior $\pi_0$.
Each \emph{(question, user) pair} $(q,u) \in \mathcal{Q} \times \mathcal{U}$
yields an answer $Y_{q,u} \in \mathcal{Y}_q$ when queried.
A history $\mathcal{H}_t = \{(q_i, u_i, Y_{q_i,u_i})\}_{i=1}^t$ records the
pairs asked and answers received up to round $t$. We write
$(q,u) \notin \mathcal{H}_t$ to mean the pair has not yet been asked.
We make the following assumption throughout.

\begin{assumption}\label{ass1}
The answers $\{Y_{q,u}\}_{(q,u) \in \mathcal{Q} \times \mathcal{U}}$
are conditionally independent given $\theta$.
\end{assumption}

The \emph{greedy policy} $\pi^g$ selects at each round
\[
  (q_t^*, u_t^*) \;\in\;
  \arg\max_{(q,u)\,\notin\,\mathcal{H}_t}\;
  I_t(\theta;\, Y_{q,u} \mid \mathcal{H}_t;\, q, u)~\footnote{We use the semicolon to separate conditioning variables from fixed parameters: $f_D(A \mid B;\, C)$ denotes quantity $f$ indexed by $D$ (e.g., the round index $t$), where $A \mid B$ is read as ``$A$ given $B$'' in the probabilistic sense, and $C$ denotes a fixed parameter or context not treated as a random variable (e.g., the query--user pair $(q, u)$).}
,
\]
which is exactly BALAR's \textsc{Ask} selection rule of \cref{subsec:mi}.
Let $\mathcal{H}_t^g$ denote the history generated by $\pi^g$ after $t$
rounds. Define
\[
  G_t \;:=\; \mathbb{H}(\theta) - \mathbb{H}(\theta\mid\mathcal{H}_t^g),
  \qquad
  G_k^* \;:=\;
  \max_{\substack{B\subseteq\mathcal{Q}\times\mathcal{U}\\|B|\leq k}}
  I\!\left(\theta;\, \{Y_{q,u}\}_{(q,u)\in B}\right),
\]
where $G_k^*$ is the information gain of the best fixed set of
at most $k$ pairs.

\begin{theorem}\label{thm1}
Under Assumption~\ref{ass1}, BALAR's greedy MI-maximising $(q,u)$-selection
satisfies
\[
  G_k \;\geq\; \left(1 - \tfrac{1}{e}\right) G_k^*,
\]
recovering at least a $(1-1/e) \approx 63\%$ fraction of the cumulative
information gain of the best fixed set of at most $k$ pairs.
\end{theorem}

\paragraph{Proof of Theorem~\ref{thm1}.} We first establish an averaged inequality for the MI of a fixed pair.

\begin{lemma}\label{lem1}
Under Assumption~\ref{ass1}, fix a realized history $\mathcal{H}_t=h_t$ and
a set $C$ of pairs not yet asked at $h_t$. For any pair $(q,u)$ not yet asked
and not in $C$,
\[
  \mathbb{E}_{y_C\sim\mathbb{P}(Y_C\mid\mathcal{H}_t=h_t)}
  \left[I(\theta;\,Y_{q,u}\mid\mathcal{H}_t=h_t,Y_C=y_C;\,q,u)\right]
  \;\leq\;
  I(\theta;\,Y_{q,u}\mid\mathcal{H}_t=h_t;\,q,u),
\]
where $Y_C$ denotes the answers to pairs in $C$. The conditional mutual
information on the left is averaged over $Y_C$ given $\mathcal{H}_t=h_t$.
\end{lemma}

\begin{proof}
By Assumption~\ref{ass1}, $Y_{q,u}$ and $Y_C$ are independent given
$(\theta,\mathcal{H}_t=h_t)$. For each $y_C$,
\begin{align*}
  I(\theta;\,Y_{q,u}\mid\mathcal{H}_t=h_t,Y_C=y_C;\,q,u)
  &=\mathbb{H}(Y_{q,u}\mid\mathcal{H}_t=h_t,Y_C=y_C;\,q,u)\\
  &\quad-\mathbb{H}(Y_{q,u}\mid\theta,\mathcal{H}_t=h_t;\,q,u).
\end{align*}
Taking expectation over $Y_C$ given $\mathcal{H}_t=h_t$ and using that
conditioning reduces entropy on average gives the stated inequality.
\end{proof}

Let $B^*$ attain the maximum in $G_k^*$. Fix a realized
history $\mathcal{H}_t^g=h_t^g$, and let $(q_i,u_i)_{i=1}^r$ be the pairs in
$B^*$ that have not yet been asked, where $r\leq k$. The answers to pairs
in $B^*$ that have already been asked are specified by $\mathcal{H}_t^g$.
By the chain rule and Lemma~\ref{lem1},
\begin{align}
  I\!\left(\theta;\,\{Y_{q,u}\}_{(q,u)\in B^*}
    \,\Big|\,\mathcal{H}_t^g=h_t^g\right)
  &=\sum_{i=1}^r I(\theta;\,Y_{q_i,u_i}
      \mid\mathcal{H}_t^g=h_t^g,Y_{q_1,u_1},\ldots,Y_{q_{i-1},u_{i-1}};\,q_i,u_i)\\
  &\leq\sum_{i=1}^r I(\theta;\,Y_{q_i,u_i}
      \mid\mathcal{H}_t^g=h_t^g;\,q_i,u_i)\\
  &\leq k\max_{(q,u)\,\notin\,h_t^g}
      I(\theta;\,Y_{q,u}\mid\mathcal{H}_t^g=h_t^g;\,q,u).
\end{align}
The inequality on the first line uses Lemma~\ref{lem1}. Taking expectations over the history, we get 
\begin{align}
  G_k^*-G_t
  &=\mathbb{E}\!\left[\mathbb{H}(\theta\mid\mathcal{H}_t^g)\right]
    -\mathbb{H}\!\left(\theta\mid\{Y_{q,u}\}_{(q,u)\in B^*}\right)\\
  &\leq\mathbb{E}_{h_t^g}\!\left[
    I\!\left(\theta;\,\{Y_{q,u}\}_{(q,u)\in B^*}
      \,\Big|\,\mathcal{H}_t^g=h_t^g\right)\right]\\
  &\leq k\,\mathbb{E}_{h_t^g}\!\left[
    \max_{(q,u)\,\notin\,h_t^g}
    I(\theta;\,Y_{q,u}\mid\mathcal{H}_t^g=h_t^g;\,q,u)\right]\\
  &=k\,(G_{t+1}-G_t), \label{eq1}
\end{align}

Rearranging \eqref{eq1}:
\[
  G_k^* \;\leq\; k\,G_{t+1} - (k-1)\,G_t,
\]
which gives
\[
  G_k^* - G_{t+1} \;\leq\; \left(1 - \tfrac{1}{k}\right)(G_k^* - G_t).
\]
So with $G_0 = 0$:
\[
  G_k^* - G_k
  \;\leq\;
  \left(1 - \tfrac{1}{k}\right)^k G_k^*
  \;\leq\;
  e^{-1}\, G_k^*,
\]
The result of Theorem~\ref{thm1} then follows.

This guarantee applies to the \textsc{Ask} sub-loop with a fixed pair
space $\mathcal{Q} \times \mathcal{U}$.  The \textsc{Expand} action changes
both the latent representation and the pair space, so it is not covered by
\cref{thm1}.  Appendix~\ref{subsec:expand-theory} gives a separate
conditional result for the representation value of successful dimension
expansion.

\section{Representation Value of Dynamic Expansion}
\label{subsec:expand-theory}

The greedy guarantee in Appendix~\ref{subsec:theory} is stated for the
multi-user extension, where candidate actions are pairs $(q,u)$. Here we give
a separate, conditional characterization of \textsc{Expand}. The result
formalizes the value of adding a missing dimension.

Fix an interaction history $\mathcal H_t$ and condition throughout on this
history. Recall that $Y^\star$ denotes the correct final answer and
$\theta_{1:p}=(\theta_1,\ldots,\theta_p)$ denotes BALAR's current latent
representation. Consider an idealized pool of candidate dimensions
$\theta_{p+1},\ldots,\theta_{p+m}$ not present in $\theta_{1:p}$. For
$S\subseteq[m]$, write $\theta_{p+S}:=(\theta_{p+j})_{j\in S}$ and define the
\emph{representation value}
\begin{equation}
    F_t(S)
    := I(Y^\star;\theta_{p+S}\mid\theta_{1:p},\mathcal H_t)
    = \mathbb H(Y^\star\mid\theta_{1:p},\mathcal H_t)
      -\mathbb H(Y^\star\mid\theta_{1:p},\theta_{p+S},\mathcal H_t).
\end{equation}
Thus $F_t(S)$ measures how much residual task uncertainty becomes
representable after adding $\theta_{p+S}$. This definition distinguishes the value of
a dimension from the entropy it adds to the latent state: a high-entropy
dimension can still be irrelevant to $Y^\star$.

\begin{theorem}\label{thm2}
For every $S\subseteq[m]$ and $j\notin S$,
\begin{equation}
    F_t(S\cup\{j\})-F_t(S)
    = I(Y^\star;\theta_{p+j}\mid\theta_{1:p},\theta_{p+S},\mathcal H_t)\ge 0.
    \label{eq2}
\end{equation}
Hence $F_t$ is monotone.  Suppose additionally that the candidate dimensions
are conditionally independent given the target $Y^\star$, current representation $\theta_{1:p}$, and
history $\mathcal H_t$. Then $F_t$ is submodular.  In particular, for
$S\subseteq T\subseteq[m]$ and $j\notin T$,
\begin{equation}
    I(Y^\star;\theta_{p+j}\mid\theta_{1:p},\theta_{p+S},\mathcal H_t)
    \ge I(Y^\star;\theta_{p+j}\mid\theta_{1:p},\theta_{p+T},\mathcal H_t).
\end{equation}
\end{theorem}

\begin{proof}
The chain rule for mutual information gives
\begin{align}
    F_t(S\cup\{j\})-F_t(S)
    &=I(Y^\star;\theta_{p+j}\mid\theta_{1:p},\theta_{p+S},\mathcal H_t)\ge0,
\end{align}
which proves monotonicity.  For submodularity, let the right-hand side be
$\Delta_j(S)$.  Its entropy form is
\begin{align}
    \Delta_j(S)
    &=\mathbb H(\theta_{p+j}\mid\theta_{1:p},\theta_{p+S},\mathcal H_t)
      -\mathbb H(\theta_{p+j}\mid Y^\star,\theta_{1:p},\theta_{p+S},\mathcal H_t)\\
    &=\mathbb H(\theta_{p+j}\mid\theta_{1:p},\theta_{p+S},\mathcal H_t)
      -\mathbb H(\theta_{p+j}\mid Y^\star,\theta_{1:p},\mathcal H_t),
\end{align}
where the second equality follows from conditional independence.  If
$S\subseteq T$, conditioning reduces the first entropy term, while the second
term is unchanged.  Therefore $\Delta_j(S)\ge\Delta_j(T)$.
\end{proof}

The marginal gain in \eqref{eq2} gives an exact
criterion: $\theta_{p+j}$ is useful relative to the dimensions already added precisely
when it carries positive residual target information.  The next result shows
that if a finite collection of missing dimensions resolves the target, at
least one of them must carry a nontrivial share of that information.

\begin{corollary}
Assume the conditions of \cref{thm2}.  Let
$D\subseteq[m]$ satisfy
$\mathbb H(Y^\star\mid\theta_{1:p},\theta_{p+D},\mathcal H_t)=0$. For any $S\subsetneq D$, let
$r=|D\setminus S|$.  Then
\begin{equation}
    \max_{j\in D\setminus S}
    I(Y^\star;\theta_{p+j}\mid\theta_{1:p},\theta_{p+S},\mathcal H_t)
    \ge \frac{\mathbb H(Y^\star\mid\theta_{1:p},\theta_{p+S},\mathcal H_t)}{r}.
\end{equation}
\end{corollary}

\begin{proof}
By the assumption on $D$,
$F_t(D)-F_t(S)=\mathbb H(Y^\star\mid\theta_{1:p},\theta_{p+S},\mathcal H_t)$. Submodularity gives
\begin{align}
    F_t(D)-F_t(S)
    &\le \sum_{j\in D\setminus S}
      \bigl[F_t(S\cup\{j\})-F_t(S)\bigr]\\
    &=\sum_{j\in D\setminus S}
      I(Y^\star;\theta_{p+j}\mid\theta_{1:p},\theta_{p+S},\mathcal H_t).
\end{align}
At least one of the $r$ terms in the right-hand side is greater than or equal to their average.
\end{proof}

BALAR does not enumerate the idealized candidate pool or evaluate
\eqref{eq2} exactly. It asks an LLM to propose a
novel dimension.  To state what would follow from a proposal mechanism of
known quality, we abstract it as an approximate oracle.

\begin{assumption}\label{ass2}
At expansion step $i$, let $S_i\subseteq[m]$ index the dimensions already
added. The proposal mechanism returns a random index
$\widehat J_i\in[m]\setminus S_i$ such that
\begin{equation}
    \mathbb E\!\left[
      I(Y^\star;\theta_{p+\widehat J_i}\mid\theta_{1:p},\theta_{p+S_i},\mathcal H_t)\mid S_i
    \right]
    \ge \eta\max_{j\notin S_i}
      I(Y^\star;\theta_{p+j}\mid\theta_{1:p},\theta_{p+S_i},\mathcal H_t)
    \label{eq3}
\end{equation}
for some $\eta\in(0,1]$, and $S_{i+1}=S_i\cup\{\widehat J_i\}$.
\end{assumption}

The following theorem concerns $k$ expansions at the fixed history $\mathcal H_t$.

\begin{theorem}
Assume the conditions of \cref{thm2} and \cref{ass2}, and let $1\le k\le m$.  Let
$S_k^\star\in\arg\max_{|S|\le k}F_t(S)$, and let $S_k$ be the set obtained
after $k$ calls to the approximate oracle.  Then
\begin{equation}
    \mathbb E[F_t(S_k)]
    \ge \left[1-\left(1-\frac{\eta}{k}\right)^k\right]F_t(S_k^\star)
    \ge (1-e^{-\eta})F_t(S_k^\star).
\end{equation}
For an exact oracle ($\eta=1$), this becomes the $(1-1/e)$ guarantee.
\end{theorem}

\begin{proof}
Monotonicity and submodularity imply
\begin{align}
    F_t(S_k^\star)-F_t(S_i)
    &\le F_t(S_i\cup S_k^\star)-F_t(S_i)\\
    &\le k\max_{j\notin S_i}
      \bigl[F_t(S_i\cup\{j\})-F_t(S_i)\bigr].
\end{align}
Combining this inequality with \eqref{eq3} gives
\begin{equation}
    \mathbb E[F_t(S_{i+1})-F_t(S_i)\mid S_i]
    \ge \frac{\eta}{k}\bigl[F_t(S_k^\star)-F_t(S_i)\bigr].
\end{equation}
Therefore, conditioned on $S_i$,
\begin{align}
    \mathbb E\!\left[F_t(S_k^\star)-F_t(S_{i+1})\mid S_i\right]
    &=F_t(S_k^\star)-F_t(S_i)
      -\mathbb E\!\left[F_t(S_{i+1})-F_t(S_i)\mid S_i\right]\\
    &\le \left(1-\frac{\eta}{k}\right)
      \bigl[F_t(S_k^\star)-F_t(S_i)\bigr].
\end{align}
Taking expectations on both sides gives
\begin{equation}
    F_t(S_k^\star)-\mathbb E[F_t(S_{i+1})]
    \le \left(1-\frac{\eta}{k}\right)
      \left(F_t(S_k^\star)-\mathbb E[F_t(S_i)]\right).
\end{equation}
Since $S_0=\varnothing$ and $F_t(\varnothing)=0$, iterating this inequality
for $i=0,\ldots,k-1$ yields
\begin{equation}
    F_t(S_k^\star)-\mathbb E[F_t(S_k)]
    \le \left(1-\frac{\eta}{k}\right)^k F_t(S_k^\star).
\end{equation}
Rearranging,
\begin{equation}
    \mathbb E[F_t(S_k)]
    \ge \left[1-\left(1-\frac{\eta}{k}\right)^k\right]
      F_t(S_k^\star)
    \ge (1-e^{-\eta})F_t(S_k^\star),
\end{equation}
where the last inequality follows from
$(1-\eta/k)^k\le e^{-\eta}$.
\end{proof}

Adding $\theta_{p+1}$ is useful only if subsequent questions reveal it. Let
$W_{p+1}$ be the collection of responses to the questions targeting $\theta_{p+1}$.

\begin{assumption}\label{ass3}
The responses $W_{p+1}$ contain no task-relevant information beyond that
mediated by the expanded latent state:
\begin{equation}
    Y^\star\perp\!\!\!\perp W_{p+1}\mid\theta_{1:p},\theta_{p+1},\mathcal H_t.
\end{equation}
\end{assumption}

\begin{lemma}
Under \cref{ass3},
\begin{align}
    I(Y^\star;W_{p+1}\mid\theta_{1:p},\mathcal H_t)
    &= I(Y^\star;\theta_{p+1}\mid\theta_{1:p},\mathcal H_t)
       -I(Y^\star;\theta_{p+1}\mid\theta_{1:p},\mathcal H_t,W_{p+1})\\
    &\ge I(Y^\star;\theta_{p+1}\mid\theta_{1:p},\mathcal H_t)
       -\mathbb H(\theta_{p+1}\mid\theta_{1:p},\mathcal H_t,W_{p+1}).
    \label{eq4}
\end{align}
Consequently, if the targeted transcript leaves at most $\varepsilon$ entropy
about $\theta_{p+1}$, it realizes all but at most $\varepsilon$ of the dimension's
representation value.
\end{lemma}

\begin{proof}
Applying the chain rule with $\theta_{p+1}$ gives
\begin{align}
  I(Y^\star;\theta_{p+1},W_{p+1}\mid\theta_{1:p},\mathcal H_t)
  ={}& I(Y^\star;\theta_{p+1}\mid\theta_{1:p},\mathcal H_t) +I(Y^\star;W_{p+1}\mid\theta_{1:p},\theta_{p+1},\mathcal H_t).
\end{align}
The second term is zero by \cref{ass3}, so
\begin{equation}
  I(Y^\star;\theta_{p+1},W_{p+1}\mid\theta_{1:p},\mathcal H_t)
  =I(Y^\star;\theta_{p+1}\mid\theta_{1:p},\mathcal H_t).
\end{equation}
Applying the chain rule in the opposite order gives
\begin{align}
  I(Y^\star;\theta_{p+1},W_{p+1}\mid\theta_{1:p},\mathcal H_t)
  ={}& I(Y^\star;W_{p+1}\mid\theta_{1:p},\mathcal H_t)+I(Y^\star;\theta_{p+1}\mid\theta_{1:p},\mathcal H_t,W_{p+1}).
\end{align}
Therefore,
\begin{align}
  I(Y^\star;W_{p+1}\mid\theta_{1:p},\mathcal H_t)
  ={}& I(Y^\star;\theta_{p+1}\mid\theta_{1:p},\mathcal H_t)-I(Y^\star;\theta_{p+1}\mid\theta_{1:p},\mathcal H_t,W_{p+1}).
\end{align}
Finally, substituting
$
  I(Y^\star;\theta_{p+1}\mid\theta_{1:p},\mathcal H_t,W_{p+1})
  \le \mathbb H(\theta_{p+1}\mid
       \theta_{1:p},\mathcal H_t,W_{p+1})
$
into the preceding equality proves
\eqref{eq4}.
\end{proof}

Expansion can also break a plateau in the current question bank. Suppose every
unasked pair in the current bank satisfies
\begin{equation}
    \max_{(q,u)\,\mathrm{unasked}}
    I(\theta_{1:p};Y_{q,u}\mid\mathcal H_t;\,q,u)
    \le \varepsilon.
\end{equation}
Let $(q_{\mathrm{new}},u_{\mathrm{new}})$ be a newly generated pair targeting
$\theta_{p+1}$, and denote
$
    \nu = I(\theta_{p+1};Y_{q_{\mathrm{new}},u_{\mathrm{new}}}
      \mid\theta_{1:p},\mathcal H_t;\,q_{\mathrm{new}},u_{\mathrm{new}})
$. The chain rule then gives
\begin{equation}
\begin{aligned}
    I((\theta_{1:p},\theta_{p+1});Y_{q_{\mathrm{new}},u_{\mathrm{new}}}
      \mid\mathcal H_t;\,q_{\mathrm{new}},u_{\mathrm{new}})
    &= I(\theta_{1:p};Y_{q_{\mathrm{new}},u_{\mathrm{new}}}
      \mid\mathcal H_t;\,q_{\mathrm{new}},u_{\mathrm{new}})\\
    &\quad + I(\theta_{p+1};Y_{q_{\mathrm{new}},u_{\mathrm{new}}}
      \mid\theta_{1:p},\mathcal H_t;\,q_{\mathrm{new}},u_{\mathrm{new}})\ge\nu.
\end{aligned}
\end{equation}
Thus $\nu>\varepsilon$ exposes a one-step information opportunity larger than
any pair available before expansion.

\section{Graphical Model}
\label{subsec:pgm-view}

The complete BALAR process is not represented by one fixed-structure
probabilistic graphical model: \textsc{Expand} can change the latent variables
and question bank during interaction, and subsequent actions are selected
adaptively from the observed history. Conditional on a realized sequence of
\textsc{Ask} and \textsc{Expand} actions, however, each stage admits a
graphical-model interpretation. \Cref{fig:balar-pgm-trajectory} illustrates
one such trajectory. The initial joint prior factorizes across dimensions. An
observed response to a selected question then generally induces posterior
dependence between those dimensions. Expansion adds a new latent dimension
using the conversation-conditioned prior described in
\cref{subsec:expand}, augments the question bank, and permits further belief
updates. The resulting joint posterior and full conversation history finally
condition the final answer generation call.

\begin{figure}[H]
  \centering
  \includegraphics[width=\linewidth]{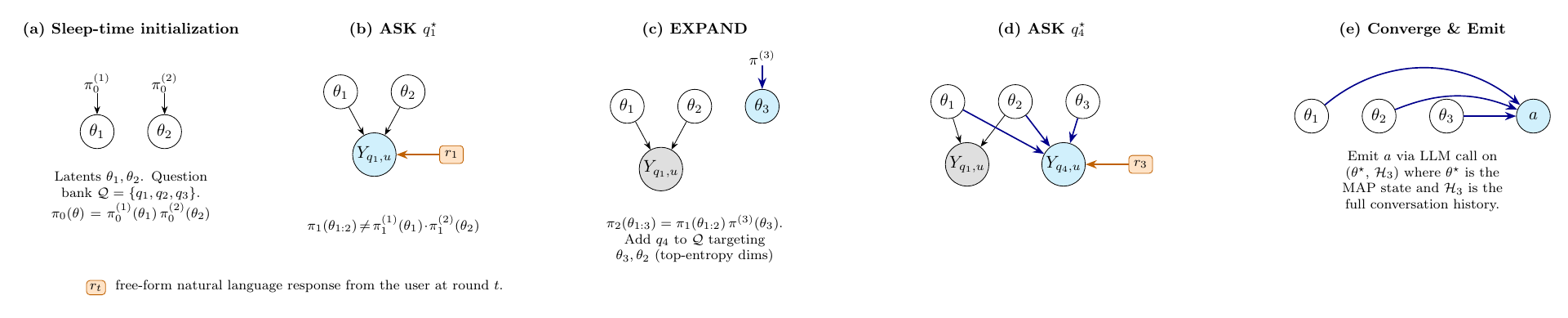}
  \caption{Graphical-model view of one realized BALAR trajectory.
  The sequence proceeds from factorized sleep-time initialization through an
  \textsc{Ask} action, an \textsc{Expand} action that introduces $\theta_3$
  and a targeted question $q_4$, a further \textsc{Ask}, and final answer
  generation. The diagram represents the probabilistic dependencies along
  this realized adaptive execution.}
  \label{fig:balar-pgm-trajectory}
\end{figure}

\clearpage
\section{BALAR Illustrations}
\label{app:balar-illustrations}

\begin{figure}[!htbp]
  \centering
  \includegraphics[width=0.80\fscale]{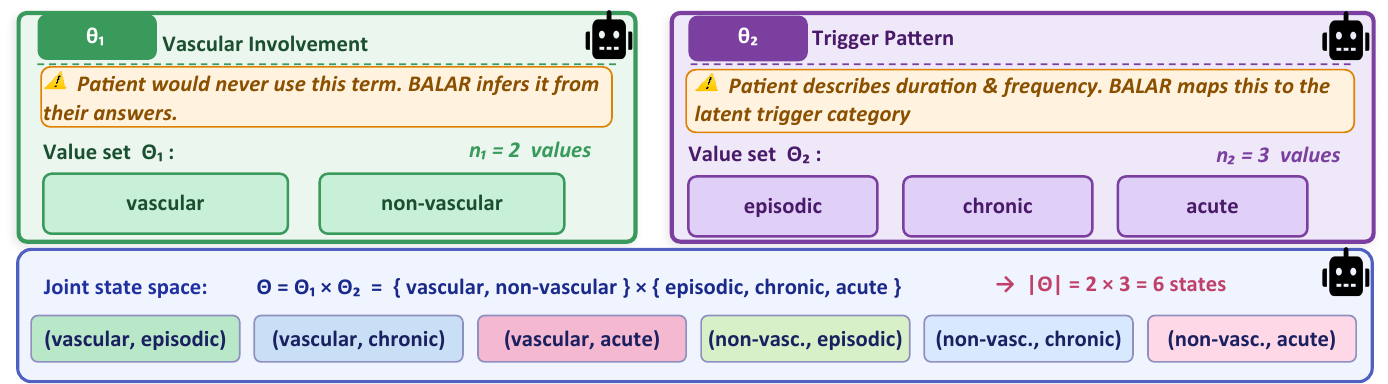}
  \caption{A single LLM call proposes two dimensions: $\theta_1 =$ \emph{Vascular Involvement}
  ($\Theta_1 = \{\text{vascular}, \text{non-vascular}\}$) and $\theta_2 =$
  \emph{Trigger Pattern} ($\Theta_2 = \{\text{episodic}, \text{chronic}, \text{acute}\}$),
  yielding a joint state space of $|\Theta| = 6$ states.}
  \label{fig:ex-step1}
\end{figure}

\begin{figure}[!htbp]
  \centering
  \includegraphics[width=0.90\fscale]{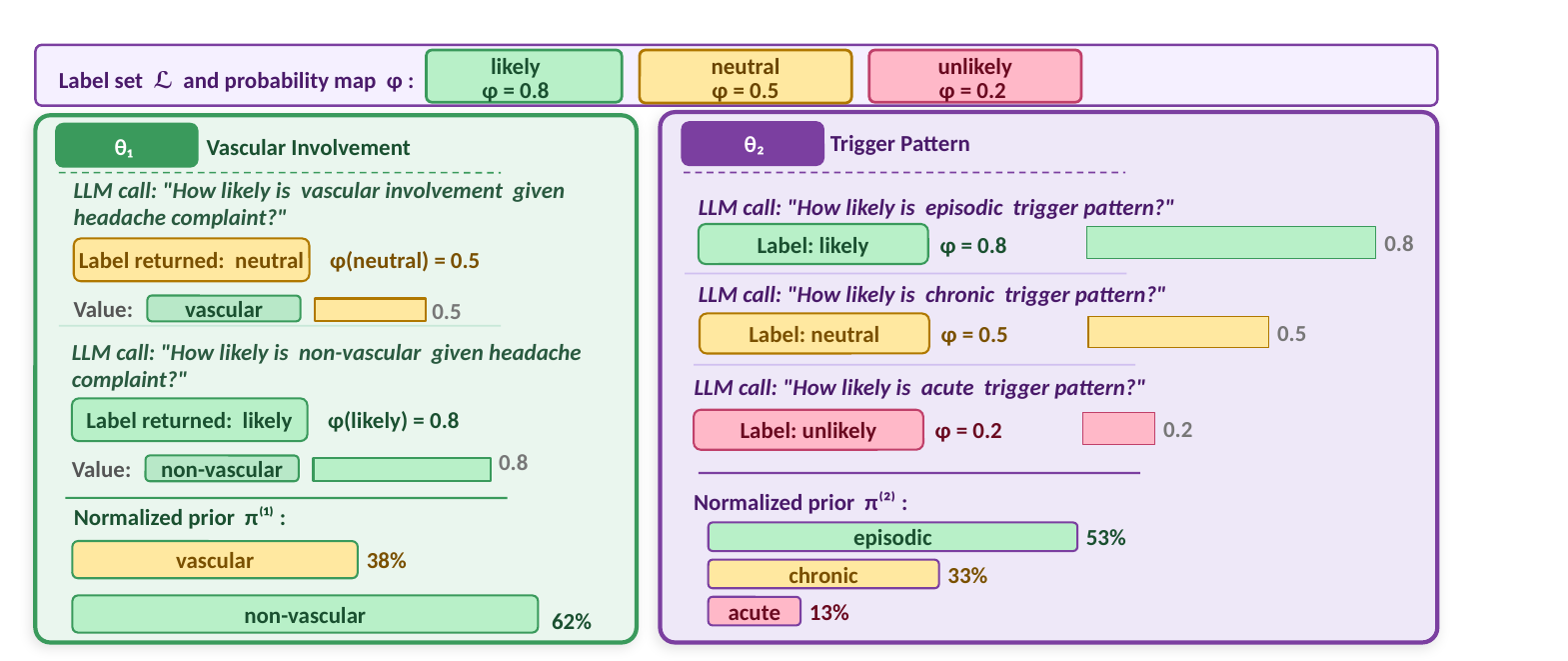}
  \caption{Parallel LLM calls assign a prior label $\ell^{\mathrm p} \in \mathcal{L}$ to each dimension value.
  Here, the LLM judges \emph{vascular} as \texttt{neutral} and
  \emph{non-vascular} as \texttt{likely}, yielding
  $\pi^{(1)} = [0.38, 0.62]$, while \emph{episodic} is \texttt{likely},
  \emph{chronic} \texttt{neutral}, \emph{acute} \texttt{unlikely},
  giving $\pi^{(2)} = [0.53, 0.33, 0.13]$.}
  \label{fig:ex-step2}
\end{figure}

\begin{figure}[!htbp]
  \centering
  \includegraphics[width=0.90\fscale]{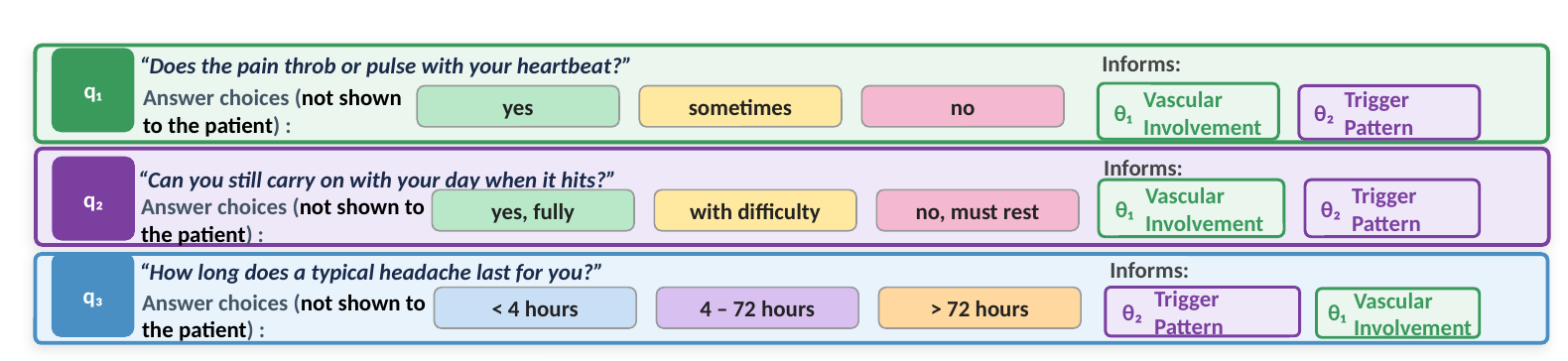}
  \caption{An LLM call generates $|\mathcal{Q}|=3$ questions
  conditioned on the proposed dimensions. Each question informs many
  dimensions simultaneously. No question uses medical jargon.}
  \label{fig:ex-step3}
\end{figure}

\begin{figure}[!htbp]
  \centering
  \includegraphics[width=\fscale]{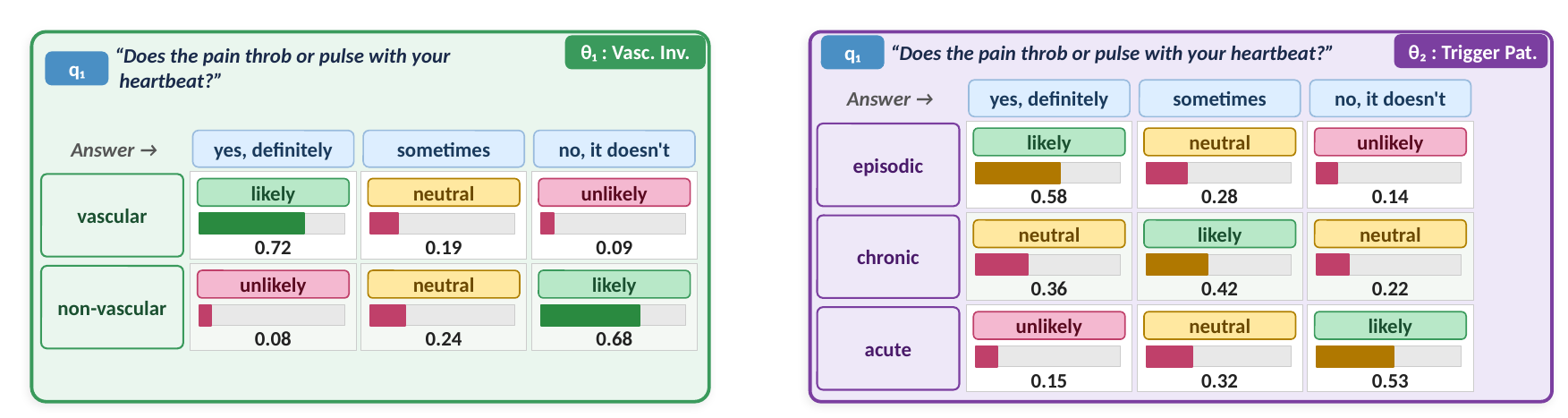}
  \caption{Six parallel LLM calls fill the $|\mathcal{Q}| \times p = 3 \times 2$ likelihood
  matrices. Shown here are the two tables for $q_1$: vascular states are
  much more likely to answer \emph{``yes, definitely''} ($L=0.72$) than
  non-vascular states ($L=0.08$).}
  \label{fig:ex-step4}
\end{figure}

\begin{figure}[!htbp]
  \centering
  \includegraphics[width=0.86\fscale]{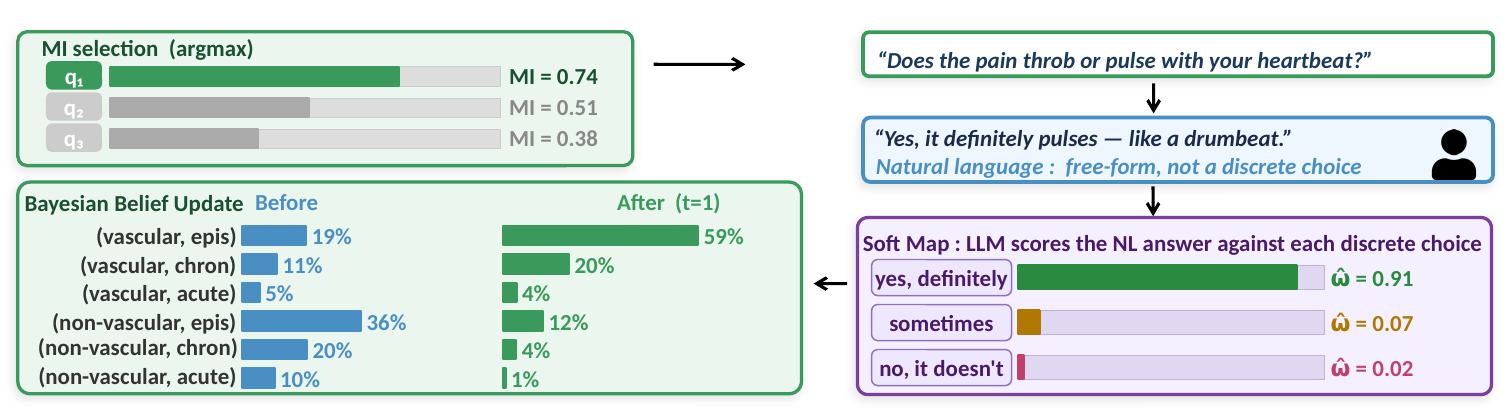}
  \caption{$q_1$ has the highest MI and is asked first.
  The patient replies in free-form natural language. The soft-map LLM
  assigns $\hat{\omega} = [0.91, 0.07, 0.02]$ over the discrete choices.
  A Bayesian update shifts mass toward vascular states:
  $(\text{vasc}, \text{epis})$ rises from $19\%$ to $59\%$.}
  \label{fig:ex-round1}
\end{figure}

\begin{figure}[!htbp]
  \centering
  \includegraphics[width=\fscale]{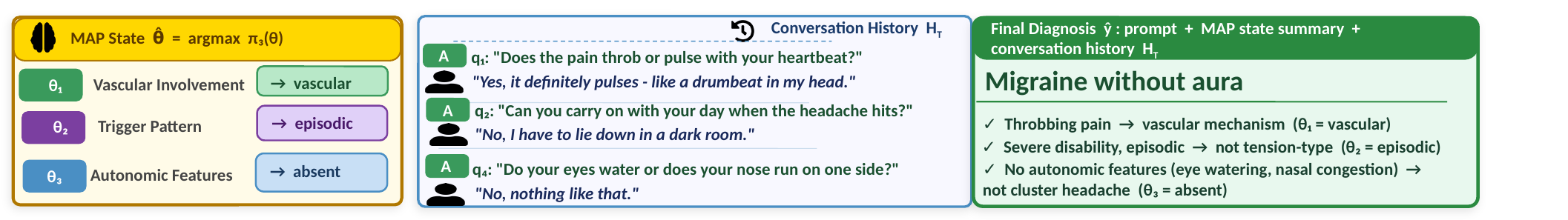}
  \caption{The MAP state $\hat{\theta} = (\text{vascular}, \text{episodic}, \text{absent})$, combined with the conversation history
  $\mathcal{H}_T$, conditions a final LLM call that produces the diagnosis \emph{Migraine without aura}.}
  \label{fig:ex-final}
\end{figure}

\clearpage

\section{Computational Complexity}
\label{subsec:complexity}

We characterize both the number of LLM calls made by BALAR and the cost
of exact Bayesian inference. Parallelism is used to keep wall-clock latency
low.
Let $p$ denote the number of dimensions, $n$ the maximum number of values per
dimension, $|\mathcal{Q}|$ the number of questions, $|\mathcal{U}|$ the number of users,
$|\mathcal{A}|$ the number of possible answers (when defined), and
$|\mathcal{Q}_{\mathrm{new}}|$ the number of new questions generated per
\textsc{Expand} step. For the inference analysis, let
$S=|\Theta|=\prod_{j=1}^p n_j$ be the number of joint states and let
$c_Y=\max_{q\in\mathcal Q}|\mathcal Y_q|$ be the largest discrete answer set.

\paragraph{Sleep-time initialization (Steps 1--4).}
Step~1 issues a single LLM call to propose all $p$ dimensions jointly.
Step~2 elicits one call per dimension value to assign a prior label, for
a total of $\sum_j n_j \leq p \cdot n$ calls. These are dispatched
concurrently.
Step~3 issues a single call to generate all $|\mathcal{Q}|$ initial
questions.
Step~4 issues one call per $(q, u, \theta_j)$ triple to fill the
likelihood table for that triple, for a total of
$|\mathcal{Q}| \cdot |\mathcal{U}| \cdot p$ calls, all dispatched concurrently.
When possible answers $\mathcal{A}$ are available, an additional
$p$ parallel calls estimate the answer-level likelihood tables
(one per dimension).
The total number of initialization calls is therefore:
\[
  C_{\text{init}} = 1 + p\cdot n + 1 + |\mathcal{Q}|\cdot |\mathcal{U}| \cdot p + p
  \;=\; \mathcal{O}(|\mathcal{Q}| \cdot |\mathcal{U}| \cdot p),
\]
all parallelizable except for the sequential dependency between
Steps~1--2--3--4 (each step uses outputs of the previous one).

\paragraph{ASK step.}
Each \textsc{Ask} round issues exactly 2 sequential LLM calls: one to the
user simulator to obtain the natural language answer $r_t$, and one to
soft-map $r_t$ to a probability vector $\hat{\omega}$ over the discrete
choices. The MI computation over all candidate $(q,u)$ pairs is a single
batched tensor operation (no LLM call). Over $T_{\text{ask}}$ rounds, the
total is $2 T_{\text{ask}}$ sequential calls.

\paragraph{EXPAND step.}
Each \textsc{Expand} round issues the following calls:
\begin{enumerate}[topsep=2pt,itemsep=1pt,leftmargin=*]
  \item One call to propose the new dimension $\theta_{p+1}$.
  \item $n_{p+1}$ parallel calls to elicit its prior (one per value).
  \item One call (optional) for the answer-level likelihood of the new
  dimension (when $\mathcal{A}$ is defined).
  \item $|\mathcal{Q}| \cdot |\mathcal{U}|$ parallel calls to compute likelihood tables
  for all old questions over the new dimension.
  \item One call to generate up to $|\mathcal{Q}_{\mathrm{new}}|$ new questions.
  \item $|\mathcal{Q}_{\mathrm{new}}| \cdot |\mathcal{U}| \cdot (p+1)$ parallel calls to compute
  likelihood tables for the new questions over all dimensions.
\end{enumerate}
The dominant cost per \textsc{Expand} round is
$\mathcal{O}((|\mathcal{Q}| + |\mathcal{Q}_{\mathrm{new}}|(p+1)) \cdot |\mathcal{U}|)$ LLM calls,
all within each group executed concurrently.

\paragraph{Exact Bayesian inference.}
Constructing every state-level likelihood tensor $K_{q,u}(\theta,y)$ from its
dimension-level factors costs
$\mathcal O(|\mathcal Q||\mathcal U|pSc_Y)$. These
tensors can then be reused until an \textsc{Expand} step changes the state or
question space. At each \textsc{Ask} round, computing predictive answer
distributions and conditional entropies for every unasked question--user pair
costs
\begin{equation}
  \mathcal O(|\mathcal Q||\mathcal U|Sc_Y).
  \label{eq:inference-time}
\end{equation}
The posterior update for the selected pair costs $\mathcal O(Sc_Y)$, computing
all $p$ dimension marginals costs $\mathcal O(pS)$, and computing the
predictive answer distribution when $\mathcal A$ is available costs
$\mathcal O(S|\mathcal A|)$. Thus the MI evaluation in
\cref{eq:inference-time} is normally the dominant tensor operation per round.

Storing the joint belief requires $\mathcal O(S)$ memory while all state-level
likelihoods require
\begin{equation}
  \mathcal O(|\mathcal Q||\mathcal U|Sc_Y).
  \label{eq:inference-memory}
\end{equation}
After expansion, the same bounds apply with the enlarged state count
$S_{\mathrm{new}}=S\,n_{p+1}$ and the augmented question bank. These operations
are implemented as batched tensor computations and require no additional LLM
calls. Because $S$ grows exponentially with the number of dimensions, BALAR
enforces the state-space cap described in Appendix~\ref{subsec:practical}.

\paragraph{Final answer.}
A single LLM call produces the final answer conditioned on the MAP state
and conversation history.

\paragraph{Empirical call counts.}
\Cref{tab:empirical-calls} reports mean LLM calls per dialogue for the six
main-results settings. Initialization dominates AR-Bench-DC with Qwen-32B,
whereas expansion dominates iCraft-MD with Qwen-32B and accounts for a
substantial share of calls on AR-Bench-SP. The mean totals range from
$83.6\pm2.3$ to $362.0\pm7.6$. Llama-8B and Qwen-32B denote
\texttt{Llama-3.1-8B-Instruct} and \texttt{Qwen2.5-32B-Instruct}, respectively.

\begin{table}[htbp]
  \centering
  \caption{LLM calls per dialogue for the main-results settings, reported as
  means $\pm$ standard errors. Stages are initialization, ASK, expansion,
  and final-answer generation.}
  \label{tab:empirical-calls}
  \begin{adjustbox}{max width=\linewidth}
  \begin{tabular}{llrrrrr}
    \toprule
    Dataset & Model & Initialization & ASK & Expansion & Final & Total \\
    \midrule
    AR-Bench-DC & Llama-8B & $17.0\pm0.0$ & $20.0\pm0.0$ & $45.6\pm2.3$ & $1.0\pm0.0$ & $83.6\pm2.3$ \\
    AR-Bench-DC & Qwen-32B & $275.9\pm0.2$ & $46.2\pm1.2$ & $38.9\pm7.4$ & $1.0\pm0.0$ & $362.0\pm7.6$ \\
    AR-Bench-SP & Llama-8B & $70.8\pm0.2$ & $38.2\pm0.4$ & $107.0\pm3.4$ & $1.0\pm0.0$ & $217.0\pm3.6$ \\
    AR-Bench-SP & Qwen-32B & $67.7\pm0.3$ & $42.6\pm1.1$ & $158.4\pm6.7$ & $1.0\pm0.0$ & $269.8\pm7.6$ \\
    iCraft-MD & Llama-8B & $47.1\pm0.3$ & $10.0\pm0.8$ & $40.1\pm3.2$ & $1.0\pm0.0$ & $98.2\pm4.0$ \\
    iCraft-MD & Qwen-32B & $8.2\pm0.1$ & $42.3\pm1.0$ & $254.5\pm7.6$ & $1.0\pm0.0$ & $306.0\pm8.5$ \\
    \bottomrule
  \end{tabular}
  \end{adjustbox}
\end{table}

\section{Likelihood labels}

\subsection{Definition}
\label{subsec:label-def}

We use a discrete label set $\mathcal{L} = \{\texttt{likely}, \texttt{neutral}, \texttt{unlikely}\}$ to elicit priors and likelihoods from the LLM. Each label corresponds to a qualitative judgment grounded in the available context:
\begin{itemize}[topsep=2pt,itemsep=1pt,leftmargin=*]
    \item \textbf{likely}: the value or answer is explicitly stated, strongly implied, or represents the most natural assumption given the prompt and context.
    \item \textbf{neutral}: the value or answer is plausible but not supported or contradicted by specific evidence. There is insufficient information to prefer it over alternatives.
    \item \textbf{unlikely}: the value or answer is contradicted by the prompt or context, or would require assumptions inconsistent with the provided information.
\end{itemize}
These qualitative labels are mapped to positive unnormalized weights via a fixed function $\phi$, which are normalized within each prior or likelihood table to produce numerical probabilities.

\subsection{Cross-Model Consistency of Label Elicitation}
\label{subsec:label-reliability}

We sample $850$ evaluation units from BALAR traces on iCraft-MD with \texttt{Qwen2.5-32B-Instruct}, stratified across four roles:
$70$ prior units, $80$ question-likelihood units, $100$ answer-likelihood
units, and $600$ soft-map units. A prior unit contains one dimension's value
set, a likelihood unit contains a complete table, and a soft-map unit
contains one user response and its discrete answer choices. Priors and question likelihoods are sampled
from initialization, answer likelihoods from initialization and expansion,
and soft maps from ASK rounds. \texttt{GPT-5-nano} replays the original inputs,
using the same role-specific prompts and fixed weight map as BALAR. These
prompts are provided in Appendix~\ref{app:prompts}.

BALAR elicits qualitative labels from the LLM, maps them to weights through
$\phi$, and normalizes those weights into the probability vectors used by its
Bayesian calculations. We therefore compare the models at two levels: whether
they produce the same labels and whether their labels induce similar
probability distributions. Exact agreement is the fraction of label cells on which the two
models agree. Total variation between normalized probability vectors $p$
and $q$ is
\begin{equation}
  \mathrm{TV}(p,q)=\frac12\sum_a |p_a-q_a|.
\end{equation}
The resulting exact-agreement and mean-TV scores are reported in
\cref{tab:label-reliability}. Soft mapping, which provides evidence for each posterior update, shows the
strongest cross-model consistency.

\begin{table}[htbp]
  \centering
  \caption{Cross-model label consistency between Qwen2.5-32B-Instruct and
  GPT-5-nano on iCraft-MD.}
  \label{tab:label-reliability}
  \begin{tabular}{lrrrr}
    \toprule
    Role & Sampled units & Label cells & Agreement (\%) & Mean TV \\
    \midrule
    Prior & 70 & 227 & 55.5 & 0.139 \\
    Question likelihood & 80 & 750 & 43.6 & 0.202 \\
    Answer likelihood & 100 & 1280 & 49.8 & 0.189 \\
    Soft map & 600 & 1781 & 78.8 & 0.094 \\
    \bottomrule
  \end{tabular}
\end{table}

\section{Semantic Equivalence Metric}
\label{subsec:semantic_equivalence}

For tasks where no single canonical ground-truth answer exists but multiple phrasings can express the same underlying meaning, such as AR-Bench-SP, metrics such as character-level F1 similarity are inadequate: a random permutation of the ground-truth characters can achieve a perfect score while a faithful paraphrase may receive a low one. 

Following \citet{semantic}, let $a$ denote the agent's predicted answer and $a^*$ the ground-truth solution, both conditioned on the same prompt $\mathbf{p}$. We use a judge LLM $\mathcal{J}$ (\texttt{Qwen2.5-32B-Instruct}) to assess the directional entailment relation
\[
  \mathcal{J}(a_1, a_2 \mid \mathbf{p}) \;\in\;
  \{\texttt{entailment},\, \texttt{neutral},\, \texttt{contradiction}\},
\]
which evaluates whether $a_1$ semantically entails $a_2$ in the context of $\mathbf{p}$. We query $\mathcal{J}$ in both directions, obtaining
\[
  e_{\rightarrow} = \mathcal{J}(a, a^* \mid \mathbf{p}),
  \qquad
  e_{\leftarrow} = \mathcal{J}(a^*, a \mid \mathbf{p}).
\]
We declare $a$ and $a^*$ \emph{semantically equivalent} (non strict) when neither direction is a contradiction and the pair is not jointly neutral:
\[
  \mathrm{Eq}(a, a^*) \;=\;
  \mathbf{1}\!\left[
    \left(\texttt{contradiction} \notin \{e_{\rightarrow}, e_{\leftarrow}\}\right)
    \;\wedge\;
    \left((e_{\rightarrow}, e_{\leftarrow}) \neq (\texttt{neutral}, \texttt{neutral})\right)
  \right].
\]
This criterion accepts predictions that are entailed by the reference in at least one direction, provided no direction is contradicted. It reflects the nature of such tasks, where multiple correct phrasings of the hidden explanation are possible and a partial match (one direction entailment, the other neutral) often corresponds to a correct answer that omits minor details.
\section{Hyperparameters}
\label{subsec:hypers}

Fixed parameters are shared across all datasets. The label set is $\mathcal{L} = \{\texttt{likely}, \texttt{neutral},
\texttt{unlikely}\}$ with unnormalized weight map $\phi(\texttt{likely}) = 0.8$,
$\phi(\texttt{neutral}) = 0.5$, $\phi(\texttt{unlikely}) = 0.2$. This mapping is chosen as simple and heuristically reasonable. The ablation in \cref{tab:ablation_grpB} shows it performs well in practice.

For \cref{fig:combined}, the fixed configurations are DC
($T_{\text{ask}}{=}25$, $\alpha{=}0.3$, $p{=}5$,
$|\mathcal{Q}|{=}10$), SP ($T_{\text{ask}}{=}25$, $\alpha{=}0.3$,
$\beta{=}0.5$, $p{=}5$, $|\mathcal{Q}|{=}10$), and iCraft-MD
($T_{\text{ask}}{=}25$, $\alpha{=}0.3$ (top) and $0.1$ (bottom), $p{=}1$, $|\mathcal{Q}|{=}2$).

\begin{table}[htbp]
\centering
\caption{Hyperparameter settings.}
\label{tab:hypers}
\begin{adjustbox}{width=0.5\textwidth}
\begin{tabular}{lcc}
\toprule
\textbf{Parameter} & \textbf{DC/iCraft-MD} & \textbf{SP} \\
\midrule
\multicolumn{3}{l}{\emph{Variable}} \\
Confidence threshold ($\alpha$)         & $\{0.1, 0.3\}$ & $\{0.1, 0.3\}$ \\
Initial dims / questions ($(p, |\mathcal{Q}|)$) & \multicolumn{2}{c}{$\{(1,2),(5,3),(5,5),(5,10)\}$} \\
Marginal convergence ($\beta$)          & $1.0$ (fixed) & $\{0.5, 0.7\}$ \\
\midrule
\multicolumn{3}{l}{\emph{Fixed}} \\
Total round budget ($T$)               & \multicolumn{2}{c}{$100$} \\
Max ask rounds ($T_{\text{ask}}$)       & \multicolumn{2}{c}{$25$} \\
Max values per dim                     & \multicolumn{2}{c}{$4$} \\
Max choices per question               & $4$  & $2$ (yes/no) \\
Max new questions per EXPAND           & \multicolumn{2}{c}{$4$} \\
Top-entropy dims for EXPAND ($m_{\mathrm{top}}$) & \multicolumn{2}{c}{$2$} \\
Expand multiplier ($\lambda$)          & \multicolumn{2}{c}{$1.0$} \\
Temperature / top-$p$                  & \multicolumn{2}{c}{$0.1$ / $1.0$} \\
\bottomrule
\end{tabular}
\end{adjustbox}
\end{table}

\paragraph{Choosing hyperparameters.}
BALAR's hyperparameters include the confidence threshold $\alpha$,
the marginal convergence fraction $\beta$, the ASK budget $T_{\text{ask}}$,
the initial number of dimensions $p$ and questions $|\mathcal{Q}|$, and the
expand multiplier $\lambda$. The appropriate values of these parameters are task-dependent. The threshold $\alpha$ should reflect the acceptable residual uncertainty: for high-stakes decisions (e.g.\ medical
diagnosis), a small $\alpha$ is appropriate, whereas for lower-stakes tasks larger values reduce unnecessary
questioning. The budget $T_{\text{ask}}$ should be set to the maximum number of questions a user is willing to answer. $(p, |\mathcal{Q}|)$ trades off initialization cost against the expressiveness of the prior representation. 
Starting with $(p{=}1, |\mathcal{Q}|{=}2)$ and relying on \textsc{Expand} is a natural default when the structure of ambiguity is unknown. 
$\lambda$ controls how aggressively BALAR triggers expansion.
Smaller values expand more aggressively, whereas larger values favor asking from the current bank.

\section{Prior Sensitivity}
\label{subsec:prior_sens}
\begin{table}[H]
  \centering
  \caption{Prior sensitivity ablation on iCraft-MD / Qwen2.5-32B-Instruct. Config: $T_{\text{ask}}{=}25$, $\alpha{=}0.1$, $p{=}1$, $|\mathcal{Q}|{=}2$.}
  \label{tab:ablation_grpB}
  \begin{tabular}{lc}
  \toprule
  \textbf{Variant} & \textbf{Accuracy (\%)} \\
  \midrule
  \textbf{Default priors (0.8/0.5/0.2)} & 73.6 ± 3.7  \\
  Flat priors (0.6/0.5/0.4) & 70.7 ± 3.8 \\
  Sharp priors (0.9/0.5/0.1) & 52.1 ± 4.2  \\
  \bottomrule
  \end{tabular}
\end{table}
We study the effect of different prior mappings. \cref{tab:ablation_grpB} shows that BALAR is robust to moderate changes in the
prior mapping: flatter priors incur only a $3.9\%$ relative drop. However,
overly sharp priors degrade performance substantially ($29.1\%$ relative drop),
as they over-concentrate the initial belief.

\section{Question Quality}
\label{subsec:question-quality}

We evaluate questions from iCraft-MD dialogues under
\texttt{Qwen2.5-32B-Instruct} using a \texttt{GPT-5-nano} judge. For each
question, the judge receives the public case, candidate diagnoses, the
complete set of initial and expanded BALAR dimensions, and all previously
asked questions. User responses are not included. The evaluation is
post-hoc, and the judge's ratings do not influence BALAR's actions.

The judge assigns scores from $1$ to $5$ for three criteria:
relevance to the case, expected
informativeness for narrowing the candidate diagnoses, and non-redundancy against prior questions. Non-redundancy therefore
concerns the previously asked questions, not the full dialogue including
patient responses. The exact prompts appear in
Appendix~\ref{app:prompts:question-quality}.

The judge also identifies one primary target dimension per question, or none. We report the number
of distinct assigned dimensions per dialogue, which uses the full set of initial and expanded dimensions.
\Cref{tab:question-quality} reports dialogue-level means and standard errors. 

The judge rates questions as highly relevant, with more moderate
informativeness and non-redundancy. Each dialogue targets $6.14$ distinct
dimensions on average. These scores characterize perceived question quality
rather than directly measuring realized information gain.

\begin{table}[htbp]
  \centering
  \caption{Question-quality evaluation on iCraft-MD with
  \texttt{Qwen2.5-32B-Instruct}.}
  \label{tab:question-quality}
  \begin{tabular}{lc}
    \toprule
    Metric & Mean $\pm$ SE \\
    \midrule
    Relevance & $4.47\pm0.016$ \\
    Informativeness & $3.63\pm0.014$ \\
    Non-redundancy & $3.69\pm0.012$ \\
    Distinct targeted dimensions & $6.14\pm0.140$ \\
    \bottomrule
  \end{tabular}
\end{table}

\section{Predictive-Belief Calibration}
\label{subsec:calibration-details}

For each model-dataset pair in \cref{tab:calibration}, we assess the final
predictive answer distribution $\hat p(a)$ defined in \cref{subsec:final}.
Motivated by recent work on calibration in agentic systems with explicit
linguistic belief states \citep{murphy2026agentic}, we use leave-one-out
temperature scaling. For each held-out instance, the temperature $\tau>0$ is
fit on the remaining instances by minimizing multinomial negative
log-likelihood (NLL), and the held-out distribution becomes :
\[
  \hat p_{\mathrm{cal}}(a)
  = \frac{\hat p(a)^{1/\tau}}
  {\sum_{a'\in\mathcal A}\hat p(a')^{1/\tau}}
  = \operatorname{softmax}\!\left(\frac{\log \hat p(a)}{\tau}\right).
\]
Thus $\tau>1$ softens an overconfident distribution, whereas $\tau<1$
sharpens an underconfident one. We report expected calibration error (ECE),
NLL, and the multiclass Brier score. Lower values are better for all three.

\begin{table}[H]
  \centering
  \caption{Calibration of BALAR's predictive answer belief using
  leave-one-out temperature scaling. ``Cal.'' denotes the corresponding metric
  after calibration. Lower is better.}
  \label{tab:calibration}
  \scriptsize
  \setlength{\tabcolsep}{3.2pt}
  \begin{tabular}{llrrrrrrr}
  \toprule
  \textbf{Dataset} & \textbf{Model} & $\boldsymbol{\tau}$ &
  \textbf{ECE} & \textbf{Cal.} & \textbf{NLL} & \textbf{Cal.} &
  \textbf{Brier} & \textbf{Cal.} \\
  \midrule
  iCraft-MD & Llama-8B & 3.256 & 0.2627 & 0.0439 & 1.6498 & 1.2771 & 0.7951 & 0.6909 \\
  iCraft-MD & Qwen-7B & 3.909 & 0.3853 & 0.0863 & 1.9083 & 1.2691 & 0.8691 & 0.6829 \\
  iCraft-MD & Qwen-14B & 2.762 & 0.2656 & 0.0862 & 1.4875 & 1.1160 & 0.6807 & 0.5997 \\
  iCraft-MD & Qwen-32B & 1.531 & 0.1315 & 0.0668 & 0.9773 & 0.9211 & 0.5066 & 0.4926 \\
  iCraft-MD & QwQ-32B & 1.514 & 0.1205 & 0.0839 & 1.0952 & 1.0501 & 0.5914 & 0.5728 \\
  \midrule
  AR-DC & Llama-8B & 0.328 & 0.2836 & 0.1010 & 1.2497 & 1.0691 & 0.6414 & 0.5396 \\
  AR-DC & Qwen-32B & 0.465 & 0.1469 & 0.1309 & 0.9500 & 0.8178 & 0.4943 & 0.4442 \\
  AR-DC & QwQ-32B & 0.565 & 0.1721 & 0.1361 & 1.1176 & 1.0545 & 0.5838 & 0.5587 \\
  \bottomrule
  \end{tabular}
\end{table}

Temperature scaling improves all three metrics for every evaluated pair. On
iCraft-MD, all fitted temperatures exceed one, indicating overconfidence before
scaling, and scaling reduces ECE by factors ranging from $1.4$ to $6.0$.
Calibration generally improves with model capacity: the two 32B models
have the lowest unscaled ECE, although the trend is not strictly monotone across
the smaller models. On AR-DC, every fitted temperature is below one, indicating
underconfidence, with smaller but consistent improvements across all three
metrics. These opposite directions also show that a single global temperature
should not be assumed to transfer across tasks.

BALAR does not emit $\arg\max_a\hat p(a)$ directly. As described in
\cref{subsec:final}, it passes the MAP latent state and dialogue history to a
final LLM reasoning call. \Cref{tab:argmax-emitted} compares the two outputs.

\begin{table}[H]
  \centering
  \caption{Accuracy of the predictive-belief argmax and the answer emitted by
  BALAR's final LLM call.}
  \label{tab:argmax-emitted}
  \small
  \begin{tabular}{llrr}
  \toprule
  \textbf{Dataset} & \textbf{Model} &
  \textbf{Argmax (\%)} & \textbf{Emitted (\%)} \\
  \midrule
  iCraft-MD & Llama-8B & 44.3 & 57.1 \\
  iCraft-MD & Qwen-7B & 42.1 & 50.7 \\
  iCraft-MD & Qwen-14B & 54.3 & 62.1 \\
  iCraft-MD & Qwen-32B & 62.9 & 73.6 \\
  iCraft-MD & QwQ-32B & 55.0 & 70.0 \\
  \midrule
  AR-DC & Llama-8B & 62.0 & 61.0 \\
  AR-DC & Qwen-32B & 63.0 & 47.0 \\
  AR-DC & QwQ-32B & 58.0 & 51.0 \\
  \bottomrule
  \end{tabular}
\end{table}

The emitted answer is consistently more accurate on iCraft-MD, suggesting that
the final reasoning call can use the structured state and dialogue context more
effectively than a direct answer-belief argmax in that domain. This benefit does
not extend to AR-DC, where the emitted answer is slightly worse for Llama-8B and
both Qwen-32B models.

\section{Answer-Belief Dynamics}
\label{subsec:answer-belief-dynamics}

We examine concentration of the predictive answer belief
$\hat p_t(a)$ in \cref{eq:answer-probability} on iCraft-MD with
\texttt{Qwen2.5-32B-Instruct}. At each saved belief state, we compute
\begin{align}
  H(t)&=-\sum_{a\in\mathcal A}
  \hat p_t(a)\log\hat p_t(a),\\
  C(t)&=\max_{a\in\mathcal A}\hat p_t(a).
\end{align}
Entropy is measured in nats. These quantities describe the answer belief
used for stopping, not the latent-state entropy or the final LLM-emitted
answer.

\Cref{fig:answer-belief-dynamics} plots instance means against cumulative
ASK count. For each dialogue and ASK count, we use the last available belief
state at that count, including any intervening expansion. Each point averages
only dialogues reaching that count. Mean entropy decreases from approximately $1.4$
to $0.82$ nats and mean maximum probability increases from approximately
$0.35$ to $0.66$. These trends indicate increasing belief concentration, not
necessarily attainment of the stopping threshold or convergence to the correct
answer.

\begin{figure}[htbp]
  \centering
  \includegraphics[width=\linewidth]{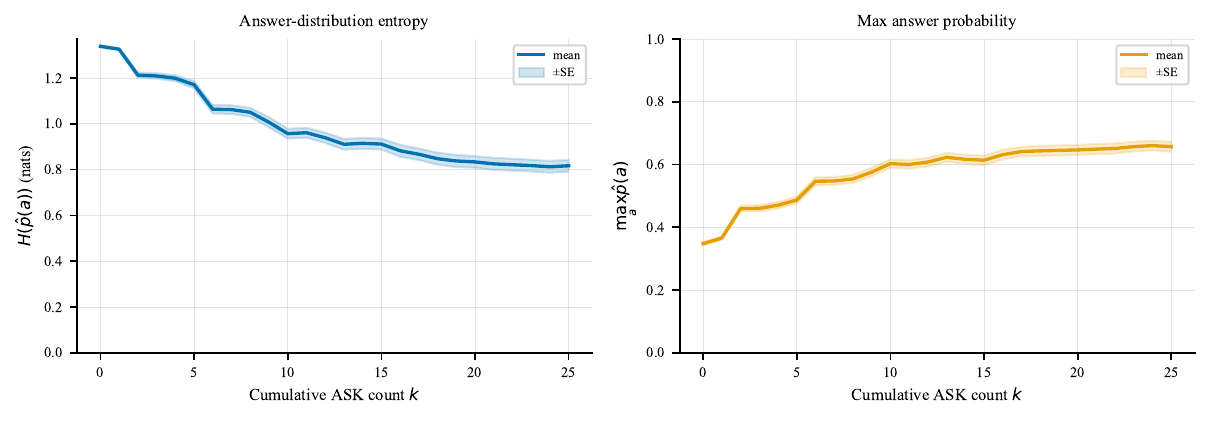}
  \caption{Answer-belief entropy (left) and maximum answer probability
  (right) versus cumulative ASK count on iCraft-MD with
  \texttt{Qwen2.5-32B-Instruct}. Curves show instance means with
  $\pm$ one standard error.}
  \label{fig:answer-belief-dynamics}
\end{figure}

\section{Sensitivity of Initial Dimensions}
\label{subsec:prompt-sensitivity}

We evaluate how strongly the initial latent representation depends on the
wording of the dimension-generation prompt. On iCraft-MD, we use
\texttt{Qwen2.5-32B-Instruct} to generate five initial dimensions under three
prompt variants:
\begin{itemize}[leftmargin=1.5em]
  \item \textbf{v0}: the standard clinical prompt used in the main experiments,
  \item \textbf{v1}: the same physician system role with definitions and other
  task-specific scaffolding removed,
  \item \textbf{v2}: a domain-neutral formulation containing no clinical
  language.
\end{itemize}

For each pair of variants, a \texttt{GPT-5-nano} judge matches semantically
equivalent dimensions and computes the Jaccard similarity between the resulting
dimension sets. \Cref{tab:prompt-sensitivity} reports the mean across evaluated
cases with standard errors.

\begin{table}[H]
  \centering
  \caption{Semantic overlap between sets of five initial dimensions generated
  under different prompt formulations on iCraft-MD using
  \texttt{Qwen2.5-32B-Instruct}.}
  \label{tab:prompt-sensitivity}
  \begin{tabular}{lc}
  \toprule
  \textbf{Prompt comparison} & \textbf{Mean Jaccard $\boldsymbol{\pm}$ SE} \\
  \midrule
  v0 vs. v1 & $0.678 \pm 0.012$ \\
  v0 vs. v2 & $0.564 \pm 0.012$ \\
  v1 vs. v2 & $0.620 \pm 0.024$ \\
  \bottomrule
  \end{tabular}
\end{table}

The dimension sets retain moderate semantic overlap across all three
formulations. In particular, the overlap involving v2 indicates that the
representation is not determined solely by clinical terminology or detailed
prompt scaffolding. At the same time, v0 and v2 have the lowest agreement,
showing measurable sensitivity to domain-specific framing.

The exact v1 and v2 prompts and the semantic-matching judge prompts
are reproduced in
Appendix~\ref{app:prompts:sensitivity}. v0 is the standard iCraft-MD prompt in
Appendix~\ref{app:prompts:init_dims}.

\section{Round Distribution}
\label{subsec:round_distribution}

We analyze how BALAR allocates its interaction budget across \textsc{Ask} and \textsc{Expand} actions. \cref{fig:round_distributions_qwen32b} shows the distribution of the number of \textsc{Ask} and \textsc{Expand} rounds per instance under \texttt{Qwen2.5-32B-Instruct}. On AR-Bench-DC, nearly all runs
exhaust the full \textsc{Ask} budget with few \textsc{Expand} calls,
suggesting the initial dimensions suffice. On AR-Bench-SP, \textsc{Expand} is called $1$--$5$ times per run. On iCraft-MD, \textsc{Expand} is called $5$--$8$ times consistently, indicating that clinical cases regularly require state refinement.
\begin{figure*}[htbp]
    \centering
    \includegraphics[width=\textwidth]{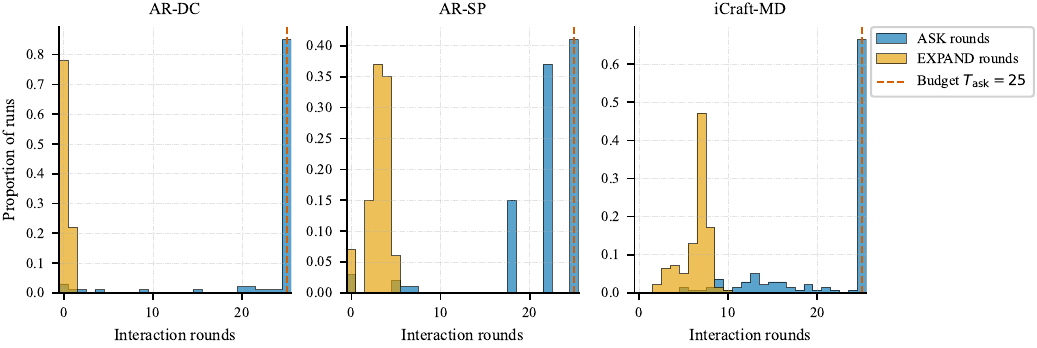}
    \caption{
    Distribution of \textsc{Ask} and \textsc{Expand} rounds for BALAR under \texttt{Qwen2.5-32B-Instruct} on each dataset. Each panel shows normalized histograms over runs, so bar heights correspond to proportions. The dashed vertical line indicates the selected $T_{\mathrm{ask}}$ budget. Fixed configs: DC ($T_{\text{ask}}{=}25$, $\alpha{=}0.3$, $p{=}5$, $|\mathcal{Q}|{=}10$), SP ($T_{\text{ask}}{=}25$, $\alpha{=}0.3$, $\beta{=}0.5$, $p{=}5$, $|\mathcal{Q}|{=}10$), iCraft-MD ($T_{\text{ask}}{=}25$, $\alpha{=}0.1$, $p{=}1$, $|\mathcal{Q}|{=}2$).}
    \label{fig:round_distributions_qwen32b}
\end{figure*}

\section{Oracle}
\label{subsec:oracale}

\paragraph{Oracle.}
The LLM is given the prompt, public context, and all users' private
information, providing an upper bound on performance.

\paragraph{Oracle dimensions.}
The oracle returns a minimal set of ambiguity dimensions and their
ground-truth values. \cref{tab:oracle_dims} reports the average number
of dimensions ($\pm$ SE), measuring intrinsic disambiguation complexity.
\begin{table}[htbp]
  \centering
  \caption{Oracle Dimensions. Entries are the average number of oracle dimensions required to resolve a case, reported as mean $\pm$ SE over instances.}
  \label{tab:oracle_dims}
  \begin{tabular}{lccc}
  \toprule
  \textbf{Model} & \textbf{AR-DC} & \textbf{AR-SP} & \textbf{iCraft-MD} \\
  \midrule
  Llama-3.1-8B-Instruct & 4.7 ± 0.05 & 4.2 ± 0.1 & 5.3 ± 0.1 \\
  Qwen2.5-32B-Instruct & 3.9 ± 0.03 & 2.4 ± 0.1 & 3.6 ± 0.1 \\
  \bottomrule
  \end{tabular}
\end{table}

\paragraph{Oracle accuracy.}
We report oracle accuracy for iCraft-MD in \cref{tab:oracle_acc}. For
AR-Bench-DC and AR-Bench-SP, full private information deterministically
fixes the answer, so oracle accuracy is trivially $100\%$ and omitted.

\begin{table}[htbp]
  \centering
  \caption{Oracle accuracy (\%) on iCraft-MD: upper bound when the agent has access to all patient information.}
  \label{tab:oracle_acc}
  \begin{tabular}{lc}
  \toprule
  \textbf{Model} & \textbf{Accuracy\ (\%)} \\
  \midrule
  Llama-3.1-8B-Instruct & 72.7 ± 3.8 \\
  Qwen2.5-7B-Instruct & 65.9 ± 4.2 \\
  Qwen2.5-14B-Instruct & 78.6 ± 3.5 \\
  Qwen2.5-32B-Instruct & 83.6 ± 3.1 \\
  QwQ-32B & 83.6 ± 3.1 \\
  \bottomrule
  \end{tabular}
\end{table}
\cref{tab:oracle_dims} shows that resolving an ambiguous prompt requires on average $3$--$5$
disambiguating dimensions, with larger models proposing fewer, more precise
dimensions. On iCraft-MD, oracle accuracy ranges from
$65.9\%$ to $83.6\%$ in \cref{tab:oracle_acc},
reflecting irreducible ambiguity in the underlying clinical questions.
BALAR with \texttt{Qwen2.5-32B-Instruct} achieves $73.6\%$ without access to
private information, reaching $88\%$ of the oracle upper bound.

\section{Dialogue State Tracking Evaluation}
\label{subsec:dst-evaluation}

We evaluate internal state tracking on iCraft-MD trajectories from
\texttt{Qwen2.5-32B-Instruct}, using the oracle dimensions and reference
values described above. These oracle annotations are LLM-generated evaluation
references.

For each instance, a \texttt{GPT-5-nano} judge
receives the public clinical case, the oracle dimensions and their optimal value annotations, all BALAR dimensions and value sets
(including those added by \textsc{Expand}), and the complete list of asked
questions with round numbers. For each oracle dimension, it returns an aligned
BALAR dimension, or $-1$ if none exists, the BALAR value label
equivalent to the oracle reference value, and the rounds of questions that
probed that dimension. Oracle information is used only for this
post-hoc evaluation and is not supplied to the BALAR agent. The exact judge
prompts are in Appendix~\ref{app:prompts:dst}.

For instance $i$, let $m_i$ be its number of scored oracle dimensions and
$c_i(k)$ the number whose aligned BALAR value in the joint MAP state after
$k$ cumulative \textsc{Ask} actions belongs to the judge's matching
set. We define
\begin{align}
  \mathrm{JGA}_i(k)&=\mathbf{1}[c_i(k)=m_i],\\
  \mathrm{SER}_i(k)&=1-\frac{c_i(k)}{m_i}.
\end{align}
Thus SER is an oracle-aligned slot-error fraction. It penalizes missing or
incorrect oracle slots but not additional BALAR dimensions. JGA requires every scored
oracle slot to be correct simultaneously.

Let $d_i(k)$ count oracle dimensions probed by at least one of the first $k$
asked questions, as identified by the judge. Questioning coverage is
\begin{equation}
  \mathrm{Coverage}_i(k)=\frac{d_i(k)}{m_i}.
\end{equation}
A dimension can be probed without having an equivalent BALAR dimension, so
this metric measures questioning breadth rather than representational
coverage or value correctness.

\begin{table}[htbp]
  \centering
  \caption{Oracle-aligned tracking and questioning coverage on iCraft-MD
  with \texttt{Qwen2.5-32B-Instruct}, for the $93$ evaluated instances
  reaching $25$ ASK actions. Entries are instance means $\pm$ standard errors.}
  \label{tab:dst-evaluation}
  \begin{tabular}{lc}
    \toprule
    Metric & Mean $\pm$ SE \\
    \midrule
    Joint goal accuracy & $0.280\pm0.047$ \\
    Oracle-aligned slot-error rate & $0.380\pm0.033$ \\
    Questioning coverage & $0.981\pm0.008$ \\
    \bottomrule
  \end{tabular}
\end{table}

\begin{figure}[htbp]
  \centering
  \includegraphics[width=\linewidth]{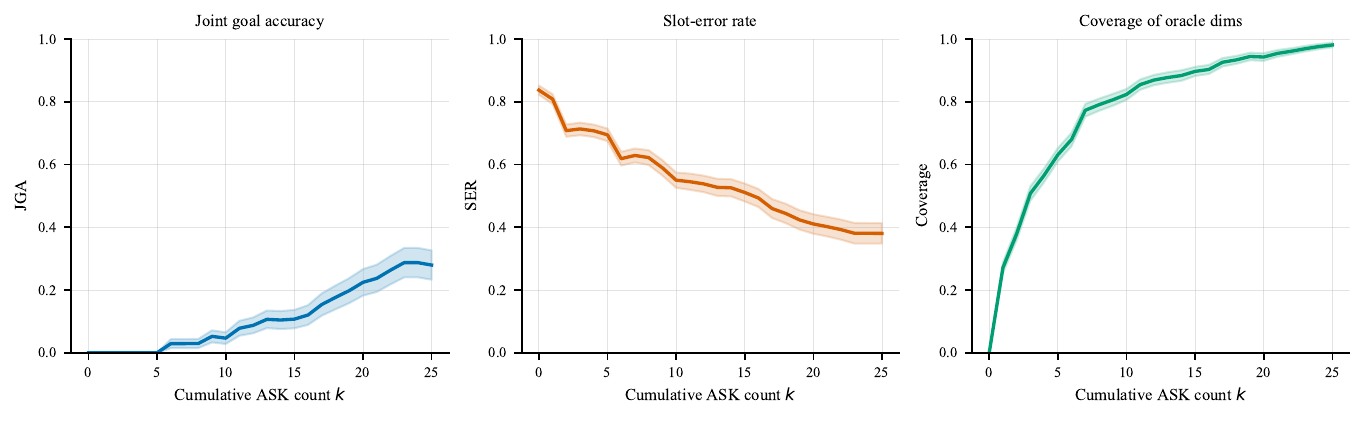}
  \caption{Oracle-aligned JGA, SER, and questioning coverage versus cumulative
  ASK count on iCraft-MD with \texttt{Qwen2.5-32B-Instruct}. Shaded bands
  indicate $\pm$ one standard error. The right panel's ``Coverage of
  oracle dims'' denotes dimensions probed by questions, not dimensions
  represented in BALAR's ontology.}
  \label{fig:dst-evaluation}
\end{figure}

\section{Distillation}
\label{app:distillation-objectives}

\begin{itemize}[leftmargin=*,nosep]
    \item \textbf{SFT :} given a dialogue prefix, the student learns to output BALAR's compact state and next question.
    \item \textbf{DPO :} given the same prefix and state, the student learns to prefer a question with high latent-state mutual information over one with low mutual information.
    \item \textbf{RL :} the student samples questions, a frozen Qwen model simulates each answer, and the student is rewarded when that answer reduces uncertainty over the clinical answer choices.
\end{itemize}

\subsection{Setup}

Recall that $\mathbf p$ is the ambiguous prompt, and $\mathbf c$ is the public meta-context. Before turn $t$, let
\begin{equation}
    \mathcal H_t=\{(q_1,a_1),\ldots,(q_{t-1},a_{t-1})\}
\end{equation}
be the interaction history and let $s_{t-1}$ be the previous compact state. Every student input is
\begin{equation}
    \mathcal I_t=(\mathbf p,\mathbf c,\mathcal H_t,s_{t-1}).
\end{equation}
The offline training records use BALAR's previous state in $\mathcal I_t$. During evaluation, the student uses its own previous state. The student never sees the private patient record or the correct answer. Its output is
\begin{equation}
    y_t=(s_t,q_t),
\end{equation}
where $s_t$ is a list of dimension--value--confidence triples and $q_t$ is one question.

The \textbf{student} is \texttt{Llama-3.1-8B-Instruct}. Only its LoRA adapter is updated. The \textbf{BALAR teacher} uses \texttt{Qwen2.5-32B-Instruct} to create the offline states, questions, and mutual-information scores. In RL and evaluation, a separate frozen instance of the same Qwen model acts as the patient simulator and answer-distribution estimator. None of the Qwen parameters are updated.

For completeness, recall that $\pi_t(\theta)$ is BALAR's joint posterior over
its active dimensions, and let
\begin{equation}
    \hat{\theta}_t=\arg\max_{\theta\in\Theta}\pi_t(\theta)
\end{equation}
be its joint MAP state. For dimension $j$, BALAR emits the value
$\hat{\theta}_{t,j}$ and converts its marginal probability
$\pi_t^{(j)}(\hat{\theta}_{t,j})$ to a label $\ell(\pi_t^{(j)}(\hat{\theta}_{t,j}))$. The compact state is then 
\begin{equation}
    s_t^*=\bigl\{(\theta_j,\hat{\theta}_{t,j},
    \ell(\pi_t^{(j)}(\hat{\theta}_{t,j})))\bigr\}_{j=1}^{p}.
\end{equation}
The value comes from the \emph{joint} MAP assignment, while the confidence label comes from that value's \emph{marginal} probability.

At test time, no BALAR loop generates the student's questions. The trained student predicts a state and question, the frozen Qwen patient simulator answers, and the next student turn receives that new dialogue plus the student's previous state. 

\subsection{Supervised fine-tuning (SFT)}

For every ASK turn in a training-patient BALAR trace, the data contain the input $\mathcal I_i$ and the teacher output
\begin{equation}
    y_i^*=(s_i^*,q_i^{\mathrm{BALAR}}).
\end{equation}
The target includes both the complete pre-question state and the question BALAR selected. If $y_{i,m}^*$ is target token $m$, SFT minimizes completion-only cross-entropy:
\begin{equation}
    \mathcal L_{\mathrm{SFT}}(\theta)
    =-\sum_i\sum_m
    \log \pi_\theta\!\left(y_{i,m}^*\mid \mathcal I_i,y_{i,<m}^*\right).
\end{equation}
This experiment trains the student to copy both BALAR's state update and BALAR's next question.

\subsection{Direct preference optimization (DPO)}

DPO starts from the fold's SFT model. At each retained BALAR prefix, BALAR's
question score is
\begin{equation}
    \mathrm{MI}_t(q)
    =\mathbb H(\pi_t)
    -\mathbb E_{Y\sim p_t^{q,u}}\!\left[
      \mathbb H(\pi_t^{q,u,Y})\right]
    =I_t(\theta;Y\mid\mathcal H_t;\,q,u).
\end{equation}
The chosen question $q_i^+$ is the available question with the largest recorded score. The rejected question $q_i^-$ is the distinct available question with the smallest score. A pair is kept only when
$\mathrm{MI}_i(q_i^+)-\mathrm{MI}_i(q_i^-)\geq0.01$.
Both completions contain the same teacher state:
\begin{equation}
    y_i^+=(s_i^*,q_i^+),\qquad y_i^-=(s_i^*,q_i^-).
\end{equation}
Consequently, their only intended difference is the question. Let $\pi_{\mathrm{ref}}$ be the fixed SFT model and let $\pi_\theta$ be the trainable DPO model initialized from it. Define the sequence log-probability difference
\begin{equation}
    \Delta_\pi(i)=\log\pi(y_i^+\mid \mathcal I_i)-\log\pi(y_i^-\mid \mathcal I_i).
\end{equation}
The DPO loss is
\begin{equation}
    \mathcal L_{\mathrm{DPO}}(\theta)
    =-\sum_i\log\sigma\!\left(
       \beta\left[\Delta_{\pi_\theta}(i)-\Delta_{\pi_{\mathrm{ref}}}(i)\right]
    \right),\qquad \beta=0.1.
\end{equation}
This trains the student to assign more probability to the high-MI question than to the low-MI question, relative to the SFT model.

\subsection{Dense-reward reinforcement learning (RL)}

RL starts independently from the SFT model. At an input $\mathcal I_i$, the current student samples a group of $G$ outputs
\begin{equation}
    y_{i,g}=(s_{i,g},q_{i,g}),\qquad g=1,\ldots,G.
\end{equation}
For each valid, non-repeated question, a frozen Qwen performs three steps:
\begin{enumerate}[leftmargin=*,nosep]
    \item Using only $(\mathbf p,\mathbf c,\mathcal H_i)$, estimate a distribution $p_i^-$ over all $K$ clinical answer choices.
    \item Using the private patient record, simulate an answer $a_{i,g}$ to $q_{i,g}$.
    \item Using only $(\mathbf p,\mathbf c,\mathcal H_i\cup\{(q_{i,g},a_{i,g})\})$, estimate a new distribution $p_{i,g}^+$ over the same choices.
\end{enumerate}
With $H(p)=-\sum_{k=1}^Kp_k\log p_k$, the reward is
\begin{equation}
    r_{i,g}=\frac{H(p_i^-)-H(p_{i,g}^+)}{\log K}.
\end{equation}
We clip this value to $[-1,1]$. A repeated question receives $-0.1$, a malformed JSON receives $-0.25$, and a question the patient cannot answer receives $0$.

Within each group, rewards are standardized:
\begin{equation}
    A_{i,g}=\frac{r_{i,g}-\bar r_i}
    {\sqrt{G^{-1}\sum_{g'=1}^G(r_{i,g'}-\bar r_i)^2}}.
\end{equation}
If all $G$ rewards are equal, every advantage is set to zero. For output token $m$, let
\begin{equation}
    \rho_{i,g,m}=\exp\!\left(
      \log\pi_\theta(y_{i,g,m}\mid\cdot)-
      \log\pi_{\mathrm{old}}(y_{i,g,m}\mid\cdot)
    \right),
\end{equation}
where $\pi_{\mathrm{old}}$ is the student that sampled the group. The token loss is
\begin{align}
    \mathcal L_{\mathrm{RL}}
    ={}&-\min\!\left(\rho A,
       \operatorname{clip}(\rho,1-\epsilon,1+\epsilon)A\right)
       +\lambda_{\mathrm{KL}}\,
       \widehat D_{\mathrm{k3}}(\pi_\theta\Vert\pi_{\mathrm{ref}}),\\
    \widehat D_{\mathrm{k3}}(\pi_\theta\Vert\pi_{\mathrm{ref}})
    ={}&\exp(\delta)-\delta-1,\qquad
    \delta=\log\pi_{\mathrm{ref}}-\log\pi_\theta,
\end{align}
averaged over completion tokens and sampled outputs, with $\epsilon=0.2$ and
$\lambda_{\mathrm{KL}}=0.02$. Here $\pi_{\mathrm{ref}}$ is the fixed SFT model,
and $\widehat D_{\mathrm{k3}}$ is the per-token estimator of
$D_{\mathrm{KL}}(\pi_\theta\Vert\pi_{\mathrm{ref}})$.

\newpage
\section{Prompts}
\label{app:prompts}

All BALAR prompts consist of a \emph{system prompt} and a \emph{user prompt}.
Each subsection below covers one prompt role and contains, for each of the
three domains (AR-Bench-DC, AR-Bench-SP, iCraft-MD), the system
prompt followed by the user prompt(s).
Every prompt requests a short \texttt{reason} field before the structured
output, implementing the lightweight chain-of-thought verification described
in \cref{subsec:practical}.

\subsection{User Simulator}
\label{app:prompts:user_sim}

\subsubsection*{AR-Bench-DC}

\paragraph{System prompt.}\leavevmode\par
\noindent
\begin{lstlisting}[style=prompt]
You are simulating a suspect being interrogated in a murder investigation. You must stay in character and answer based only on your character's knowledge and story. If you are the murderer, you will try to deflect suspicion while remaining consistent with your story. If you are innocent, answer honestly based on what your character knows.
\end{lstlisting}

\paragraph{With discrete choices.}\leavevmode\par
\begin{lstlisting}[style=prompt]
<USER_CONTEXT>
{user_context}
</USER_CONTEXT>

<QUESTION>
{question}
</QUESTION>

<CHOICES_WITH_IDS>
{choices_with_ids}
</CHOICES_WITH_IDS>

Task:
You are a suspect being interrogated by a detective. Answer <QUESTION> in a way that is consistent with your character described in <USER_CONTEXT>. If you are the murderer, you may deflect or be evasive, but you must remain consistent with your story.

What to generate:
- reason : a short one-sentence explanation of why you chose the answer you did.
- answer_id : the id of the choice you are selecting as your answer to the question.

Constraints:
Use ONLY information that is supported by <USER_CONTEXT>.
answer_id must be one of the ids provided in <CHOICES_WITH_IDS>.

Output format:
Return STRICT JSON only with the following schema:
{{
    "reason": string,
    "answer_id": string
}}
\end{lstlisting}

\paragraph{Without discrete choices (free-form).}\leavevmode\par
\begin{lstlisting}[style=prompt]
<USER_CONTEXT>
{user_context}
</USER_CONTEXT>

<QUESTION>
{question}
</QUESTION>

Task:
You are a suspect being interrogated by a detective. Answer <QUESTION> in a way that is consistent with your character described in <USER_CONTEXT>. If you are the murderer, you may deflect or be evasive, but you must remain consistent with your story.

What to generate:
- reason : a short one-sentence explanation of why you chose the answer you did.
- answer : your answer to the question.

Constraints:
Use ONLY information that is supported by <USER_CONTEXT>.
answer must be a natural language answer to the question.

Output format:
Return STRICT JSON only with the following schema:
{{
    "reason": string,
    "answer": string
}}
\end{lstlisting}

\subsubsection*{AR-Bench-SP}

\paragraph{System prompt.}\leavevmode\par
\begin{lstlisting}[style=prompt]
You are the host of a thinking puzzle (situation puzzle). You know the hidden explanation behind the puzzle. Answer the solver's questions truthfully based on the solution. Never reveal the full solution directly; only confirm or deny specific aspects when asked.
\end{lstlisting}

\paragraph{With discrete choices.}\leavevmode\par
\begin{lstlisting}[style=prompt]
<USER_CONTEXT>
{user_context}
</USER_CONTEXT>

<QUESTION>
{question}
</QUESTION>

<CHOICES_WITH_IDS>
{choices_with_ids}
</CHOICES_WITH_IDS>

Task:
You are the host of a thinking puzzle. A solver is asking you <QUESTION> to try to figure out the hidden explanation. Answer based on your knowledge of the puzzle solution described in <USER_CONTEXT>. Give a truthful answer — do not mislead, but do not give extra information.

What to generate:
- reason : a short one-sentence explanation of why you chose the answer you did.
- answer_id : the id of the choice you are selecting as your answer to the question.

Constraints:
Use ONLY information that is supported by <USER_CONTEXT>.
answer_id must be one of the ids provided in <CHOICES_WITH_IDS>.

Output format:
Return STRICT JSON only with the following schema:
{{
    "reason": string,
    "answer_id": string
}}
\end{lstlisting}

\paragraph{Without discrete choices (free-form).}\leavevmode\par
\begin{lstlisting}[style=prompt]
<USER_CONTEXT>
{user_context}
</USER_CONTEXT>

<QUESTION>
{question}
</QUESTION>

Task:
You are the host of a thinking puzzle. A solver is asking you <QUESTION> to try to figure out the hidden explanation. Answer based on your knowledge of the puzzle solution described in <USER_CONTEXT>. Give a truthful answer — do not mislead, but do not volunteer extra information.

What to generate:
- reason : a short one-sentence explanation of why you chose the answer you did.
- answer : your answer to the question.

Constraints:
Use ONLY information that is supported by <USER_CONTEXT>.
answer must be a natural language answer to the question.

Output format:
Return STRICT JSON only with the following schema:
{{
    "reason": string,
    "answer": string
}}
\end{lstlisting}

\subsubsection*{iCraft-MD}

\paragraph{System prompt.}\leavevmode\par
\begin{lstlisting}[style=prompt]
You are simulating a patient being examined by a physician. Answer questions based ONLY on your medical record and symptoms. Use facts from your medical record. Do not give information that was not specifically asked about.
\end{lstlisting}

\paragraph{With discrete choices.}\leavevmode\par
\begin{lstlisting}[style=prompt]
<USER_CONTEXT>
{user_context}
</USER_CONTEXT>

<QUESTION>
{question}
</QUESTION>

<CHOICES_WITH_IDS>
{choices_with_ids}
</CHOICES_WITH_IDS>

Task:
You are a patient being examined by a physician. Answer <QUESTION> in a way that is consistent with your medical record and symptoms described in <USER_CONTEXT>. Use facts from your medical record. Do not give information that was not specifically asked about.

What to generate:
- reason : a short one-sentence explanation of why you chose the answer you did.
- answer_id : the id of the choice you are selecting as your answer to the question.

Constraints:
Use ONLY information that is supported by <USER_CONTEXT>.
answer_id must be one of the ids provided in <CHOICES_WITH_IDS>.

Output format:
Return STRICT JSON only with the following schema:
{{
    "reason": string,
    "answer_id": string
}}
\end{lstlisting}

\paragraph{Without discrete choices (free-form).}\leavevmode\par
\begin{lstlisting}[style=prompt]
<USER_CONTEXT>
{user_context}
</USER_CONTEXT>

<QUESTION>
{question}
</QUESTION>

Task:
You are a patient being examined by a physician. Answer <QUESTION> in a way that is consistent with your medical record and symptoms described in <USER_CONTEXT>. Use facts from your medical record. Do not give information that was not specifically asked about.

What to generate:
- reason : a short one-sentence explanation of why you chose the answer you did.
- answer : your answer to the question.

Constraints:
Use ONLY information that is supported by <USER_CONTEXT>.
answer must be a natural language answer to the question.

Output format:
Return STRICT JSON only with the following schema:
{{
    "reason": string,
    "answer": string
}}
\end{lstlisting}

\subsection{Step 1 --- Dimension Proposal}
\label{app:prompts:init_dims}

\subsubsection*{AR-Bench-DC}

\paragraph{System prompt.}\leavevmode\par
\begin{lstlisting}[style=prompt]
You are an experienced detective analyzing a murder case. Your goal is to identify the key investigative dimensions that must be resolved to determine who the real murderer is.
\end{lstlisting}

\paragraph{User prompt.}\leavevmode\par
\begin{lstlisting}[style=prompt]
<CASE_QUESTION>
{ambiguous_prompt}
</CASE_QUESTION>

<CASE_BACKGROUND>
{meta_context}
</CASE_BACKGROUND>

Task:
Identify the key investigative dimensions in the <CASE_QUESTION> that must be resolved to determine who the real murderer is.

Definition:
An investigative dimension is a specific aspect of the murder case (e.g., motive, alibi, access to the murder weapon, relationship to the victim) where multiple suspects could plausibly be implicated, and resolving it would narrow down the true murderer. Once the dimension's value is known, the case moves toward identifying a single suspect as the murderer.

What to generate:
- Produce a minimal, non-overlapping set of investigative dimensions.
- Produce exactly {num_initial_dims} investigative dimensions.
- Each dimension must correspond to a distinct, investigative uncertainty.
- If <CASE_BACKGROUND> already resolves a dimension, do not include it.
- If <CASE_BACKGROUND> proposes some investigative dimensions, use them.

For each investigative dimension, provide :
- reason: a short one-sentence explanation of why this dimension is critical for identifying the murderer.
- name: a short, specific label (e.g., "Motive", "Alibi at time of death", "Access to murder weapon", etc.)
- values: a list of plausible values (e.g., one per suspect or per scenario), no larger than {max_num_values_per_dim}, that this dimension could take in the context of the case.

Constraints:
- Use ONLY the information provided in <CASE_QUESTION> and <CASE_BACKGROUND>.
- Do NOT answer the <CASE_QUESTION> itself. Focus ONLY on identifying investigative dimensions.
- Do NOT rewrite or restate the <CASE_QUESTION> or <CASE_BACKGROUND>.

Output format:
Return STRICT JSON only with the following schema:
{{
    "dimensions": [
        {{
            "reason": string,
            "name": string,
            "values": [string, ...]
        }},
        ...
    ]
}}
\end{lstlisting}

\subsubsection*{AR-Bench-SP}

\paragraph{System prompt.}\leavevmode\par
\begin{lstlisting}[style=prompt]
You are an expert thinking puzzle solver. Your goal is to identify the hidden dimensions of the puzzle — the unstated aspects of the scenario whose true values would explain the strange situation presented.
\end{lstlisting}

\paragraph{User prompt.}\leavevmode\par
\begin{lstlisting}[style=prompt]
<PUZZLE>
{ambiguous_prompt}
</PUZZLE>

<PUZZLE_CONTEXT>
{meta_context}
</PUZZLE_CONTEXT>

Task:
Identify the hidden dimensions of the thinking puzzle in <PUZZLE> that must be uncovered to explain the strange scenario.

Definition:
A puzzle dimension is a hidden aspect of the scenario — such as a non-obvious word meaning, an unstated context, a surprising identity, or an unusual causal mechanism — where knowing its true value would explain the puzzle. Once the dimension's value is known, the puzzle moves toward a single coherent explanation.

What to generate:
- Produce a minimal, non-overlapping set of puzzle dimensions.
- Produce exactly {num_initial_dims} puzzle dimensions.
- Each dimension must correspond to a distinct hidden aspect of the puzzle.
- If <PUZZLE_CONTEXT> already resolves a dimension, do not include it.
- If <PUZZLE_CONTEXT> proposes some puzzle dimensions, use them.

For each puzzle dimension, provide :
- reason: a short one-sentence explanation of why this dimension is a key unknown in the puzzle.
- name: a short, specific label
- values: a list of plausible interpretations, no larger than {max_num_values_per_dim}, that this dimension could take.

Constraints:
- Use ONLY the information provided in <PUZZLE> and <PUZZLE_CONTEXT>.
- Do NOT solve the <PUZZLE> itself. Focus ONLY on identifying hidden dimensions.
- Do NOT rewrite or restate the <PUZZLE> or <PUZZLE_CONTEXT>.

Output format:
Return STRICT JSON only with the following schema:
{{
    "dimensions": [
        {{
            "reason": string,
            "name": string,
            "values": [string, ...]
        }},
        ...
    ]
}}
\end{lstlisting}

\subsubsection*{iCraft-MD}

\paragraph{System prompt.}\leavevmode\par
\begin{lstlisting}[style=prompt]
You are an experienced physician performing a diagnosis. Your goal is to identify the key clinical dimensions that must be clarified to arrive at the correct diagnosis or clinical decision.
\end{lstlisting}

\paragraph{User prompt.}\leavevmode\par
\begin{lstlisting}[style=prompt]
<CLINICAL_QUESTION>
{ambiguous_prompt}
</CLINICAL_QUESTION>

<PATIENT_INFORMATION>
{meta_context}
</PATIENT_INFORMATION>

Task:
Identify the key clinical dimensions in the <CLINICAL_QUESTION> that must be clarified to arrive at the correct diagnosis or clinical decision.

Definition:
A clinical dimension is a specific clinical factor (e.g., symptom characterization, lab finding, risk factor, past medical history) where different values would point toward different diagnoses or clinical decisions. Once the dimension's value is known, the diagnosis narrows toward a single correct answer.

What to generate:
- Produce a minimal, non-overlapping set of clinical dimensions.
- Produce exactly {num_initial_dims} clinical dimensions.
- Each dimension must correspond to a distinct clinical uncertainty.
- If <PATIENT_INFORMATION> already resolves a dimension, do not include it.
- If <PATIENT_INFORMATION> proposes some clinical dimensions, use them.

For each clinical dimension, provide :
- reason: a short one-sentence explanation of why this clinical factor is discriminating between diagnoses.
- name: a short, specific clinical label
- values: a list of clinically plausible values, no larger than {max_num_values_per_dim}, that this dimension could take.

Constraints:
- Use ONLY the information provided in <CLINICAL_QUESTION> and <PATIENT_INFORMATION>.
- Do NOT answer the <CLINICAL_QUESTION> itself. Focus ONLY on identifying clinical dimensions.
- Do NOT rewrite or restate the <CLINICAL_QUESTION> or <PATIENT_INFORMATION>.

Output format:
Return STRICT JSON only with the following schema:
{{
    "dimensions": [
        {{
            "reason": string,
            "name": string,
            "values": [string, ...]
        }},
        ...
    ]
}}
\end{lstlisting}

\subsection{Step 2 --- Prior Elicitation}
\label{app:prompts:init_priors}

\subsubsection*{AR-Bench-DC}

\paragraph{System prompt.}\leavevmode\par
\begin{lstlisting}[style=prompt]
You are an experienced detective forming initial hypotheses about a murder case based on the available evidence and case background.
\end{lstlisting}

\paragraph{User prompt.}\leavevmode\par
\begin{lstlisting}[style=prompt]
<CASE_QUESTION>
{ambiguous_prompt}
</CASE_QUESTION>

<CASE_BACKGROUND>
{meta_context}
</CASE_BACKGROUND>

<DIMENSION_NAME>
{dimension_name}
</DIMENSION_NAME>

<DIMENSION_VALUE>
{dimension_value}
</DIMENSION_VALUE>

Task:
Given <CASE_QUESTION> and <CASE_BACKGROUND>, judge how likely the investigative dimension <DIMENSION_NAME> takes on the value <DIMENSION_VALUE>.

What to generate:
- reason: a short one-sentence explanation of why the <DIMENSION_NAME> is likely, unlikely, or neutral to take on the value <DIMENSION_VALUE>.
- label: one of "likely", "unlikely", or "neutral" according to the following definitions:
    - likely: <DIMENSION_VALUE> is explicitly stated, strongly implied, or is the most natural assumption given the evidence in the <CASE_QUESTION> and <CASE_BACKGROUND>.
    - neutral: <DIMENSION_VALUE> is plausible but not implied or supported by specific evidence in the <CASE_QUESTION> or <CASE_BACKGROUND>.
    - unlikely: <DIMENSION_VALUE> is contradicted by the <CASE_QUESTION> or <CASE_BACKGROUND>, or would require assumptions that are inconsistent with the available evidence.

Constraints:
- Use ONLY the information provided in <CASE_QUESTION> and <CASE_BACKGROUND>.
- Do NOT answer the <CASE_QUESTION> itself. Focus ONLY on judging the likelihood of the dimension value.
- Do NOT rewrite or restate the <CASE_QUESTION> or <CASE_BACKGROUND>.
- label must be one of "likely", "unlikely", or "neutral".

Output format:
Return STRICT JSON only with the following schema:
{{
    "reason": string,
    "label": string
}}
\end{lstlisting}

\subsubsection*{AR-Bench-SP}

\paragraph{System prompt.}\leavevmode\par
\begin{lstlisting}[style=prompt]
You are an expert thinking puzzle solver forming initial hypotheses about the hidden aspects of a puzzle based on the scenario description.
\end{lstlisting}

\paragraph{User prompt.}\leavevmode\par
\begin{lstlisting}[style=prompt]
<PUZZLE>
{ambiguous_prompt}
</PUZZLE>

<PUZZLE_CONTEXT>
{meta_context}
</PUZZLE_CONTEXT>

<DIMENSION_NAME>
{dimension_name}
</DIMENSION_NAME>

<DIMENSION_VALUE>
{dimension_value}
</DIMENSION_VALUE>

Task:
Given <PUZZLE> and <PUZZLE_CONTEXT>, judge how likely the puzzle dimension <DIMENSION_NAME> takes on the value <DIMENSION_VALUE>.

What to generate:
- reason: a short one-sentence explanation of why the <DIMENSION_NAME> is likely, unlikely, or neutral to take on the value <DIMENSION_VALUE>.
- label: one of "likely", "unlikely", or "neutral" according to the following definitions:
    - likely: <DIMENSION_VALUE> is explicitly suggested by, strongly implied by, or is the most natural interpretation given the clues in the <PUZZLE> and <PUZZLE_CONTEXT>.
    - neutral: <DIMENSION_VALUE> is a plausible interpretation but not implied or supported by specific clues in the <PUZZLE> or <PUZZLE_CONTEXT>.
    - unlikely: <DIMENSION_VALUE> is contradicted by the <PUZZLE> or <PUZZLE_CONTEXT>, or would require assumptions that are inconsistent with the scenario.

Constraints:
- Use ONLY the information provided in <PUZZLE> and <PUZZLE_CONTEXT>.
- Do NOT solve the <PUZZLE> itself. Focus ONLY on judging the likelihood of the dimension value.
- Do NOT rewrite or restate the <PUZZLE> or <PUZZLE_CONTEXT>.
- label must be one of "likely", "unlikely", or "neutral".

Output format:
Return STRICT JSON only with the following schema:
{{
    "reason": string,
    "label": string
}}
\end{lstlisting}

\subsubsection*{iCraft-MD}

\paragraph{System prompt.}\leavevmode\par
\begin{lstlisting}[style=prompt]
You are an experienced physician forming initial clinical hypotheses based on the available patient information. Use your clinical knowledge to assess the likelihood of different clinical findings.
\end{lstlisting}

\paragraph{User prompt.}\leavevmode\par
\begin{lstlisting}[style=prompt]
<CLINICAL_QUESTION>
{ambiguous_prompt}
</CLINICAL_QUESTION>

<PATIENT_INFORMATION>
{meta_context}
</PATIENT_INFORMATION>

<DIMENSION_NAME>
{dimension_name}
</DIMENSION_NAME>

<DIMENSION_VALUE>
{dimension_value}
</DIMENSION_VALUE>

Task:
Given <CLINICAL_QUESTION> and <PATIENT_INFORMATION>, judge how likely the clinical dimension <DIMENSION_NAME> takes on the value <DIMENSION_VALUE>.

What to generate:
- reason: a short one-sentence explanation of why the <DIMENSION_NAME> is likely, unlikely, or neutral to take on the value <DIMENSION_VALUE>.
- label: one of "likely", "unlikely", or "neutral" according to the following definitions:
    - likely: <DIMENSION_VALUE> is explicitly stated, strongly implied, or is the most natural clinical assumption given the patient information in <CLINICAL_QUESTION> and <PATIENT_INFORMATION>.
    - neutral: <DIMENSION_VALUE> is clinically plausible but not implied or supported by specific evidence in the <CLINICAL_QUESTION> or <PATIENT_INFORMATION>.
    - unlikely: <DIMENSION_VALUE> is contradicted by the <CLINICAL_QUESTION> or <PATIENT_INFORMATION>, or would require assumptions that are inconsistent with the patient's presentation.

Constraints:
- Use ONLY the information provided in <CLINICAL_QUESTION> and <PATIENT_INFORMATION>.
- Do NOT answer the <CLINICAL_QUESTION> itself. Focus ONLY on judging the likelihood of the dimension value.
- Do NOT rewrite or restate the <CLINICAL_QUESTION> or <PATIENT_INFORMATION>.
- label must be one of "likely", "unlikely", or "neutral".

Output format:
Return STRICT JSON only with the following schema:
{{
    "reason": string,
    "label": string
}}
\end{lstlisting}

\subsection{Step 3 --- Question Generation}
\label{app:prompts:init_questions}

\subsubsection*{AR-Bench-DC}

\paragraph{System prompt.}\leavevmode\par
\begin{lstlisting}[style=prompt]
You are an experienced detective preparing interrogation questions for suspects in a murder investigation. Your questions should be designed to reveal inconsistencies, uncover motives, and verify alibis.
\end{lstlisting}

\paragraph{User prompt.}\leavevmode\par
\begin{lstlisting}[style=prompt]
<CASE_QUESTION>
{ambiguous_prompt}
</CASE_QUESTION>

<CASE_BACKGROUND>
{meta_context}
</CASE_BACKGROUND>

<INVESTIGATIVE_DIMENSIONS>
{dimensions_with_values}
</INVESTIGATIVE_DIMENSIONS>

Task:
Given <CASE_QUESTION>, <CASE_BACKGROUND>, and <INVESTIGATIVE_DIMENSIONS>, generate exactly {num_initial_questions} interrogation questions to ask the suspects that would help identify the real murderer. Each question should target one or more <INVESTIGATIVE_DIMENSIONS> and have multiple-choice answers.

Definition:
<INVESTIGATIVE_DIMENSIONS> is a list of investigative dimensions, where each dimension has a name and a list of possible values it could take. An investigative dimension is a specific aspect of the murder case where multiple suspects could plausibly be implicated, and resolving it would narrow down the true murderer.

What to generate:
For each of the {num_initial_questions} questions, provide:
- reason: a short one-sentence explanation of why this question would help identify the murderer.
- question: the text of the interrogation question.
- choices: a list of multiple-choice answer options for the question, no larger than {max_choices_per_question}.

Constraints:
- Use ONLY the information provided in <CASE_QUESTION>, <CASE_BACKGROUND>, and <INVESTIGATIVE_DIMENSIONS>.
- Do NOT answer the <CASE_QUESTION> itself. Focus ONLY on generating interrogation questions.
- Do NOT rewrite or restate the <CASE_QUESTION> or <CASE_BACKGROUND>.
- Each question must be designed to elicit information about one or more of the dimensions in <INVESTIGATIVE_DIMENSIONS>.
- Each question must have multiple-choice answers.

Output format:
Return STRICT JSON only with the following schema:
{{
    "questions": [
        {{
            "reason": string,
            "question": string,
            "choices": [string, ...]
        }},
        ...
    ]
}}
\end{lstlisting}

\subsubsection*{AR-Bench-SP}

\paragraph{System prompt.}\leavevmode\par
\begin{lstlisting}[style=prompt]
You are an expert thinking puzzle solver. Generate clarifying questions to ask the puzzle host that will help you uncover the hidden explanation. Good puzzle questions test specific hypotheses about what is really going on in the scenario.
\end{lstlisting}

\paragraph{User prompt.}\leavevmode\par
\begin{lstlisting}[style=prompt]
<PUZZLE>
{ambiguous_prompt}
</PUZZLE>

<PUZZLE_CONTEXT>
{meta_context}
</PUZZLE_CONTEXT>

<PUZZLE_DIMENSIONS>
{dimensions_with_values}
</PUZZLE_DIMENSIONS>

Task:
Given <PUZZLE>, <PUZZLE_CONTEXT>, and <PUZZLE_DIMENSIONS>, generate exactly {num_initial_questions} clarifying questions to ask the puzzle host that would help uncover the hidden explanation. Each question should target one or more <PUZZLE_DIMENSIONS> and have "yes"/"no" answers.

Definition:
<PUZZLE_DIMENSIONS> is a list of puzzle dimensions, where each dimension has a name and a list of possible values it could take. A puzzle dimension is a hidden aspect of the scenario where knowing its true value would explain the puzzle.

What to generate:
For each of the {num_initial_questions} questions, provide:
- reason: a short one-sentence explanation of why this question would help solve the puzzle.
- question: the text of the clarifying question to ask the puzzle host.
- choices: ["yes", "no"] as the {max_choices_per_question} multiple-choice answer options for the question.

Constraints:
- Use ONLY the information provided in <PUZZLE>, <PUZZLE_CONTEXT>, and <PUZZLE_DIMENSIONS>.
- Do NOT solve the <PUZZLE> itself. Focus ONLY on generating clarifying questions.
- Do NOT rewrite or restate the <PUZZLE> or <PUZZLE_CONTEXT>.
- Each question must be designed to elicit information about one or more of the dimensions in <PUZZLE_DIMENSIONS>.
- Each question must have "yes"/"no" answers.
- Keep each question short: at most 20 words.
- Keep each reason short: at most 15 words.

Output format:
Return STRICT JSON only with the following schema:
{{
    "questions": [
        {{
            "reason": string,
            "question": string,
            "choices": ["yes", "no"]
        }},
        ...
    ]
}}
\end{lstlisting}

\subsubsection*{iCraft-MD}

\paragraph{System prompt.}\leavevmode\par
\begin{lstlisting}[style=prompt]
You are an experienced physician conducting a patient interview to gather clinical information for a diagnosis. Your questions should be targeted, clinically relevant, and designed to discriminate between competing diagnoses.
\end{lstlisting}

\paragraph{User prompt.}\leavevmode\par
\begin{lstlisting}[style=prompt]
<CLINICAL_QUESTION>
{ambiguous_prompt}
</CLINICAL_QUESTION>

<PATIENT_INFORMATION>
{meta_context}
</PATIENT_INFORMATION>

<CLINICAL_DIMENSIONS>
{dimensions_with_values}
</CLINICAL_DIMENSIONS>

Task:
Given <CLINICAL_QUESTION>, <PATIENT_INFORMATION>, and <CLINICAL_DIMENSIONS>, generate exactly {num_initial_questions} clinical questions to ask the patient that would help arrive at the correct diagnosis. Each question should target one or more <CLINICAL_DIMENSIONS> and have multiple-choice answers.

Definition:
<CLINICAL_DIMENSIONS> is a list of clinical dimensions, where each dimension has a name and a list of possible values it could take. A clinical dimension is a specific clinical factor where different values would point toward different diagnoses.

What to generate:
For each of the {num_initial_questions} questions, provide:
- reason: a short one-sentence explanation of why this question would help narrow the diagnosis.
- question: the text of the clinical question to ask the patient.
- choices: a list of multiple-choice answer options for the question, no larger than {max_choices_per_question}.

Constraints:
- Use ONLY the information provided in <CLINICAL_QUESTION>, <PATIENT_INFORMATION>, and <CLINICAL_DIMENSIONS>.
- Do NOT answer the <CLINICAL_QUESTION> itself. Focus ONLY on generating clinical questions.
- Do NOT rewrite or restate the <CLINICAL_QUESTION> or <PATIENT_INFORMATION>.
- Each question must be designed to elicit information about one or more of the dimensions in <CLINICAL_DIMENSIONS>.
- Each question must have multiple-choice answers.

Output format:
Return STRICT JSON only with the following schema:
{{
    "questions": [
        {{
            "reason": string,
            "question": string,
            "choices": [string, ...]
        }},
        ...
    ]
}}
\end{lstlisting}

\subsection{Step 4 --- Likelihood Table Construction}
\label{app:prompts:likelihood}

\subsubsection*{AR-Bench-DC}

\paragraph{System prompt.}\leavevmode\par
\begin{lstlisting}[style=prompt]
You are an experienced detective evaluating how a suspect would likely respond to an interrogation question under different assumptions about the case. Consider that guilty suspects may deflect, lie, or give evasive answers, while innocent suspects will answer based on their genuine knowledge.
\end{lstlisting}

\paragraph{Without conversation history.}\leavevmode\par
\begin{lstlisting}[style=prompt]
<CASE_QUESTION>
{ambiguous_prompt}
</CASE_QUESTION>

<CASE_BACKGROUND>
{meta_context}
</CASE_BACKGROUND>

<SUSPECT_INFO>
{user_info}
</SUSPECT_INFO>

<DIMENSION_NAME>
{dimension_name}
</DIMENSION_NAME>

<DIMENSION_VALUES_WITH_IDS>
{dimension_values_with_ids}
</DIMENSION_VALUES_WITH_IDS>

<QUESTION>
{question_text}
</QUESTION>

<QUESTION_CHOICES_WITH_IDS>
{question_choices_with_ids}
</QUESTION_CHOICES_WITH_IDS>

Definition:
An investigative dimension is a specific aspect of the murder case where multiple suspects could plausibly be implicated, and resolving it would narrow down the true murderer.
<DIMENSION_VALUES_WITH_IDS> is a list of dicts with "id" and "text" fields. Each dict corresponds to a possible value that the <DIMENSION_NAME> could take.
<QUESTION_CHOICES_WITH_IDS> is a list of dicts with "id" and "text" fields. Each dict corresponds to a multiple-choice answer option for the question.
Let values[i] be the i-th element of <DIMENSION_VALUES_WITH_IDS>.
Let choices[j] be the j-th element of <QUESTION_CHOICES_WITH_IDS>.

Task (row-major order):
For i = 0..len(values)-1:
  For j = 0..len(choices)-1:
    - Assume the true state of the case is <DIMENSION_NAME> = values[i]["text"].
    - Impersonate the suspect described in <SUSPECT_INFO>.
    - Judge how likely it is that this suspect would answer the question <QUESTION> with choices[j]["text"] under that assumption. Consider that a guilty suspect may try to deflect or mislead.

What to generate:
For i = 0..len(values)-1:
  For j = 0..len(choices)-1:
    - question_choice_id: the id of the question choice being evaluated, i.e. choices[j]["id"]
    - dimension_value_id: the id of the dimension value being evaluated, i.e. values[i]["id"]
    - reason: a short one-sentence explanation of why choices[j] is labeled likely/neutral/unlikely and why the other two labels were not chosen.
    - label: one of "likely", "neutral", or "unlikely" according to the following definitions:
      - "likely": Given <DIMENSION_NAME> = values[i]["text"] and the suspect acts according to <SUSPECT_INFO>, the suspect is expected to give choices[j]["text"] for <QUESTION>.
      - "neutral": Given <DIMENSION_NAME> = values[i]["text"] and the suspect acts according to <SUSPECT_INFO>, choices[j]["text"] is plausible but not specifically supported; there is insufficient evidence to say that the suspect would or would not prefer it over other choices.
      - "unlikely": Given <DIMENSION_NAME> = values[i]["text"] and the suspect acts according to <SUSPECT_INFO>, the suspect is not expected to give choices[j]["text"] for <QUESTION>.

Constraints:
- Use ONLY the information provided in <CASE_QUESTION>, <CASE_BACKGROUND>, and <SUSPECT_INFO>.
- Do NOT answer the <CASE_QUESTION> itself. Focus ONLY on judging the likelihood of the question choices under different assumptions about the dimension value.
- Do NOT rewrite or restate the <CASE_QUESTION>, <CASE_BACKGROUND>, or <SUSPECT_INFO>.
- label must be one of "likely", "neutral", or "unlikely".
- The output must include an entry for every combination of dimension value and question choice.

Output format:
Return STRICT JSON only with the following schema:
{{
    "evaluations": [
        [
            {{
                "question_choice_id": string,
                "dimension_value_id": string,
                "reason": string,
                "label": string
            }},
            ... // one object for each question choice
        ],
        ... // one array for each dimension value
    ]
}}
The "evaluations" field must contain exactly {num_dimension_values} arrays (one per dimension value).
Each inner array must contain exactly {num_question_choices} objects (one per question choice).
\end{lstlisting}

\paragraph{With conversation history.}\leavevmode\par
\begin{lstlisting}[style=prompt]
<CASE_QUESTION>
{ambiguous_prompt}
</CASE_QUESTION>

<CASE_BACKGROUND>
{meta_context}
</CASE_BACKGROUND>

<SUSPECT_INFO>
{user_info}
</SUSPECT_INFO>

<INTERROGATION_LOG>
{conversation_log}
</INTERROGATION_LOG>

<DIMENSION_NAME>
{dimension_name}
</DIMENSION_NAME>

<DIMENSION_VALUES_WITH_IDS>
{dimension_values_with_ids}
</DIMENSION_VALUES_WITH_IDS>

<QUESTION>
{question_text}
</QUESTION>

<QUESTION_CHOICES_WITH_IDS>
{question_choices_with_ids}
</QUESTION_CHOICES_WITH_IDS>

Definition:
An investigative dimension is a specific aspect of the murder case where multiple suspects could plausibly be implicated, and resolving it would narrow down the true murderer.
<DIMENSION_VALUES_WITH_IDS> is a list of dicts with "id" and "text" fields. Each dict corresponds to a possible value that the <DIMENSION_NAME> could take.
<QUESTION_CHOICES_WITH_IDS> is a list of dicts with "id" and "text" fields. Each dict corresponds to a multiple-choice answer option for the question.
<INTERROGATION_LOG> is a list of dicts with "question_text", "user_name", and "user_answer" fields, representing the history of the interrogation between the detective and the suspects up to this point. This information may provide additional context and reveal inconsistencies.
Let values[i] be the i-th element of <DIMENSION_VALUES_WITH_IDS>.
Let choices[j] be the j-th element of <QUESTION_CHOICES_WITH_IDS>.

Task (row-major order):
For i = 0..len(values)-1:
  For j = 0..len(choices)-1:
    - Assume the true state of the case is <DIMENSION_NAME> = values[i]["text"].
    - Impersonate the suspect described in <SUSPECT_INFO>.
    - Judge how likely it is that this suspect would answer the question <QUESTION> with choices[j]["text"] under that assumption. Consider that a guilty suspect may try to deflect or mislead.

What to generate:
For i = 0..len(values)-1:
  For j = 0..len(choices)-1:
    - question_choice_id: the id of the question choice being evaluated, i.e. choices[j]["id"]
    - dimension_value_id: the id of the dimension value being evaluated, i.e. values[i]["id"]
    - reason: a short one-sentence explanation of why choices[j] is labeled likely/neutral/unlikely and why the other two labels were not chosen.
    - label: one of "likely", "neutral", or "unlikely" according to the following definitions:
      - "likely": Given <DIMENSION_NAME> = values[i]["text"] and the suspect acts according to <SUSPECT_INFO>, the suspect is expected to give choices[j]["text"] for <QUESTION>.
      - "neutral": Given <DIMENSION_NAME> = values[i]["text"] and the suspect acts according to <SUSPECT_INFO>, choices[j]["text"] is plausible but not specifically supported; there is insufficient evidence to say that the suspect would or would not prefer it over other choices.
      - "unlikely": Given <DIMENSION_NAME> = values[i]["text"] and the suspect acts according to <SUSPECT_INFO>, the suspect is not expected to give choices[j]["text"] for <QUESTION>.

Constraints:
- Use ONLY the information provided in <CASE_QUESTION>, <CASE_BACKGROUND>, <SUSPECT_INFO>, and <INTERROGATION_LOG>.
- Do NOT answer the <CASE_QUESTION> itself. Focus ONLY on judging the likelihood of the question choices under different assumptions about the dimension value.
- Do NOT rewrite or restate the <CASE_QUESTION>, <CASE_BACKGROUND>, <SUSPECT_INFO>, or <INTERROGATION_LOG>.
- label must be one of "likely", "neutral", or "unlikely".
- The output must include an entry for every combination of dimension value and question choice.

Output format:
Return STRICT JSON only with the following schema:
{{
    "evaluations": [
        [
            {{
                "question_choice_id": string,
                "dimension_value_id": string,
                "reason": string,
                "label": string
            }},
            ... // one object for each question choice
        ],
        ... // one array for each dimension value
    ]
}}
The "evaluations" field must contain exactly {num_dimension_values} arrays (one per dimension value).
Each inner array must contain exactly {num_question_choices} objects (one per question choice).
\end{lstlisting}

\subsubsection*{AR-Bench-SP}

\paragraph{System prompt.}\leavevmode\par
\begin{lstlisting}[style=prompt]
You are an expert thinking puzzle analyst. Evaluate how a puzzle host who knows the hidden explanation would likely respond to a solver's question under different assumptions about the puzzle's hidden aspects. The host answers truthfully, without giving extra information.
\end{lstlisting}

\paragraph{Without conversation history.}\leavevmode\par
\begin{lstlisting}[style=prompt]
<PUZZLE>
{ambiguous_prompt}
</PUZZLE>

<PUZZLE_CONTEXT>
{meta_context}
</PUZZLE_CONTEXT>

<HOST_INFO>
{user_info}
</HOST_INFO>

<DIMENSION_NAME>
{dimension_name}
</DIMENSION_NAME>

<DIMENSION_VALUES_WITH_IDS>
{dimension_values_with_ids}
</DIMENSION_VALUES_WITH_IDS>

<QUESTION>
{question_text}
</QUESTION>

<QUESTION_CHOICES_WITH_IDS>
{question_choices_with_ids}
</QUESTION_CHOICES_WITH_IDS>

Definition:
A puzzle dimension is a hidden aspect of the scenario where knowing its true value would explain the puzzle.
<DIMENSION_VALUES_WITH_IDS> is a list of dicts with "id" and "text" fields. Each dict corresponds to a possible value that the <DIMENSION_NAME> could take.
<QUESTION_CHOICES_WITH_IDS> is a list of dicts with "id" and "text" fields. Each dict corresponds to a multiple-choice answer option for the question.
Let values[i] be the i-th element of <DIMENSION_VALUES_WITH_IDS>.
Let choices[j] be the j-th element of <QUESTION_CHOICES_WITH_IDS>.

Task (row-major order):
For i = 0..len(values)-1:
  For j = 0..len(choices)-1:
    - Assume the hidden explanation is such that <DIMENSION_NAME> = values[i]["text"].
    - Impersonate the puzzle host described in <HOST_INFO>, who knows the hidden explanation.
    - Judge how likely it is that this host would answer the question <QUESTION> with choices[j]["text"] under that assumption. The host answers truthfully.

What to generate:
For i = 0..len(values)-1:
  For j = 0..len(choices)-1:
    - question_choice_id: the id of the question choice being evaluated, i.e. choices[j]["id"]
    - dimension_value_id: the id of the dimension value being evaluated, i.e. values[i]["id"]
    - reason: a short one-sentence explanation of why choices[j] is labeled likely/neutral/unlikely and why the other two labels were not chosen.
    - label: one of "likely", "neutral", or "unlikely" according to the following definitions:
      - "likely": Given <DIMENSION_NAME> = values[i]["text"] and the host acts according to <HOST_INFO>, the host is expected to give choices[j]["text"] for <QUESTION>.
      - "neutral": Given <DIMENSION_NAME> = values[i]["text"] and the host acts according to <HOST_INFO>, choices[j]["text"] is plausible but not specifically supported; there is insufficient evidence to say that the host would or would not prefer it over other choices.
      - "unlikely": Given <DIMENSION_NAME> = values[i]["text"] and the host acts according to <HOST_INFO>, the host is not expected to give choices[j]["text"] for <QUESTION>.

Constraints:
- Use ONLY the information provided in <PUZZLE>, <PUZZLE_CONTEXT>, and <HOST_INFO>.
- Do NOT solve the <PUZZLE> itself. Focus ONLY on judging the likelihood of the question choices under different assumptions about the dimension value.
- Do NOT rewrite or restate the <PUZZLE>, <PUZZLE_CONTEXT>, or <HOST_INFO>.
- label must be one of "likely", "neutral", or "unlikely".
- The output must include an entry for every combination of dimension value and question choice.

Output format:
Return STRICT JSON only with the following schema:
{{
    "evaluations": [
        [
            {{
                "question_choice_id": string,
                "dimension_value_id": string,
                "reason": string,
                "label": string
            }},
            ... // one object for each question choice
        ],
        ... // one array for each dimension value
    ]
}}
The "evaluations" field must contain exactly {num_dimension_values} arrays (one per dimension value).
Each inner array must contain exactly {num_question_choices} objects (one per question choice).
\end{lstlisting}

\paragraph{With conversation history.}\leavevmode\par
\begin{lstlisting}[style=prompt]
<PUZZLE>
{ambiguous_prompt}
</PUZZLE>

<PUZZLE_CONTEXT>
{meta_context}
</PUZZLE_CONTEXT>

<HOST_INFO>
{user_info}
</HOST_INFO>

<CONVERSATION_LOG>
{conversation_log}
</CONVERSATION_LOG>

<DIMENSION_NAME>
{dimension_name}
</DIMENSION_NAME>

<DIMENSION_VALUES_WITH_IDS>
{dimension_values_with_ids}
</DIMENSION_VALUES_WITH_IDS>

<QUESTION>
{question_text}
</QUESTION>

<QUESTION_CHOICES_WITH_IDS>
{question_choices_with_ids}
</QUESTION_CHOICES_WITH_IDS>

Definition:
A puzzle dimension is a hidden aspect of the scenario where knowing its true value would explain the puzzle.
<DIMENSION_VALUES_WITH_IDS> is a list of dicts with "id" and "text" fields. Each dict corresponds to a possible value that the <DIMENSION_NAME> could take.
<QUESTION_CHOICES_WITH_IDS> is a list of dicts with "id" and "text" fields. Each dict corresponds to a multiple-choice answer option for the question.
<CONVERSATION_LOG> is a list of dicts with "question_text", "user_name", and "user_answer" fields, representing the history of the conversation between the solver and the host up to this point. This information may provide additional clues.
Let values[i] be the i-th element of <DIMENSION_VALUES_WITH_IDS>.
Let choices[j] be the j-th element of <QUESTION_CHOICES_WITH_IDS>.

Task (row-major order):
For i = 0..len(values)-1:
  For j = 0..len(choices)-1:
    - Assume the hidden explanation is such that <DIMENSION_NAME> = values[i]["text"].
    - Impersonate the puzzle host described in <HOST_INFO>, who knows the hidden explanation.
    - Judge how likely it is that this host would answer the question <QUESTION> with choices[j]["text"] under that assumption. The host answers truthfully.

What to generate:
For i = 0..len(values)-1:
  For j = 0..len(choices)-1:
    - question_choice_id: the id of the question choice being evaluated, i.e. choices[j]["id"]
    - dimension_value_id: the id of the dimension value being evaluated, i.e. values[i]["id"]
    - reason: a short one-sentence explanation of why choices[j] is labeled likely/neutral/unlikely and why the other two labels were not chosen.
    - label: one of "likely", "neutral", or "unlikely" according to the following definitions:
      - "likely": Given <DIMENSION_NAME> = values[i]["text"] and the host acts according to <HOST_INFO>, the host is expected to give choices[j]["text"] for <QUESTION>.
      - "neutral": Given <DIMENSION_NAME> = values[i]["text"] and the host acts according to <HOST_INFO>, choices[j]["text"] is plausible but not specifically supported; there is insufficient evidence to say that the host would or would not prefer it over other choices.
      - "unlikely": Given <DIMENSION_NAME> = values[i]["text"] and the host acts according to <HOST_INFO>, the host is not expected to give choices[j]["text"] for <QUESTION>.

Constraints:
- Use ONLY the information provided in <PUZZLE>, <PUZZLE_CONTEXT>, <HOST_INFO>, and <CONVERSATION_LOG>.
- Do NOT solve the <PUZZLE> itself. Focus ONLY on judging the likelihood of the question choices under different assumptions about the dimension value.
- Do NOT rewrite or restate the <PUZZLE>, <PUZZLE_CONTEXT>, <HOST_INFO>, or <CONVERSATION_LOG>.
- label must be one of "likely", "neutral", or "unlikely".
- The output must include an entry for every combination of dimension value and question choice.

Output format:
Return STRICT JSON only with the following schema:
{{
    "evaluations": [
        [
            {{
                "question_choice_id": string,
                "dimension_value_id": string,
                "reason": string,
                "label": string
            }},
            ... // one object for each question choice
        ],
        ... // one array for each dimension value
    ]
}}
The "evaluations" field must contain exactly {num_dimension_values} arrays (one per dimension value).
Each inner array must contain exactly {num_question_choices} objects (one per question choice).
\end{lstlisting}

\subsubsection*{iCraft-MD}

\paragraph{System prompt.}\leavevmode\par
\begin{lstlisting}[style=prompt]
You are an experienced physician evaluating how a patient would likely respond to a clinical question under different assumptions about their underlying condition.
\end{lstlisting}

\paragraph{Without conversation history.}\leavevmode\par
\begin{lstlisting}[style=prompt]
<CLINICAL_QUESTION>
{ambiguous_prompt}
</CLINICAL_QUESTION>

<PATIENT_INFORMATION>
{meta_context}
</PATIENT_INFORMATION>

<PATIENT_PROFILE>
{user_info}
</PATIENT_PROFILE>

<DIMENSION_NAME>
{dimension_name}
</DIMENSION_NAME>

<DIMENSION_VALUES_WITH_IDS>
{dimension_values_with_ids}
</DIMENSION_VALUES_WITH_IDS>

<QUESTION>
{question_text}
</QUESTION>

<QUESTION_CHOICES_WITH_IDS>
{question_choices_with_ids}
</QUESTION_CHOICES_WITH_IDS>

Definition:
A clinical dimension is a specific clinical factor where different values would point toward different diagnoses or clinical decisions.
<DIMENSION_VALUES_WITH_IDS> is a list of dicts with "id" and "text" fields. Each dict corresponds to a possible value that the <DIMENSION_NAME> could take.
<QUESTION_CHOICES_WITH_IDS> is a list of dicts with "id" and "text" fields. Each dict corresponds to a multiple-choice answer option for the question.
Let values[i] be the i-th element of <DIMENSION_VALUES_WITH_IDS>.
Let choices[j] be the j-th element of <QUESTION_CHOICES_WITH_IDS>.

Task (row-major order):
For i = 0..len(values)-1:
  For j = 0..len(choices)-1:
    - Assume the patient's true clinical state is <DIMENSION_NAME> = values[i]["text"].
    - Impersonate the patient described in <PATIENT_PROFILE>.
    - Judge how likely it is that this patient would answer the question <QUESTION> with choices[j]["text"] under that assumption.

What to generate:
For i = 0..len(values)-1:
  For j = 0..len(choices)-1:
    - question_choice_id: the id of the question choice being evaluated, i.e. choices[j]["id"]
    - dimension_value_id: the id of the dimension value being evaluated, i.e. values[i]["id"]
    - reason: a short one-sentence explanation of why choices[j] is labeled likely/neutral/unlikely and why the other two labels were not chosen.
    - label: one of "likely", "neutral", or "unlikely" according to the following definitions:
      - "likely": Given <DIMENSION_NAME> = values[i]["text"] and the patient acts according to <PATIENT_PROFILE>, the patient is expected to give choices[j]["text"] for <QUESTION>.
      - "neutral": Given <DIMENSION_NAME> = values[i]["text"] and the patient acts according to <PATIENT_PROFILE>, choices[j]["text"] is plausible but not specifically supported; there is insufficient evidence to say that the patient would or would not prefer it over other choices.
      - "unlikely": Given <DIMENSION_NAME> = values[i]["text"] and the patient acts according to <PATIENT_PROFILE>, the patient is not expected to give choices[j]["text"] for <QUESTION>.

Constraints:
- Use ONLY the information provided in <CLINICAL_QUESTION>, <PATIENT_INFORMATION>, and <PATIENT_PROFILE>.
- Do NOT answer the <CLINICAL_QUESTION> itself. Focus ONLY on judging the likelihood of the question choices under different assumptions about the dimension value.
- Do NOT rewrite or restate the <CLINICAL_QUESTION>, <PATIENT_INFORMATION>, or <PATIENT_PROFILE>.
- label must be one of "likely", "neutral", or "unlikely".
- The output must include an entry for every combination of dimension value and question choice.

Output format:
Return STRICT JSON only with the following schema:
{{
    "evaluations": [
        [
            {{
                "question_choice_id": string,
                "dimension_value_id": string,
                "reason": string,
                "label": string
            }},
            ... // one object for each question choice
        ],
        ... // one array for each dimension value
    ]
}}
The "evaluations" field must contain exactly {num_dimension_values} arrays (one per dimension value).
Each inner array must contain exactly {num_question_choices} objects (one per question choice).
\end{lstlisting}

\paragraph{With conversation history.}\leavevmode\par
\begin{lstlisting}[style=prompt]
<CLINICAL_QUESTION>
{ambiguous_prompt}
</CLINICAL_QUESTION>

<PATIENT_INFORMATION>
{meta_context}
</PATIENT_INFORMATION>

<PATIENT_PROFILE>
{user_info}
</PATIENT_PROFILE>

<CLINICAL_INTERVIEW_LOG>
{conversation_log}
</CLINICAL_INTERVIEW_LOG>

<DIMENSION_NAME>
{dimension_name}
</DIMENSION_NAME>

<DIMENSION_VALUES_WITH_IDS>
{dimension_values_with_ids}
</DIMENSION_VALUES_WITH_IDS>

<QUESTION>
{question_text}
</QUESTION>

<QUESTION_CHOICES_WITH_IDS>
{question_choices_with_ids}
</QUESTION_CHOICES_WITH_IDS>

Definition:
A clinical dimension is a specific clinical factor where different values would point toward different diagnoses or clinical decisions.
<DIMENSION_VALUES_WITH_IDS> is a list of dicts with "id" and "text" fields. Each dict corresponds to a possible value that the <DIMENSION_NAME> could take.
<QUESTION_CHOICES_WITH_IDS> is a list of dicts with "id" and "text" fields. Each dict corresponds to a multiple-choice answer option for the question.
<CLINICAL_INTERVIEW_LOG> is a list of dicts with "question_text", "user_name", and "user_answer" fields, representing the history of the clinical interview between the physician and the patient up to this point. This information may reveal additional symptoms or clinical details.
Let values[i] be the i-th element of <DIMENSION_VALUES_WITH_IDS>.
Let choices[j] be the j-th element of <QUESTION_CHOICES_WITH_IDS>.

Task (row-major order):
For i = 0..len(values)-1:
  For j = 0..len(choices)-1:
    - Assume the patient's true clinical state is <DIMENSION_NAME> = values[i]["text"].
    - Impersonate the patient described in <PATIENT_PROFILE>.
    - Judge how likely it is that this patient would answer the question <QUESTION> with choices[j]["text"] under that assumption.

What to generate:
For i = 0..len(values)-1:
  For j = 0..len(choices)-1:
    - question_choice_id: the id of the question choice being evaluated, i.e. choices[j]["id"]
    - dimension_value_id: the id of the dimension value being evaluated, i.e. values[i]["id"]
    - reason: a short one-sentence explanation of why choices[j] is labeled likely/neutral/unlikely and why the other two labels were not chosen.
    - label: one of "likely", "neutral", or "unlikely" according to the following definitions:
      - "likely": Given <DIMENSION_NAME> = values[i]["text"] and the patient acts according to <PATIENT_PROFILE>, the patient is expected to give choices[j]["text"] for <QUESTION>.
      - "neutral": Given <DIMENSION_NAME> = values[i]["text"] and the patient acts according to <PATIENT_PROFILE>, choices[j]["text"] is plausible but not specifically supported; there is insufficient evidence to say that the patient would or would not prefer it over other choices.
      - "unlikely": Given <DIMENSION_NAME> = values[i]["text"] and the patient acts according to <PATIENT_PROFILE>, the patient is not expected to give choices[j]["text"] for <QUESTION>.

Constraints:
- Use ONLY the information provided in <CLINICAL_QUESTION>, <PATIENT_INFORMATION>, <PATIENT_PROFILE>, and <CLINICAL_INTERVIEW_LOG>.
- Do NOT answer the <CLINICAL_QUESTION> itself. Focus ONLY on judging the likelihood of the question choices under different assumptions about the dimension value.
- Do NOT rewrite or restate the <CLINICAL_QUESTION>, <PATIENT_INFORMATION>, <PATIENT_PROFILE>, or <CLINICAL_INTERVIEW_LOG>.
- label must be one of "likely", "neutral", or "unlikely".
- The output must include an entry for every combination of dimension value and question choice.

Output format:
Return STRICT JSON only with the following schema:
{{
    "evaluations": [
        [
            {{
                "question_choice_id": string,
                "dimension_value_id": string,
                "reason": string,
                "label": string
            }},
            ... // one object for each question choice
        ],
        ... // one array for each dimension value
    ]
}}
The "evaluations" field must contain exactly {num_dimension_values} arrays (one per dimension value).
Each inner array must contain exactly {num_question_choices} objects (one per question choice).
\end{lstlisting}

\subsection{Soft-Map Scoring}
\label{app:prompts:softmap}

\subsubsection*{AR-Bench-DC}

\paragraph{System prompt.}\leavevmode\par
\begin{lstlisting}[style=prompt]
You are an experienced detective analyzing a suspect's response to an interrogation question. Your task is to judge how well the suspect's answer maps to each of the predefined answer choices.
\end{lstlisting}

\paragraph{User prompt.}\leavevmode\par
\begin{lstlisting}[style=prompt]
<QUESTION>
{question}
</QUESTION>

<CHOICES_WITH_IDS>
{choices_with_ids}
</CHOICES_WITH_IDS>

<SUSPECT_ANSWER>
{user_answer}
</SUSPECT_ANSWER>

Definition:
<CHOICES_WITH_IDS> is a list of dicts with "id" and "value" fields. Each dict corresponds to a multiple-choice answer option for the question.
Let choices[i] be the i-th element of <CHOICES_WITH_IDS>.

Task:
Judge how well the <SUSPECT_ANSWER> maps to each of the choices in <CHOICES_WITH_IDS> for the question <QUESTION>.

What to generate:
For i = 0..len(choices)-1:
- choice_id: the id of the question choice being evaluated, i.e. choices[i]["id"]
- reason: a short one-sentence explanation of choices[i]["value"] is likely/neutral/unlikely given the <SUSPECT_ANSWER>.
- label: one of "likely", "neutral", or "unlikely" according to the following definitions:
    -"likely": choices[i]["value"] aligns well with the <SUSPECT_ANSWER> and fits it better than most other choices.
    -"neutral": choices[i]["value"] is neither clearly supported nor clearly contradicted by the <SUSPECT_ANSWER>.
    -"unlikely": choices[i]["value"] fits the <SUSPECT_ANSWER> worse than other choices, or conflicts with the meaning of the <SUSPECT_ANSWER>.

Constraints:
- Use ONLY the information provided in <QUESTION>, <CHOICES_WITH_IDS>, and <SUSPECT_ANSWER>.
- Do NOT answer the <QUESTION> itself. Focus ONLY on judging how well the <SUSPECT_ANSWER> maps to the provided choices.
- Do NOT rewrite or restate the <QUESTION> or <SUSPECT_ANSWER>.
- label must be one of "likely", "neutral", or "unlikely".
- The output must include an entry for every choice in <CHOICES_WITH_IDS>.

Output format:
Return STRICT JSON only with the following schema:
{{
    "scores": [
        {{
            "choice_id": string,
            "reason": string,
            "label": string
        }},
        ...
    ]
}}
\end{lstlisting}

\subsubsection*{AR-Bench-SP}

\paragraph{System prompt.}\leavevmode\par
\begin{lstlisting}[style=prompt]
You are an expert thinking puzzle analyst. Your task is to judge how well a puzzle host's response maps to each of the predefined answer choices. The host gives truthful answers ("yes" or "no").
\end{lstlisting}

\paragraph{User prompt.}\leavevmode\par
\begin{lstlisting}[style=prompt]
<QUESTION>
{question}
</QUESTION>

<CHOICES_WITH_IDS>
{choices_with_ids}
</CHOICES_WITH_IDS>

<HOST_ANSWER>
{user_answer}
</HOST_ANSWER>

Definition:
<CHOICES_WITH_IDS> is a list of dicts with "id" and "value" fields. Each dict corresponds to a multiple-choice answer option for the question.
Let choices[i] be the i-th element of <CHOICES_WITH_IDS>.

Task:
Judge how well the <HOST_ANSWER> maps to each of the choices in <CHOICES_WITH_IDS> for the question <QUESTION>.

What to generate:
For i = 0..len(choices)-1:
- choice_id: the id of the question choice being evaluated, i.e. choices[i]["id"]
- reason: a short one-sentence explanation of choices[i]["value"] is likely/neutral/unlikely given the <HOST_ANSWER>.
- label: one of "likely", "neutral", or "unlikely" according to the following definitions:
    -"likely": choices[i]["value"] aligns well with the <HOST_ANSWER> and fits it better than most other choices.
    -"neutral": choices[i]["value"] is neither clearly supported nor clearly contradicted by the <HOST_ANSWER>.
    -"unlikely": choices[i]["value"] fits the <HOST_ANSWER> worse than other choices, or conflicts with the meaning of the <HOST_ANSWER>.

Constraints:
- Use ONLY the information provided in <QUESTION>, <CHOICES_WITH_IDS>, and <HOST_ANSWER>.
- Do NOT answer the <QUESTION> itself. Focus ONLY on judging how well the <HOST_ANSWER> maps to the provided choices.
- Do NOT rewrite or restate the <QUESTION> or <HOST_ANSWER>.
- label must be one of "likely", "neutral", or "unlikely".
- The output must include an entry for every choice in <CHOICES_WITH_IDS>.

Output format:
Return STRICT JSON only with the following schema:
{{
    "scores": [
        {{
            "choice_id": string,
            "reason": string,
            "label": string
        }},
        ...
    ]
}}
\end{lstlisting}

\subsubsection*{iCraft-MD}

\paragraph{System prompt.}\leavevmode\par
\begin{lstlisting}[style=prompt]
You are an experienced physician interpreting a patient's response to a clinical question. Your task is to judge how well the patient's answer maps to each of the predefined answer choices.
\end{lstlisting}

\paragraph{User prompt.}\leavevmode\par
\begin{lstlisting}[style=prompt]
<QUESTION>
{question}
</QUESTION>

<CHOICES_WITH_IDS>
{choices_with_ids}
</CHOICES_WITH_IDS>

<PATIENT_ANSWER>
{user_answer}
</PATIENT_ANSWER>

Definition:
<CHOICES_WITH_IDS> is a list of dicts with "id" and "value" fields. Each dict corresponds to a multiple-choice answer option for the question.
Let choices[i] be the i-th element of <CHOICES_WITH_IDS>.

Task:
Judge how well the <PATIENT_ANSWER> maps to each of the choices in <CHOICES_WITH_IDS> for the question <QUESTION>.

What to generate:
For i = 0..len(choices)-1:
- choice_id: the id of the question choice being evaluated, i.e. choices[i]["id"]
- reason: a short one-sentence explanation of choices[i]["value"] is likely/neutral/unlikely given the <PATIENT_ANSWER>.
- label: one of "likely", "neutral", or "unlikely" according to the following definitions:
    -"likely": choices[i]["value"] aligns well with the <PATIENT_ANSWER> and fits it better than most other choices.
    -"neutral": choices[i]["value"] is neither clearly supported nor clearly contradicted by the <PATIENT_ANSWER>.
    -"unlikely": choices[i]["value"] fits the <PATIENT_ANSWER> worse than other choices, or conflicts with the meaning of the <PATIENT_ANSWER>.

Constraints:
- Use ONLY the information provided in <QUESTION>, <CHOICES_WITH_IDS>, and <PATIENT_ANSWER>.
- Do NOT answer the <QUESTION> itself. Focus ONLY on judging how well the <PATIENT_ANSWER> maps to the provided choices.
- Do NOT rewrite or restate the <QUESTION> or <PATIENT_ANSWER>.
- label must be one of "likely", "neutral", or "unlikely".
- The output must include an entry for every choice in <CHOICES_WITH_IDS>.

Output format:
Return STRICT JSON only with the following schema:
{{
    "scores": [
        {{
            "choice_id": string,
            "reason": string,
            "label": string
        }},
        ...
    ]
}}
\end{lstlisting}

\subsection{Expand --- New Dimension Proposal}
\label{app:prompts:expand_dim}

\subsubsection*{AR-Bench-DC}

\paragraph{System prompt.}\leavevmode\par
\begin{lstlisting}[style=prompt]
You are an experienced detective who needs to explore a new line of investigation in a murder case. The current investigative dimensions have not been sufficient to identify the murderer, so you must identify a new aspect of the case to investigate.
\end{lstlisting}

\paragraph{User prompt.}\leavevmode\par
\begin{lstlisting}[style=prompt]
<CASE_QUESTION>
{ambiguous_prompt}
</CASE_QUESTION>

<CASE_BACKGROUND>
{meta_context}
</CASE_BACKGROUND>

<PAST_INVESTIGATIVE_DIMENSIONS>
{past_dimensions}
</PAST_INVESTIGATIVE_DIMENSIONS>

<INTERROGATION_LOG>
{conversation_log}
</INTERROGATION_LOG>

Definition:
An investigative dimension is a specific aspect of the murder case where multiple suspects could plausibly be implicated, and resolving it would narrow down the true murderer.
<PAST_INVESTIGATIVE_DIMENSIONS> is a list of dicts with "name" fields, representing the investigative dimensions that have already been explored in the investigation.
<INTERROGATION_LOG> is a list of dicts with "question_text", "user_name", and "user_answer" fields, representing the history of the interrogation between the detective and the suspects up to this point.

Task:
Identify a new investigative dimension in the murder case described by <CASE_QUESTION> that has not been previously identified in <PAST_INVESTIGATIVE_DIMENSIONS>. Use insights from the <INTERROGATION_LOG> to guide your choice.

What to generate:
- reason: a short one-sentence explanation of why this dimension is a critical new line of investigation.
- name: a short, specific label for this investigative dimension (e.g., "Forensic evidence", "Financial motive", "Witness credibility", etc.).
- values: a list of plausible values, no larger than {max_num_values_per_dim}, that this dimension could take in the context of the case.

Constraints:
- Use ONLY the information provided in <CASE_QUESTION>, <CASE_BACKGROUND>, <PAST_INVESTIGATIVE_DIMENSIONS>, and <INTERROGATION_LOG>.
- Do NOT answer the <CASE_QUESTION> itself. Focus ONLY on identifying a new investigative dimension.
- Do NOT rewrite or restate the <CASE_QUESTION>, <CASE_BACKGROUND>, <PAST_INVESTIGATIVE_DIMENSIONS>, or <INTERROGATION_LOG>.
- The generated dimension name must not be the same as any of the names in <PAST_INVESTIGATIVE_DIMENSIONS>.

Output format:
Return STRICT JSON only with the following schema:
{{
    "reason": string,
    "name": string,
    "values": [string, ...]
}}
\end{lstlisting}

\subsubsection*{AR-Bench-SP}

\paragraph{System prompt.}\leavevmode\par
\begin{lstlisting}[style=prompt]
You are an expert thinking puzzle solver. The puzzle dimensions explored so far have not been sufficient to explain the puzzle, so you must identify a new hidden aspect of the scenario to investigate.
\end{lstlisting}

\paragraph{User prompt.}\leavevmode\par
\begin{lstlisting}[style=prompt]
<PUZZLE>
{ambiguous_prompt}
</PUZZLE>

<PUZZLE_CONTEXT>
{meta_context}
</PUZZLE_CONTEXT>

<PAST_PUZZLE_DIMENSIONS>
{past_dimensions}
</PAST_PUZZLE_DIMENSIONS>

<CONVERSATION_LOG>
{conversation_log}
</CONVERSATION_LOG>

Definition:
A puzzle dimension is a hidden aspect of the scenario where knowing its true value would explain the puzzle.
<PAST_PUZZLE_DIMENSIONS> is a list of dicts with "name" fields, representing the puzzle dimensions that have already been explored.
<CONVERSATION_LOG> is a list of dicts with "question_text", "user_name", and "user_answer" fields, representing the history of the conversation between the solver and the host up to this point.

Task:
Identify a new puzzle dimension in <PUZZLE> that has not been previously identified in <PAST_PUZZLE_DIMENSIONS>. Use insights from the <CONVERSATION_LOG> to guide your choice.

What to generate:
- reason: a short one-sentence explanation of why this dimension is a key unknown in the puzzle.
- name: a short, specific label for this puzzle dimension.
- values: a list of plausible interpretations, no larger than {max_num_values_per_dim}, that this dimension could take.

Constraints:
- Use ONLY the information provided in <PUZZLE>, <PUZZLE_CONTEXT>, <PAST_PUZZLE_DIMENSIONS>, and <CONVERSATION_LOG>.
- Do NOT solve the <PUZZLE> itself. Focus ONLY on identifying a new puzzle dimension.
- Do NOT rewrite or restate the <PUZZLE>, <PUZZLE_CONTEXT>, <PAST_PUZZLE_DIMENSIONS>, or <CONVERSATION_LOG>.
- The generated dimension name must not be the same as any of the names in <PAST_PUZZLE_DIMENSIONS>.

Output format:
Return STRICT JSON only with the following schema:
{{
    "reason": string,
    "name": string,
    "values": [string, ...]
}}
\end{lstlisting}

\subsubsection*{iCraft-MD}

\paragraph{System prompt.}\leavevmode\par
\begin{lstlisting}[style=prompt]
You are an experienced physician who needs to explore a new line of clinical inquiry. The clinical dimensions investigated so far have not been sufficient to arrive at a definitive diagnosis, so you must identify a new clinical factor to assess.
\end{lstlisting}

\paragraph{User prompt.}\leavevmode\par
\begin{lstlisting}[style=prompt]
<CLINICAL_QUESTION>
{ambiguous_prompt}
</CLINICAL_QUESTION>

<PATIENT_INFORMATION>
{meta_context}
</PATIENT_INFORMATION>

<PAST_CLINICAL_DIMENSIONS>
{past_dimensions}
</PAST_CLINICAL_DIMENSIONS>

<CLINICAL_INTERVIEW_LOG>
{conversation_log}
</CLINICAL_INTERVIEW_LOG>

Definition:
A clinical dimension is a specific clinical factor where different values would point toward different diagnoses or clinical decisions.
<PAST_CLINICAL_DIMENSIONS> is a list of dicts with "name" fields, representing the clinical dimensions that have already been assessed in the patient interview.
<CLINICAL_INTERVIEW_LOG> is a list of dicts with "question_text", "user_name", and "user_answer" fields, representing the history of the clinical interview between the physician and the patient up to this point.

Task:
Identify a new clinical dimension relevant to the <CLINICAL_QUESTION> that has not been previously identified in <PAST_CLINICAL_DIMENSIONS>. Use insights from the <CLINICAL_INTERVIEW_LOG> to guide your choice — the patient's answers may reveal the need to investigate additional clinical factors.

What to generate:
- reason: a short one-sentence explanation of why this clinical factor is important for narrowing the diagnosis.
- name: a short, specific clinical label for this dimension.
- values: a list of clinically plausible values, no larger than {max_num_values_per_dim}, that this dimension could take.

Constraints:
- Use ONLY the information provided in <CLINICAL_QUESTION>, <PATIENT_INFORMATION>, <PAST_CLINICAL_DIMENSIONS>, and <CLINICAL_INTERVIEW_LOG>.
- Do NOT answer the <CLINICAL_QUESTION> itself. Focus ONLY on identifying a new clinical dimension.
- Do NOT rewrite or restate the <CLINICAL_QUESTION>, <PATIENT_INFORMATION>, <PAST_CLINICAL_DIMENSIONS>, or <CLINICAL_INTERVIEW_LOG>.
- The generated dimension name must not be the same as any of the names in <PAST_CLINICAL_DIMENSIONS>.

Output format:
Return STRICT JSON only with the following schema:
{{
    "reason": string,
    "name": string,
    "values": [string, ...]
}}
\end{lstlisting}

\subsection{Expand --- Prior Elicitation}
\label{app:prompts:expand_priors}

\subsubsection*{AR-Bench-DC}

\paragraph{System prompt.}\leavevmode\par
\begin{lstlisting}[style=prompt]
You are an experienced detective forming a hypothesis about a newly identified aspect of a murder case, taking into account both the case background and what has been revealed during the interrogation so far.
\end{lstlisting}

\paragraph{User prompt.}\leavevmode\par
\begin{lstlisting}[style=prompt]
<CASE_QUESTION>
{ambiguous_prompt}
</CASE_QUESTION>

<CASE_BACKGROUND>
{meta_context}
</CASE_BACKGROUND>

<INTERROGATION_LOG>
{conversation_log}
</INTERROGATION_LOG>

<DIMENSION_NAME>
{dimension_name}
</DIMENSION_NAME>

<DIMENSION_VALUE>
{dimension_value}
</DIMENSION_VALUE>

Definition:
An investigative dimension is a specific aspect of the murder case where multiple suspects could plausibly be implicated, and resolving it would narrow down the true murderer.
<INTERROGATION_LOG> is a list of dicts with "question_text", "user_name", and "user_answer" fields, representing the history of the interrogation between the detective and the suspects up to this point.

Task:
Given <CASE_QUESTION>, <CASE_BACKGROUND>, <INTERROGATION_LOG>, and a specific investigative dimension defined by <DIMENSION_NAME> and <DIMENSION_VALUE>, judge how likely it is that the <DIMENSION_NAME> takes on the value <DIMENSION_VALUE>.

What to generate:
- reason: a short one-sentence explanation of why the <DIMENSION_NAME> is likely, unlikely, or neutral to take on the value <DIMENSION_VALUE>.
- label: one of "likely", "unlikely", or "neutral" according to the following definitions:
    - likely: <DIMENSION_VALUE> is explicitly stated, strongly implied, or is the most natural assumption given the evidence in the <CASE_QUESTION>, <CASE_BACKGROUND>, and <INTERROGATION_LOG>.
    - neutral: <DIMENSION_VALUE> is plausible but not implied or supported by specific evidence in the <CASE_QUESTION>, <CASE_BACKGROUND>, or <INTERROGATION_LOG>.
    - unlikely: <DIMENSION_VALUE> is contradicted by the <CASE_QUESTION>, <CASE_BACKGROUND> or <INTERROGATION_LOG>, or would require assumptions that are inconsistent with the available evidence.

Constraints:
- Use ONLY the information provided in <CASE_QUESTION>, <CASE_BACKGROUND>, and <INTERROGATION_LOG>.
- Do NOT answer the <CASE_QUESTION> itself. Focus ONLY on judging the likelihood of the dimension value.
- Do NOT rewrite or restate the <CASE_QUESTION>, <CASE_BACKGROUND>, or <INTERROGATION_LOG>.
- label must be one of "likely", "unlikely", or "neutral".

Output format:
Return STRICT JSON only with the following schema:
{{
    "reason": string,
    "label": string
}}
\end{lstlisting}

\subsubsection*{AR-Bench-SP}

\paragraph{System prompt.}\leavevmode\par
\begin{lstlisting}[style=prompt]
You are an expert thinking puzzle solver forming a hypothesis about a newly identified hidden aspect of the puzzle, taking into account both the puzzle scenario and what the host has revealed so far.
\end{lstlisting}

\paragraph{User prompt.}\leavevmode\par
\begin{lstlisting}[style=prompt]
<PUZZLE>
{ambiguous_prompt}
</PUZZLE>

<PUZZLE_CONTEXT>
{meta_context}
</PUZZLE_CONTEXT>

<CONVERSATION_LOG>
{conversation_log}
</CONVERSATION_LOG>

<DIMENSION_NAME>
{dimension_name}
</DIMENSION_NAME>

<DIMENSION_VALUE>
{dimension_value}
</DIMENSION_VALUE>

Definition:
A puzzle dimension is a hidden aspect of the scenario where knowing its true value would explain the puzzle.
<CONVERSATION_LOG> is a list of dicts with "question_text", "user_name", and "user_answer" fields, representing the history of the conversation between the solver and the host up to this point.

Task:
Given <PUZZLE>, <PUZZLE_CONTEXT>, <CONVERSATION_LOG>, and a specific puzzle dimension defined by <DIMENSION_NAME> and <DIMENSION_VALUE>, judge how likely it is that the <DIMENSION_NAME> takes on the value <DIMENSION_VALUE>.

What to generate:
- reason: a short one-sentence explanation of why the <DIMENSION_NAME> is likely, unlikely, or neutral to take on the value <DIMENSION_VALUE>.
- label: one of "likely", "unlikely", or "neutral" according to the following definitions:
    - likely: <DIMENSION_VALUE> is explicitly suggested by, strongly implied by, or is the most natural interpretation given the clues in the <PUZZLE>, <PUZZLE_CONTEXT>, and <CONVERSATION_LOG>.
    - neutral: <DIMENSION_VALUE> is plausible but not implied or supported by specific clues in the <PUZZLE>, <PUZZLE_CONTEXT>, or <CONVERSATION_LOG>.
    - unlikely: <DIMENSION_VALUE> is contradicted by the <PUZZLE>, <PUZZLE_CONTEXT> or <CONVERSATION_LOG>, or would require assumptions that are inconsistent with the scenario.

Constraints:
- Use ONLY the information provided in <PUZZLE>, <PUZZLE_CONTEXT>, and <CONVERSATION_LOG>.
- Do NOT solve the <PUZZLE> itself. Focus ONLY on judging the likelihood of the dimension value.
- Do NOT rewrite or restate the <PUZZLE>, <PUZZLE_CONTEXT>, or <CONVERSATION_LOG>.
- label must be one of "likely", "unlikely", or "neutral".

Output format:
Return STRICT JSON only with the following schema:
{{
    "reason": string,
    "label": string
}}
\end{lstlisting}

\subsubsection*{iCraft-MD}

\paragraph{System prompt.}\leavevmode\par
\begin{lstlisting}[style=prompt]
You are an experienced physician forming a clinical hypothesis about a newly identified clinical factor, taking into account the patient's baseline information and what has been revealed during the clinical interview so far.
\end{lstlisting}

\paragraph{User prompt.}\leavevmode\par
\begin{lstlisting}[style=prompt]
<CLINICAL_QUESTION>
{ambiguous_prompt}
</CLINICAL_QUESTION>

<PATIENT_INFORMATION>
{meta_context}
</PATIENT_INFORMATION>

<CLINICAL_INTERVIEW_LOG>
{conversation_log}
</CLINICAL_INTERVIEW_LOG>

<DIMENSION_NAME>
{dimension_name}
</DIMENSION_NAME>

<DIMENSION_VALUE>
{dimension_value}
</DIMENSION_VALUE>

Definition:
A clinical dimension is a specific clinical factor where different values would point toward different diagnoses or clinical decisions.
<CLINICAL_INTERVIEW_LOG> is a list of dicts with "question_text", "user_name", and "user_answer" fields, representing the history of the clinical interview between the physician and the patient up to this point.

Task:
Given <CLINICAL_QUESTION>, <PATIENT_INFORMATION>, <CLINICAL_INTERVIEW_LOG>, and a specific clinical dimension defined by <DIMENSION_NAME> and <DIMENSION_VALUE>, judge how likely it is that the <DIMENSION_NAME> takes on the value <DIMENSION_VALUE>.

What to generate:
- reason: a short one-sentence explanation of why the <DIMENSION_NAME> is likely, unlikely, or neutral to take on the value <DIMENSION_VALUE>.
- label: one of "likely", "unlikely", or "neutral" according to the following definitions:
    - likely: <DIMENSION_VALUE> is explicitly stated, strongly implied, or is the most natural clinical assumption given the patient information in <CLINICAL_QUESTION>, <PATIENT_INFORMATION>, and <CLINICAL_INTERVIEW_LOG>.
    - neutral: <DIMENSION_VALUE> is clinically plausible but not implied or supported by specific evidence in the <CLINICAL_QUESTION>, <PATIENT_INFORMATION>, or <CLINICAL_INTERVIEW_LOG>.
    - unlikely: <DIMENSION_VALUE> is contradicted by the <CLINICAL_QUESTION>, <PATIENT_INFORMATION> or <CLINICAL_INTERVIEW_LOG>, or would require assumptions that are inconsistent with the patient's presentation.

Constraints:
- Use ONLY the information provided in <CLINICAL_QUESTION>, <PATIENT_INFORMATION>, and <CLINICAL_INTERVIEW_LOG>.
- Do NOT answer the <CLINICAL_QUESTION> itself. Focus ONLY on judging the likelihood of the dimension value.
- Do NOT rewrite or restate the <CLINICAL_QUESTION>, <PATIENT_INFORMATION>, or <CLINICAL_INTERVIEW_LOG>.
- label must be one of "likely", "unlikely", or "neutral".

Output format:
Return STRICT JSON only with the following schema:
{{
    "reason": string,
    "label": string
}}
\end{lstlisting}

\subsection{Expand --- Question Generation}
\label{app:prompts:expand_questions}

\subsubsection*{AR-Bench-DC}

\paragraph{System prompt.}\leavevmode\par
\begin{lstlisting}[style=prompt]
You are an experienced detective preparing new interrogation questions based on a newly discovered line of investigation and unresolved aspects of the murder case.
\end{lstlisting}

\paragraph{User prompt.}\leavevmode\par
\begin{lstlisting}[style=prompt]
<CASE_QUESTION>
{ambiguous_prompt}
</CASE_QUESTION>

<CASE_BACKGROUND>
{meta_context}
</CASE_BACKGROUND>

<INTERROGATION_LOG>
{conversation_log}
</INTERROGATION_LOG>

<NEW_INVESTIGATIVE_DIMENSION>
{new_dimension_with_values}
</NEW_INVESTIGATIVE_DIMENSION>

<UNRESOLVED_INVESTIGATIVE_DIMENSIONS>
{high_uncertainty_dimensions_with_values}
</UNRESOLVED_INVESTIGATIVE_DIMENSIONS>

Definition:
An investigative dimension is a specific aspect of the murder case where multiple suspects could plausibly be implicated, and resolving it would narrow down the true murderer.
<INTERROGATION_LOG> is a list of dicts with "question_text", "user_name", and "user_answer" fields, representing the history of the interrogation between the detective and the suspects up to this point.
<NEW_INVESTIGATIVE_DIMENSION> is a dict with "name" and "values" fields, representing the newly identified line of investigation along with its possible values.
<UNRESOLVED_INVESTIGATIVE_DIMENSIONS> is a list of dicts with "name" and "values" fields, representing the investigative dimensions that currently have the highest uncertainty. They do not include the new dimension in <NEW_INVESTIGATIVE_DIMENSION>.

Task:
Given <CASE_QUESTION>, <CASE_BACKGROUND>, <INTERROGATION_LOG>, a newly identified investigative dimension in <NEW_INVESTIGATIVE_DIMENSION>, and the most uncertain dimensions in <UNRESOLVED_INVESTIGATIVE_DIMENSIONS>, generate interrogation questions that would help identify the murderer by targeting the new dimension and/or the unresolved dimensions.

What to generate:
Generate at most {max_new_questions_per_round} interrogation questions. For each question, provide:
- reason: a short one-sentence explanation of why this question would help identify the murderer.
- question: the text of the interrogation question.
- choices: a list of multiple-choice answer options for the question, no larger than {max_choices_per_question}.

Constraints:
- Use ONLY the information provided in <CASE_QUESTION>, <CASE_BACKGROUND>, <INTERROGATION_LOG>, <NEW_INVESTIGATIVE_DIMENSION>, and <UNRESOLVED_INVESTIGATIVE_DIMENSIONS>.
- Do NOT answer the <CASE_QUESTION> itself. Focus ONLY on generating interrogation questions.
- Do NOT rewrite or restate the <CASE_QUESTION>, <CASE_BACKGROUND>, <INTERROGATION_LOG>, <NEW_INVESTIGATIVE_DIMENSION>, or <UNRESOLVED_INVESTIGATIVE_DIMENSIONS>.
- Each question must be designed to elicit information about the new dimension in <NEW_INVESTIGATIVE_DIMENSION> and/or the unresolved dimensions in <UNRESOLVED_INVESTIGATIVE_DIMENSIONS>.
- Each question must have multiple-choice answers.
- Generate at most {max_new_questions_per_round} questions.

Output format:
Return STRICT JSON only with the following schema:
{{
    "questions": [
        {{
            "reason": string,
            "question": string,
            "choices": [string, ...]
        }},
        ...
    ]
}}
\end{lstlisting}

\subsubsection*{AR-Bench-SP}

\paragraph{System prompt.}\leavevmode\par
\begin{lstlisting}[style=prompt]
You are an expert thinking puzzle solver generating new questions to ask the puzzle host. Focus on testing specific hypotheses about the hidden explanation, especially targeting newly discovered or still-uncertain aspects of the puzzle.
\end{lstlisting}

\paragraph{User prompt.}\leavevmode\par
\begin{lstlisting}[style=prompt]
<PUZZLE>
{ambiguous_prompt}
</PUZZLE>

<PUZZLE_CONTEXT>
{meta_context}
</PUZZLE_CONTEXT>

<CONVERSATION_LOG>
{conversation_log}
</CONVERSATION_LOG>

<NEW_PUZZLE_DIMENSION>
{new_dimension_with_values}
</NEW_PUZZLE_DIMENSION>

<UNCERTAIN_PUZZLE_DIMENSIONS>
{high_uncertainty_dimensions_with_values}
</UNCERTAIN_PUZZLE_DIMENSIONS>

Definition:
A puzzle dimension is a hidden aspect of the scenario where knowing its true value would explain the puzzle.
<CONVERSATION_LOG> is a list of dicts with "question_text", "user_name", and "user_answer" fields, representing the history of the conversation between the solver and the host up to this point.
<NEW_PUZZLE_DIMENSION> is a dict with "name" and "values" fields, representing the newly identified hidden aspect of the puzzle along with its possible interpretations.
<UNCERTAIN_PUZZLE_DIMENSIONS> is a list of dicts with "name" and "values" fields, representing the puzzle dimensions that currently have the highest uncertainty. They do not include the new dimension in <NEW_PUZZLE_DIMENSION>.

Task:
Given <PUZZLE>, <PUZZLE_CONTEXT>, <CONVERSATION_LOG>, a newly identified puzzle dimension in <NEW_PUZZLE_DIMENSION>, and the most uncertain dimensions in <UNCERTAIN_PUZZLE_DIMENSIONS>, generate questions to ask the puzzle host that would help uncover the hidden explanation by targeting the new dimension and/or the uncertain dimensions.

What to generate:
Generate at most {max_new_questions_per_round} clarifying questions. For each question, provide:
- reason: a short one-sentence explanation of why this question would help solve the puzzle.
- question: the text of the clarifying question to ask the puzzle host.
- choices: ["yes", "no"] as the {max_choices_per_question} multiple-choice answer options for the question.

Constraints:
- Use ONLY the information provided in <PUZZLE>, <PUZZLE_CONTEXT>, <CONVERSATION_LOG>, <NEW_PUZZLE_DIMENSION>, and <UNCERTAIN_PUZZLE_DIMENSIONS>.
- Do NOT solve the <PUZZLE> itself. Focus ONLY on generating clarifying questions.
- Do NOT rewrite or restate the <PUZZLE>, <PUZZLE_CONTEXT>, <CONVERSATION_LOG>, <NEW_PUZZLE_DIMENSION>, or <UNCERTAIN_PUZZLE_DIMENSIONS>.
- Each question must be designed to elicit information about the new dimension in <NEW_PUZZLE_DIMENSION> and/or the uncertain dimensions in <UNCERTAIN_PUZZLE_DIMENSIONS>.
- Each question must have multiple-choice answers.
- Generate at most {max_new_questions_per_round} questions.
- Keep each question short: at most 20 words.
- Keep each reason short: at most 15 words.

Output format:
Return STRICT JSON only with the following schema:
{{
    "questions": [
        {{
            "reason": string,
            "question": string,
            "choices": ["yes", "no"]
        }},
        ...
    ]
}}
\end{lstlisting}

\subsubsection*{iCraft-MD}

\paragraph{System prompt.}\leavevmode\par
\begin{lstlisting}[style=prompt]
You are an experienced physician preparing additional clinical questions for a patient interview based on a newly identified clinical factor and unresolved aspects of the diagnosis.
\end{lstlisting}

\paragraph{User prompt.}\leavevmode\par
\begin{lstlisting}[style=prompt]
<CLINICAL_QUESTION>
{ambiguous_prompt}
</CLINICAL_QUESTION>

<PATIENT_INFORMATION>
{meta_context}
</PATIENT_INFORMATION>

<CLINICAL_INTERVIEW_LOG>
{conversation_log}
</CLINICAL_INTERVIEW_LOG>

<NEW_CLINICAL_DIMENSION>
{new_dimension_with_values}
</NEW_CLINICAL_DIMENSION>

<UNRESOLVED_CLINICAL_DIMENSIONS>
{high_uncertainty_dimensions_with_values}
</UNRESOLVED_CLINICAL_DIMENSIONS>

Definition:
A clinical dimension is a specific clinical factor where different values would point toward different diagnoses or clinical decisions.
<CLINICAL_INTERVIEW_LOG> is a list of dicts with "question_text", "user_name", and "user_answer" fields, representing the history of the clinical interview between the physician and the patient up to this point.
<NEW_CLINICAL_DIMENSION> is a dict with "name" and "values" fields, representing the newly identified clinical factor along with its possible values.
<UNRESOLVED_CLINICAL_DIMENSIONS> is a list of dicts with "name" and "values" fields, representing the clinical dimensions that currently have the highest diagnostic uncertainty. They do not include the new dimension in <NEW_CLINICAL_DIMENSION>.

Task:
Given <CLINICAL_QUESTION>, <PATIENT_INFORMATION>, <CLINICAL_INTERVIEW_LOG>, a newly identified clinical dimension in <NEW_CLINICAL_DIMENSION>, and the most uncertain dimensions in <UNRESOLVED_CLINICAL_DIMENSIONS>, generate clinical questions to ask the patient that would help narrow the diagnosis by targeting the new dimension and/or the unresolved dimensions.

What to generate:
Generate at most {max_new_questions_per_round} clinical questions. For each question, provide:
- reason: a short one-sentence explanation of why this question would help narrow the diagnosis.
- question: the text of the clinical question to ask the patient.
- choices: a list of multiple-choice answer options for the question, no larger than {max_choices_per_question}.

Constraints:
- Use ONLY the information provided in <CLINICAL_QUESTION>, <PATIENT_INFORMATION>, <CLINICAL_INTERVIEW_LOG>, <NEW_CLINICAL_DIMENSION>, and <UNRESOLVED_CLINICAL_DIMENSIONS>.
- Do NOT answer the <CLINICAL_QUESTION> itself. Focus ONLY on generating clinical questions.
- Do NOT rewrite or restate the <CLINICAL_QUESTION>, <PATIENT_INFORMATION>, <CLINICAL_INTERVIEW_LOG>, <NEW_CLINICAL_DIMENSION>, or <UNRESOLVED_CLINICAL_DIMENSIONS>.
- Each question must be designed to elicit information about the new dimension in <NEW_CLINICAL_DIMENSION> and/or the unresolved dimensions in <UNRESOLVED_CLINICAL_DIMENSIONS>.
- Each question must have multiple-choice answers.
- Generate at most {max_new_questions_per_round} questions.

Output format:
Return STRICT JSON only with the following schema:
{{
    "questions": [
        {{
            "reason": string,
            "question": string,
            "choices": [string, ...]
        }},
        ...
    ]
}}
\end{lstlisting}

\subsection{Final Answer}
\label{app:prompts:final}

\subsubsection*{AR-Bench-DC}

\paragraph{System prompt.}\leavevmode\par
\begin{lstlisting}[style=prompt]
You are an experienced detective concluding a murder investigation. Based on all the evidence gathered from interrogating the suspects and your analysis of the case, you must now identify the real murderer.
\end{lstlisting}

\paragraph{Without discrete choices.}\leavevmode\par
\begin{lstlisting}[style=prompt]
<CASE_QUESTION>
{ambiguous_prompt}
</CASE_QUESTION>

<CASE_BACKGROUND>
{meta_context}
</CASE_BACKGROUND>

<INTERROGATION_LOG>
{conversation_log}
</INTERROGATION_LOG>

<INVESTIGATION_CONCLUSION>
{map_state}
</INVESTIGATION_CONCLUSION>

Definition:
<INTERROGATION_LOG> is a list of dicts with "question_text", "user_name", and "user_answer" fields, representing the full history of the interrogation between the detective and the suspects.
An investigative dimension is a specific aspect of the murder case where multiple suspects could plausibly be implicated, and resolving it would narrow down the true murderer.
<INVESTIGATION_CONCLUSION> is a structured representation of the detective's current understanding of the case, where each investigative dimension is mapped to its most likely value. This represents the detective's best assessment of the true state of the case based on the investigation so far.

Task:
Given <CASE_QUESTION>, <CASE_BACKGROUND>, <INTERROGATION_LOG>, and <INVESTIGATION_CONCLUSION>, identify the real murderer.

What to generate:
- reason: a short one-sentence explanation of why the identified suspect is the real murderer given the evidence in <CASE_QUESTION>, <CASE_BACKGROUND>, <INTERROGATION_LOG>, and <INVESTIGATION_CONCLUSION>.
- final_answer: the name of the suspect you are identifying as the real murderer.

Constraints:
- The final answer must be consistent with the evidence in <CASE_QUESTION>, <CASE_BACKGROUND>, <INTERROGATION_LOG>, and <INVESTIGATION_CONCLUSION>.
- Do NOT rewrite or restate the <CASE_QUESTION>, <CASE_BACKGROUND>, <INTERROGATION_LOG>, or <INVESTIGATION_CONCLUSION>.
- answer must be a natural language answer to the <CASE_QUESTION>.

Output format:
Return STRICT JSON only with the following schema:
{{
    "reason": string,
    "final_answer": string
}}
\end{lstlisting}

\paragraph{With discrete choices.}\leavevmode\par
\begin{lstlisting}[style=prompt]
<CASE_QUESTION>
{ambiguous_prompt}
</CASE_QUESTION>

<CASE_BACKGROUND>
{meta_context}
</CASE_BACKGROUND>

<INTERROGATION_LOG>
{conversation_log}
</INTERROGATION_LOG>

<INVESTIGATION_CONCLUSION>
{map_state}
</INVESTIGATION_CONCLUSION>

<SUSPECTS_WITH_IDS>
{possible_answers_with_ids}
</SUSPECTS_WITH_IDS>

Definition:
<INTERROGATION_LOG> is a list of dicts with "question_text", "user_name", and "user_answer" fields, representing the full history of the interrogation between the detective and the suspects.
An investigative dimension is a specific aspect of the murder case where multiple suspects could plausibly be implicated, and resolving it would narrow down the true murderer.
<INVESTIGATION_CONCLUSION> is a structured representation of the detective's current understanding of the case, where each investigative dimension is mapped to its most likely value. This represents the detective's best assessment of the true state of the case based on the investigation so far.
<SUSPECTS_WITH_IDS> is a list of dicts with "id" and "value" fields. Each dict corresponds to a suspect who could be the murderer.
Let possible_answers[i] be the i-th element of <SUSPECTS_WITH_IDS>.

Task:
Given <CASE_QUESTION>, <CASE_BACKGROUND>, <INTERROGATION_LOG>, and <INVESTIGATION_CONCLUSION>, identify which suspect from <SUSPECTS_WITH_IDS> is the real murderer.

What to generate:
- reason: a short one-sentence explanation of why the identified suspect is the real murderer given the evidence in <CASE_QUESTION>, <CASE_BACKGROUND>, <INTERROGATION_LOG>, and <INVESTIGATION_CONCLUSION>.
- final_answer_id: the id of the suspect in <SUSPECTS_WITH_IDS> that you are identifying as the real murderer.

Constraints:
- The final answer must be consistent with the evidence in <CASE_QUESTION>, <CASE_BACKGROUND>, <INTERROGATION_LOG>, and <INVESTIGATION_CONCLUSION>.
- Do NOT rewrite or restate the <CASE_QUESTION>, <CASE_BACKGROUND>, <INTERROGATION_LOG>, or <INVESTIGATION_CONCLUSION>.
- final_answer_id must be one of the ids (i.e possible_answers[i]["id"]) provided in <SUSPECTS_WITH_IDS>.

Output format:
Return STRICT JSON only with the following schema:
{{
    "reason": string,
    "final_answer_id": string
}}
\end{lstlisting}

\subsubsection*{AR-Bench-SP}

\paragraph{System prompt.}\leavevmode\par
\begin{lstlisting}[style=prompt]
You are an expert thinking puzzle solver. Based on all the clues gathered from the puzzle host's responses and your analysis of the puzzle dimensions, you must now provide the hidden explanation of the puzzle.
\end{lstlisting}

\paragraph{Without discrete choices.}\leavevmode\par
\begin{lstlisting}[style=prompt]
<PUZZLE>
{ambiguous_prompt}
</PUZZLE>

<PUZZLE_CONTEXT>
{meta_context}
</PUZZLE_CONTEXT>

<CONVERSATION_LOG>
{conversation_log}
</CONVERSATION_LOG>

<PUZZLE_SOLUTION_STATE>
{map_state}
</PUZZLE_SOLUTION_STATE>

Definition:
<CONVERSATION_LOG> is a list of dicts with "question_text", "user_name", and "user_answer" fields, representing the full history of the conversation between the solver and the puzzle host.
A puzzle dimension is a hidden aspect of the scenario where knowing its true value would explain the puzzle.
<PUZZLE_SOLUTION_STATE> is a structured representation of the solver's current understanding of the puzzle, where each puzzle dimension is mapped to its most likely value. This represents the solver's best guess of the hidden explanation based on the clues gathered so far.

Task:
Given <PUZZLE>, <PUZZLE_CONTEXT>, <CONVERSATION_LOG>, and <PUZZLE_SOLUTION_STATE>, provide the hidden explanation of the puzzle.

What to generate:
- reason: a short one-sentence explanation of why this is the correct explanation given the clues in <PUZZLE>, <PUZZLE_CONTEXT>, <CONVERSATION_LOG>, and <PUZZLE_SOLUTION_STATE>.
- final_answer: the hidden explanation of the puzzle.

Constraints:
- The final answer must be consistent with the clues in <PUZZLE>, <PUZZLE_CONTEXT>, <CONVERSATION_LOG>, and <PUZZLE_SOLUTION_STATE>.
- Do NOT rewrite or restate the <PUZZLE>, <PUZZLE_CONTEXT>, <CONVERSATION_LOG>, or <PUZZLE_SOLUTION_STATE>.
- answer must be a natural language explanation of the puzzle.

Output format:
Return STRICT JSON only with the following schema:
{{
    "reason": string,
    "final_answer": string
}}
\end{lstlisting}

\paragraph{With discrete choices.}\leavevmode\par
\begin{lstlisting}[style=prompt]
<PUZZLE>
{ambiguous_prompt}
</PUZZLE>

<PUZZLE_CONTEXT>
{meta_context}
</PUZZLE_CONTEXT>

<CONVERSATION_LOG>
{conversation_log}
</CONVERSATION_LOG>

<PUZZLE_SOLUTION_STATE>
{map_state}
</PUZZLE_SOLUTION_STATE>

<POSSIBLE_EXPLANATIONS_WITH_IDS>
{possible_answers_with_ids}
</POSSIBLE_EXPLANATIONS_WITH_IDS>

Definition:
<CONVERSATION_LOG> is a list of dicts with "question_text", "user_name", and "user_answer" fields, representing the full history of the conversation between the solver and the puzzle host.
A puzzle dimension is a hidden aspect of the scenario where knowing its true value would explain the puzzle.
<PUZZLE_SOLUTION_STATE> is a structured representation of the solver's current understanding of the puzzle, where each puzzle dimension is mapped to its most likely value. This represents the solver's best guess of the hidden explanation based on the clues gathered so far.
<POSSIBLE_EXPLANATIONS_WITH_IDS> is a list of dicts with "id" and "value" fields. Each dict corresponds to a possible explanation of the puzzle.
Let possible_answers[i] be the i-th element of <POSSIBLE_EXPLANATIONS_WITH_IDS>.

Task:
Given <PUZZLE>, <PUZZLE_CONTEXT>, <CONVERSATION_LOG>, and <PUZZLE_SOLUTION_STATE>, select the correct explanation from <POSSIBLE_EXPLANATIONS_WITH_IDS>.

What to generate:
- reason: a short one-sentence explanation of why this is the correct explanation given the clues in <PUZZLE>, <PUZZLE_CONTEXT>, <CONVERSATION_LOG>, and <PUZZLE_SOLUTION_STATE>.
- final_answer_id: the id of the explanation in <POSSIBLE_EXPLANATIONS_WITH_IDS> that you are selecting as the hidden explanation of the puzzle.

Constraints:
- The final answer must be consistent with the clues in <PUZZLE>, <PUZZLE_CONTEXT>, <CONVERSATION_LOG>, and <PUZZLE_SOLUTION_STATE>.
- Do NOT rewrite or restate the <PUZZLE>, <PUZZLE_CONTEXT>, <CONVERSATION_LOG>, or <PUZZLE_SOLUTION_STATE>.
- final_answer_id must be one of the ids (i.e possible_answers[i]["id"]) provided in <POSSIBLE_EXPLANATIONS_WITH_IDS>.

Output format:
Return STRICT JSON only with the following schema:
{{
    "reason": string,
    "final_answer_id": string
}}
\end{lstlisting}

\subsubsection*{iCraft-MD}

\paragraph{System prompt.}\leavevmode\par
\begin{lstlisting}[style=prompt]
You are an experienced physician concluding a clinical assessment. Based on all the clinical information gathered from the patient interview and your analysis of the clinical dimensions, you must now provide your diagnosis or clinical decision.
\end{lstlisting}

\paragraph{Without discrete choices.}\leavevmode\par
\begin{lstlisting}[style=prompt]
<CLINICAL_QUESTION>
{ambiguous_prompt}
</CLINICAL_QUESTION>

<PATIENT_INFORMATION>
{meta_context}
</PATIENT_INFORMATION>

<CLINICAL_INTERVIEW_LOG>
{conversation_log}
</CLINICAL_INTERVIEW_LOG>

<CLINICAL_ASSESSMENT>
{map_state}
</CLINICAL_ASSESSMENT>

Definition:
<CLINICAL_INTERVIEW_LOG> is a list of dicts with "question_text", "user_name", and "user_answer" fields, representing the full history of the clinical interview between the physician and the patient.
A clinical dimension is a specific clinical factor where different values would point toward different diagnoses or clinical decisions.
<CLINICAL_ASSESSMENT> is a structured representation of the physician's current understanding of the patient's condition, where each clinical dimension is mapped to its most likely value. This represents the physician's best assessment of the true clinical state based on the interview so far.

Task:
Given <CLINICAL_QUESTION>, <PATIENT_INFORMATION>, <CLINICAL_INTERVIEW_LOG>, and <CLINICAL_ASSESSMENT>, provide your diagnosis or clinical decision.

What to generate:
- reason: a short one-sentence explanation of why this is the correct diagnosis or clinical decision given the clinical evidence in <CLINICAL_QUESTION>, <PATIENT_INFORMATION>, <CLINICAL_INTERVIEW_LOG>, and <CLINICAL_ASSESSMENT>.
- final_answer: your diagnosis or clinical decision.

Constraints:
- The final answer must be consistent with the clinical evidence in <CLINICAL_QUESTION>, <PATIENT_INFORMATION>, <CLINICAL_INTERVIEW_LOG>, and <CLINICAL_ASSESSMENT>.
- Do NOT rewrite or restate the <CLINICAL_QUESTION>, <PATIENT_INFORMATION>, <CLINICAL_INTERVIEW_LOG>, or <CLINICAL_ASSESSMENT>.
- answer must be a natural language answer to the <CLINICAL_QUESTION>.

Output format:
Return STRICT JSON only with the following schema:
{{
    "reason": string,
    "final_answer": string
}}
\end{lstlisting}

\paragraph{With discrete choices.}\leavevmode\par
\begin{lstlisting}[style=prompt]
<CLINICAL_QUESTION>
{ambiguous_prompt}
</CLINICAL_QUESTION>

<PATIENT_INFORMATION>
{meta_context}
</PATIENT_INFORMATION>

<CLINICAL_INTERVIEW_LOG>
{conversation_log}
</CLINICAL_INTERVIEW_LOG>

<CLINICAL_ASSESSMENT>
{map_state}
</CLINICAL_ASSESSMENT>

<DIAGNOSTIC_OPTIONS_WITH_IDS>
{possible_answers_with_ids}
</DIAGNOSTIC_OPTIONS_WITH_IDS>

Definition:
<CLINICAL_INTERVIEW_LOG> is a list of dicts with "question_text", "user_name", and "user_answer" fields, representing the full history of the clinical interview between the physician and the patient.
A clinical dimension is a specific clinical factor where different values would point toward different diagnoses or clinical decisions.
<CLINICAL_ASSESSMENT> is a structured representation of the physician's current understanding of the patient's condition, where each clinical dimension is mapped to its most likely value. This represents the physician's best assessment of the true clinical state based on the interview so far.
<DIAGNOSTIC_OPTIONS_WITH_IDS> is a list of dicts with "id" and "value" fields. Each dict corresponds to a possible diagnosis or clinical decision.
Let possible_answers[i] be the i-th element of <DIAGNOSTIC_OPTIONS_WITH_IDS>.

Task:
Given <CLINICAL_QUESTION>, <PATIENT_INFORMATION>, <CLINICAL_INTERVIEW_LOG>, and <CLINICAL_ASSESSMENT>, select the correct diagnosis or clinical decision from <DIAGNOSTIC_OPTIONS_WITH_IDS>.

What to generate:
- reason: a short one-sentence explanation of why this is the correct diagnosis or clinical decision given the clinical evidence in <CLINICAL_QUESTION>, <PATIENT_INFORMATION>, <CLINICAL_INTERVIEW_LOG>, and <CLINICAL_ASSESSMENT>.
- final_answer_id: the id of the choice in <DIAGNOSTIC_OPTIONS_WITH_IDS> that you are selecting as your diagnosis or clinical decision.

Constraints:
- The final answer must be consistent with the clinical evidence in <CLINICAL_QUESTION>, <PATIENT_INFORMATION>, <CLINICAL_INTERVIEW_LOG>, and <CLINICAL_ASSESSMENT>.
- Do NOT rewrite or restate the <CLINICAL_QUESTION>, <PATIENT_INFORMATION>, <CLINICAL_INTERVIEW_LOG>, or <CLINICAL_ASSESSMENT>.
- final_answer_id must be one of the ids (i.e possible_answers[i]["id"]) provided in <DIAGNOSTIC_OPTIONS_WITH_IDS>.

Output format:
Return STRICT JSON only with the following schema:
{{
    "reason": string,
    "final_answer_id": string
}}
\end{lstlisting}

\subsection{Answer Likelihood}
\label{app:prompts:answer_lik}

\subsubsection*{AR-Bench-DC}

\paragraph{System prompt.}\leavevmode\par
\begin{lstlisting}[style=prompt]
You are an experienced detective evaluating how likely each suspect is to be the real murderer under different assumptions about the state of the case.
\end{lstlisting}

\paragraph{User prompt.}\leavevmode\par
\begin{lstlisting}[style=prompt]
<CASE_QUESTION>
{ambiguous_prompt}
</CASE_QUESTION>

<CASE_BACKGROUND>
{meta_context}
</CASE_BACKGROUND>

<DIMENSION_NAME>
{dimension_name}
</DIMENSION_NAME>

<DIMENSION_VALUES_WITH_IDS>
{dimension_values_with_ids}
</DIMENSION_VALUES_WITH_IDS>

<SUSPECTS_WITH_IDS>
{possible_answers_with_ids}
</SUSPECTS_WITH_IDS>

Definition:
An investigative dimension is a specific aspect of the murder case where multiple suspects could plausibly be implicated, and resolving it would narrow down the true murderer.
<DIMENSION_VALUES_WITH_IDS> is a list of dicts with "id" and "text" fields. Each dict corresponds to a possible value that the <DIMENSION_NAME> could take.
<SUSPECTS_WITH_IDS> is a list of dicts with "id" and "text" fields. Each dict corresponds to a suspect who could be the real murderer.
Let values[i] be the i-th element of <DIMENSION_VALUES_WITH_IDS>.
Let answers[j] be the j-th element of <SUSPECTS_WITH_IDS>.

Task (row-major order):
For i = 0..len(values)-1:
For j = 0..len(answers)-1:
    - Assume the true state of the case is <DIMENSION_NAME> = values[i]["text"].
    - Judge how likely it is that answers[j]["text"] is the real murderer given that assumption and the evidence in <CASE_BACKGROUND>.

What to generate:
For i = 0..len(values)-1:
For j = 0..len(answers)-1:
    - answer_id: the id of the suspect being evaluated, i.e. answers[j]["id"]
    - dimension_value_id: the id of the dimension value being evaluated, i.e. values[i]["id"]
    - reason: a short one-sentence explanation of why answers[j] is likely/neutral/unlikely to be the real murderer given <DIMENSION_NAME> = values[i]["text"].
    - label: one of "likely", "neutral", or "unlikely" according to the following definitions:
        - "likely": Given <DIMENSION_NAME> = values[i]["text"] and the evidence in <CASE_BACKGROUND>, answers[j]["text"] is the expected real murderer.
        - "neutral": Given <DIMENSION_NAME> = values[i]["text"] and the evidence in <CASE_BACKGROUND>, answers[j]["text"] is a plausible suspect but not specifically implicated; there is insufficient evidence to say this suspect is more or less guilty than others.
        - "unlikely": Given <DIMENSION_NAME> = values[i]["text"] and the evidence in <CASE_BACKGROUND>, answers[j]["text"] is not expected to be the real murderer.

Constraints:
- Use ONLY the information provided in <CASE_QUESTION>, <CASE_BACKGROUND>, and the assumed dimension value.
- Do NOT answer the <CASE_QUESTION> itself. Focus ONLY on judging how likely each suspect is to be the real murderer under different assumptions about the investigative dimension.
- Do NOT rewrite or restate the <CASE_QUESTION> or <CASE_BACKGROUND>.
- label must be one of "likely", "neutral", or "unlikely".
- The output must include an entry for every combination of dimension value and suspect.

Output format:
Return STRICT JSON only with the following schema:
{{
    "evaluations": [
        [
            {{
                "answer_id": string,
                "dimension_value_id": string,
                "reason": string,
                "label": string
            }},
            ... // one object for each suspect
        ],
        ... // one array for each dimension value
    ]
}}
The "evaluations" field must contain exactly {num_dimension_values} arrays (one per dimension value).
Each inner array must contain exactly {num_possible_answers} objects (one per suspect).
\end{lstlisting}

\subsubsection*{AR-Bench-SP}

\paragraph{System prompt.}\leavevmode\par
\begin{lstlisting}[style=prompt]
You are an expert thinking puzzle analyst. Evaluate how likely each candidate explanation is to be the correct hidden explanation of the puzzle under different assumptions about the puzzle's hidden dimensions.
\end{lstlisting}

\paragraph{User prompt.}\leavevmode\par
\begin{lstlisting}[style=prompt]
<PUZZLE>
{ambiguous_prompt}
</PUZZLE>

<PUZZLE_CONTEXT>
{meta_context}
</PUZZLE_CONTEXT>

<DIMENSION_NAME>
{dimension_name}
</DIMENSION_NAME>

<DIMENSION_VALUES_WITH_IDS>
{dimension_values_with_ids}
</DIMENSION_VALUES_WITH_IDS>

<POSSIBLE_EXPLANATIONS_WITH_IDS>
{possible_answers_with_ids}
</POSSIBLE_EXPLANATIONS_WITH_IDS>

Definition:
A puzzle dimension is a hidden aspect of the scenario where knowing its true value would explain the puzzle.
<DIMENSION_VALUES_WITH_IDS> is a list of dicts with "id" and "text" fields. Each dict corresponds to a possible value that the <DIMENSION_NAME> could take.
<POSSIBLE_EXPLANATIONS_WITH_IDS> is a list of dicts with "id" and "text" fields. Each dict corresponds to a candidate hidden explanation of the <PUZZLE>.
Let values[i] be the i-th element of <DIMENSION_VALUES_WITH_IDS>.
Let answers[j] be the j-th element of <POSSIBLE_EXPLANATIONS_WITH_IDS>.

Task (row-major order):
For i = 0..len(values)-1:
For j = 0..len(answers)-1:
    - Assume the hidden explanation is such that <DIMENSION_NAME> = values[i]["text"].
    - Judge how likely it is that answers[j]["text"] is the correct hidden explanation of the <PUZZLE> under that assumption.

What to generate:
For i = 0..len(values)-1:
For j = 0..len(answers)-1:
    - answer_id: the id of the candidate explanation being evaluated, i.e. answers[j]["id"]
    - dimension_value_id: the id of the dimension value being evaluated, i.e. values[i]["id"]
    - reason: a short one-sentence explanation of why answers[j] is likely/neutral/unlikely to be the correct hidden explanation given <DIMENSION_NAME> = values[i]["text"].
    - label: one of "likely", "neutral", or "unlikely" according to the following definitions:
        - "likely": Given <DIMENSION_NAME> = values[i]["text"] and the clues in <PUZZLE_CONTEXT>, answers[j]["text"] is the expected correct hidden explanation of the <PUZZLE>.
        - "neutral": Given <DIMENSION_NAME> = values[i]["text"] and the clues in <PUZZLE_CONTEXT>, answers[j]["text"] is a plausible explanation but not specifically supported; there is insufficient evidence to say it is more or less correct than other explanations.
        - "unlikely": Given <DIMENSION_NAME> = values[i]["text"] and the clues in <PUZZLE_CONTEXT>, answers[j]["text"] is not expected to be the correct hidden explanation of the <PUZZLE>.

Constraints:
- Use ONLY the information provided in <PUZZLE>, <PUZZLE_CONTEXT>, and the assumed dimension value.
- Do NOT solve the <PUZZLE> itself. Focus ONLY on judging how likely each candidate explanation is to be correct under different assumptions about the puzzle dimension.
- Do NOT rewrite or restate the <PUZZLE> or <PUZZLE_CONTEXT>.
- label must be one of "likely", "neutral", or "unlikely".
- The output must include an entry for every combination of dimension value and candidate explanation.

Output format:
Return STRICT JSON only with the following schema:
{{
    "evaluations": [
        [
            {{
                "answer_id": string,
                "dimension_value_id": string,
                "reason": string,
                "label": string
            }},
            ... // one object for each candidate explanation
        ],
        ... // one array for each dimension value
    ]
}}
The "evaluations" field must contain exactly {num_dimension_values} arrays (one per dimension value).
Each inner array must contain exactly {num_possible_answers} objects (one per candidate explanation).
\end{lstlisting}

\subsubsection*{iCraft-MD}

\paragraph{System prompt.}\leavevmode\par
\begin{lstlisting}[style=prompt]
You are an experienced physician evaluating how likely each candidate diagnosis or clinical decision is to be correct under different assumptions about the patient's clinical state.
\end{lstlisting}

\paragraph{User prompt.}\leavevmode\par
\begin{lstlisting}[style=prompt]
<CLINICAL_QUESTION>
{ambiguous_prompt}
</CLINICAL_QUESTION>

<PATIENT_INFORMATION>
{meta_context}
</PATIENT_INFORMATION>

<DIMENSION_NAME>
{dimension_name}
</DIMENSION_NAME>

<DIMENSION_VALUES_WITH_IDS>
{dimension_values_with_ids}
</DIMENSION_VALUES_WITH_IDS>

<DIAGNOSTIC_OPTIONS_WITH_IDS>
{possible_answers_with_ids}
</DIAGNOSTIC_OPTIONS_WITH_IDS>

Definition:
A clinical dimension is a specific clinical factor where different values would point toward different diagnoses or clinical decisions.
<DIMENSION_VALUES_WITH_IDS> is a list of dicts with "id" and "text" fields. Each dict corresponds to a possible value that the <DIMENSION_NAME> could take.
<DIAGNOSTIC_OPTIONS_WITH_IDS> is a list of dicts with "id" and "text" fields. Each dict corresponds to a candidate diagnosis or clinical decision for the <CLINICAL_QUESTION>.
Let values[i] be the i-th element of <DIMENSION_VALUES_WITH_IDS>.
Let answers[j] be the j-th element of <DIAGNOSTIC_OPTIONS_WITH_IDS>.

Task (row-major order):
For i = 0..len(values)-1:
For j = 0..len(answers)-1:
    - Assume the patient's true clinical state is <DIMENSION_NAME> = values[i]["text"].
    - Judge how likely it is that answers[j]["text"] is the correct diagnosis or clinical decision for <CLINICAL_QUESTION> under that assumption.

What to generate:
For i = 0..len(values)-1:
For j = 0..len(answers)-1:
    - answer_id: the id of the candidate diagnosis being evaluated, i.e. answers[j]["id"]
    - dimension_value_id: the id of the dimension value being evaluated, i.e. values[i]["id"]
    - reason: a short one-sentence explanation of why answers[j] is likely/neutral/unlikely to be the correct diagnosis given <DIMENSION_NAME> = values[i]["text"].
    - label: one of "likely", "neutral", or "unlikely" according to the following definitions:
        - "likely": Given <DIMENSION_NAME> = values[i]["text"] and the patient information in <PATIENT_INFORMATION>, answers[j]["text"] is the expected correct diagnosis or clinical decision for <CLINICAL_QUESTION>.
        - "neutral": Given <DIMENSION_NAME> = values[i]["text"] and the patient information in <PATIENT_INFORMATION>, answers[j]["text"] is a plausible diagnosis but not specifically supported; there is insufficient clinical evidence to say it is more or less correct than other options.
        - "unlikely": Given <DIMENSION_NAME> = values[i]["text"] and the patient information in <PATIENT_INFORMATION>, answers[j]["text"] is not expected to be the correct diagnosis or clinical decision for <CLINICAL_QUESTION>.

Constraints:
- Use ONLY the information provided in <CLINICAL_QUESTION>, <PATIENT_INFORMATION>, and the assumed dimension value.
- Do NOT answer the <CLINICAL_QUESTION> itself. Focus ONLY on judging how likely each candidate diagnosis is to be correct under different assumptions about the clinical dimension.
- Do NOT rewrite or restate the <CLINICAL_QUESTION> or <PATIENT_INFORMATION>.
- label must be one of "likely", "neutral", or "unlikely".
- The output must include an entry for every combination of dimension value and candidate diagnosis.

Output format:
Return STRICT JSON only with the following schema:
{{
    "evaluations": [
        [
            {{
                "answer_id": string,
                "dimension_value_id": string,
                "reason": string,
                "label": string
            }},
            ... // one object for each candidate diagnosis
        ],
        ... // one array for each dimension value
    ]
}}
The "evaluations" field must contain exactly {num_dimension_values} arrays (one per dimension value).
Each inner array must contain exactly {num_possible_answers} objects (one per candidate diagnosis).
\end{lstlisting}

\subsection{Prompt-Sensitivity Experiment}
\label{app:prompts:sensitivity}

The v0 prompts are the standard iCraft-MD dimension-proposal prompts in
Appendix~\ref{app:prompts:init_dims}. Variant v1 retains the v0 system prompt
but uses a minimally scaffolded user prompt. Variant v2 uses the generic
dimension-proposal prompt.

\subsubsection*{Variant v1: Terse Clinical Prompt}

\paragraph{System prompt.}\leavevmode\par
\begin{lstlisting}[style=prompt]
You are an experienced physician performing a diagnosis. Your goal is to identify the key clinical dimensions that must be clarified to arrive at the correct diagnosis or clinical decision.
\end{lstlisting}

\paragraph{User prompt.}\leavevmode\par
\begin{lstlisting}[style=prompt]
<CLINICAL_QUESTION>
{ambiguous_prompt}
</CLINICAL_QUESTION>

<PATIENT_INFORMATION>
{meta_context}
</PATIENT_INFORMATION>

List exactly {num_initial_dims} clinical dimensions whose values would discriminate between candidate diagnoses. For each: a one-sentence reason, a short name, and at most {max_num_values_per_dim} plausible values.

Return STRICT JSON:
{{
    "dimensions": [
        {{"reason": string, "name": string, "values": [string, ...]}},
        ...
    ]
}}
\end{lstlisting}

\subsubsection*{Variant v2: Domain-Neutral Prompt}

The system prompt is empty.

\paragraph{User prompt.}\leavevmode\par
\begin{lstlisting}[style=prompt]
<AMBIGUOUS_PROMPT>
{ambiguous_prompt}
</AMBIGUOUS_PROMPT>

<META_CONTEXT>
{meta_context}
</META_CONTEXT>

Task:
Identify the dimensions of ambiguity in the <AMBIGUOUS_PROMPT>.

Definition:
A dimension of ambiguity is a specific aspect of the <AMBIGUOUS_PROMPT> that can be interpreted in multiple valid ways, leading to different possible meanings or outcomes. Once the dimension's value is known, the prompt moves toward a single dominant interpretation with a well-defined ground-truth answer.

What to generate:
- Produce a minimal, non-overlapping set of dimensions of ambiguity.
- Produce exactly {num_initial_dims} dimensions of ambiguity.
- Each dimension must correspond to a distinct, intent-level uncertainty.
- If <META_CONTEXT> already resolves a dimension, do not include it.
- If <META_CONTEXT> context proposes some dimensions of ambiguity, use them.

For each dimension of ambiguity, provide :
- reason: a short one-sentence explanation of why this dimension is a source of ambiguity in the prompt.
- name: a short, specific label
- values: a list of plausible values, no larger than {max_num_values_per_dim}, that this dimension could take in the context of the prompt.

Constraints:
- Use ONLY the information provided in <AMBIGUOUS_PROMPT> and <META_CONTEXT>.
- Do NOT answer the <AMBIGUOUS_PROMPT> itself. Focus ONLY on identifying dimensions of ambiguity.
- Do NOT rewrite or restate the <AMBIGUOUS_PROMPT> or <META_CONTEXT>.

Output format:
Return STRICT JSON only with the following schema:
{{
    "dimensions": [
        {{
            "reason": string,
            "name": string,
            "values": [string, ...]
        }},
        ...
    ]
}}
\end{lstlisting}

\subsubsection*{Semantic-Matching Judge}

\paragraph{System prompt.}\leavevmode\par
\begin{lstlisting}[style=prompt]
You are comparing two sets of clinical dimensions proposed for the same clinical case. Identify which dimensions in Set A are semantically equivalent to dimensions in Set B.
\end{lstlisting}

\paragraph{User prompt.}\leavevmode\par
\begin{lstlisting}[style=prompt]
<SET_A>
{a_block}
</SET_A>

<SET_B>
{b_block}
</SET_B>

Task: For each dimension in Set A, identify if there is a semantically equivalent dimension in Set B (same underlying clinical concept / observation / factor, even if worded differently). Return ONLY the matched pairs.

Constraints:
- Each dimension in A appears in AT MOST one matched pair.
- Each dimension in B appears in AT MOST one matched pair.
- Use the 0-based indices shown in <SET_A> and <SET_B>.

For each matched pair, provide:
- a_idx: integer index in Set A
- b_idx: integer index in Set B
- reason: one short sentence explaining the equivalence

Return STRICT JSON only.
\end{lstlisting}

\subsection{Oracle-Aligned DST Evaluation Judge}
\label{app:prompts:dst}

These prompts are used only for post-hoc evaluation.

\paragraph{System prompt.}\leavevmode\par
\begin{lstlisting}[style=prompt]
You are aligning two clinical dimension ontologies for the same patient case: (a) oracle dimensions derived from the gold patient state, and (b) BALAR's dynamically-proposed dimensions. For each oracle dimension, identify the semantically equivalent BALAR dimension (if any), the BALAR value tokens that semantically match the oracle's optimal_value, and which of BALAR's asked questions probed this oracle dimension.
\end{lstlisting}

\paragraph{User prompt.}\leavevmode\par
\begin{lstlisting}[style=prompt]
<CLINICAL_CASE>
{amb}

{ctx}
</CLINICAL_CASE>

<ORACLE_DIMENSIONS>
{oracle_block}
</ORACLE_DIMENSIONS>

<BALAR_DIMENSIONS>
{balar_block}
</BALAR_DIMENSIONS>

<BALAR_ASKED_QUESTIONS>
{asked_block}
</BALAR_ASKED_QUESTIONS>

Task: For each oracle dimension (one record per oracle_dim_id, all exactly once), produce:
- oracle_dim_id: integer index from <ORACLE_DIMENSIONS> (0-based)
- aligned_balar_dim_id: id of the BALAR dimension that is semantically equivalent to this oracle dim, or -1 if BALAR proposed no such dimension
- matching_balar_values: BALAR value tokens (verbatim strings copied from the chosen BALAR dim's values list) that are semantically equivalent to this oracle dim's optimal_value; [] if no aligned BALAR dim
- probing_question_rounds: the round numbers from <BALAR_ASKED_QUESTIONS> of asked questions that probed this oracle dim; [] if none
- reason: one short sentence

Return STRICT JSON only.
\end{lstlisting}

\subsection{Question-Quality Evaluation Judge}
\label{app:prompts:question-quality}

These prompts are used only for post-hoc evaluation.

\paragraph{System prompt.}\leavevmode\par
\begin{lstlisting}[style=prompt]
You are evaluating clarifying questions asked by an expert clinician during a patient interview to disambiguate between candidate diagnoses. Score each question for relevance to the case, informativeness for the differential, non-redundancy vs prior questions, and classify it against a fixed list of proposed clinical dimensions.
\end{lstlisting}

\paragraph{User prompt.}\leavevmode\par
\begin{lstlisting}[style=prompt]
<CLINICAL_CASE>
{amb}

{ctx}
</CLINICAL_CASE>

<CANDIDATE_DIAGNOSES>
{candidates_block}
</CANDIDATE_DIAGNOSES>

<PROPOSED_DIMENSIONS>
{dims_block}
</PROPOSED_DIMENSIONS>

<PRIOR_QUESTIONS>
{prior_block}
</PRIOR_QUESTIONS>

<CURRENT_QUESTION>
{question}
</CURRENT_QUESTION>

Rate the <CURRENT_QUESTION> on a 1-5 integer scale for each:
- relevance: does it relate to the <CLINICAL_CASE>? (1=unrelated, 5=highly relevant)
- informativeness: would the answer narrow the <CANDIDATE_DIAGNOSES>? (1=uninformative, 5=substantially narrows)
- non_redundancy: is it asking something new vs the <PRIOR_QUESTIONS>? (1=fully redundant, 5=entirely novel)

Also identify:
- target_dim_id: integer id from <PROPOSED_DIMENSIONS> that this question primarily targets, or -1 if none.
- reason: one short sentence explaining your scoring.

Return STRICT JSON only.
\end{lstlisting}

\end{document}